\documentclass{article}

\usepackage{fullpage}       %
\usepackage[sort]{natbib}
\usepackage[utf8]{inputenc} %
\usepackage[T1]{fontenc}    %
\usepackage[backref=page]{hyperref}
 \renewcommand*{\backrefalt}[4]{%
    \ifcase #1%
     \or (Cited on page~#2.)%
     \else (Cited on pages~#2.)%
    \fi%
    }

\usepackage{url}            %
\usepackage{booktabs}       %
\usepackage{amsfonts}       %
\usepackage{dsfont}
\usepackage{nicefrac}       %
\usepackage{microtype}      %
\usepackage{graphicx}
\usepackage{algorithm}
\usepackage{algorithmic}
\usepackage{enumitem,kantlipsum}
\usepackage[]{mdframed}

\usepackage[theorem, hyperstyle, subfigref]{xiangpackage}

\usepackage{booktabs}
\usepackage{tabularx}

\usepackage{physics}
\usepackage{makecell}
\usepackage{tablefootnote}

\usepackage{crossreftools}
\usepackage{tikz}
\usetikzlibrary{tikzmark,calc}

\newtheorem{principle}{Principle}
\newtheorem{desideratum}{Desideratum}
\crefname{assumption}{assumption}{assumptions}
\crefname{desideratum}{desideratum}{desiderata}
\crefname{principle}{principle}{Principle}

\usepackage{enumitem}

\newif\ifsubmissionready 
\submissionreadytrue 

\newcommand{\red}[1]{#1}
\renewcommand{\xiang}[1]{}

\newcommand{\added}[1]{#1}

\newcommand{\oursde}{\texttt{HA-SME}}
\newcommand{\oursdefull}{Hessian-Aware Stochastic Modified Equation}
\newcommand{\oursdefullh}{Hessian-Aware Stochastic Modified Equation}
\newcommand{\sme}{\texttt{SME}}
\newcommand{\pfs}{\texttt{SPF}}
\newcommand{\pfsfull}{Stochastic Principal Flow}

\title{\bf 
A Hessian-Aware Stochastic Differential Equation for Modelling SGD
}

\addtocounter{footnote}{1} %
\author{
Xiang Li\thanks{Department of Computer Science, ETH Zurich, Switzerland.
\texttt{xiang.li@inf.ethz.ch},
\texttt{zebang.shen@inf.ethz.ch},
\texttt{liang.zhang@inf.ethz.ch},
\texttt{niao.he@inf.ethz.ch}.}
\and
Zebang Shen\footnotemark[2]
\and
Liang Zhang\footnotemark[2]
\and
Niao He\footnotemark[2]
}

\date{}

\begin{document}

\maketitle

\begin{abstract}
    Continuous-time approximation of Stochastic Gradient Descent (SGD) is a crucial tool to study its escaping behaviors from stationary points.
    However, existing stochastic differential equation (SDE) models fail to fully capture these behaviors, even for simple quadratic objectives.
    Built on a novel stochastic backward error analysis framework, we derive the \oursdefull{} (\oursde), an SDE that incorporates Hessian information of the objective function into both its drift and diffusion terms.  
    Our analysis shows that \oursde{} achieves the order-best approximation error guarantee among existing SDE models in the literature, 
    while significantly reducing the dependence on the smoothness parameter of the objective.
    \added{Empirical experiments on neural network–based loss functions further validate
    this improvement.}
    Further, for quadratic objectives, under mild conditions, \oursde\ is proved to be the first SDE model that recovers exactly the SGD dynamics in the distributional sense.
    Consequently, when the local landscape near a stationary point can be approximated by quadratics,
    \oursde{} provides a more precise characterization of the local escaping behaviors of SGD.
    \added{ With the enhanced approximation guarantee, we further conduct an escape
      time analysis using \oursde{}, showcasing how it can be employed to
      analytically study the escaping behavior of SGD for general function
      classes. }
  \end{abstract}

\section{Introduction}
\label{sec:intro}
We consider unconstrained stochastic minimization problems of the form
\begin{align} \label{eq:min}
  \min_{x} f(x) \coloneqq \Ep[\xi]{F(x; \xi)},
\end{align}
where $f: \mathbb{R}^d \to \mathbb{R}$ is a smooth function and $\xi$ is a
random vector.
In this work, we study Stochastic Gradient Descent
(SGD)~\citep{10.1214/aoms/1177729586}, the most popular approach for solving such
problems, which updates iteratively as follows:
\begin{align} \label{eq:sgd}
  x_{k+1} = x_k - \eta_k \nabla F(x_k; \xi_k),
\end{align}
where $\eta_k > 0$ is the stepsize and $\nabla F(x_k; \xi_k)$ is an unbiased stochastic gradient at $x_k$.
In the absence of noise, SGD reduces to Gradient Descent~(GD).
The choice of the stepsize $\eta_k$ is crucial, affecting aspects such as convergence speed, training
stability~\citep{nar2018step,bengio2012practical}, and generalization error~\citep{wilson2017marginal,Zhou2022}.

Existing non-convex optimization theory have established the non-asymptotic convergence of SGD
towards a \emph{stationary point} using a \emph{diminishing} stepsize; see e.g., 
\citet{ghadimi2013stochastic,khaled2022better,10.5555/3524938.3525187,pmlr-v99-fang19a,sebbouh2021almost,li2020high}
and \citet{yang2024two}, to list just a few. 
Intuitively, the diminishing stepsize reduces the fluctuation caused by
SGD's inherent randomness,
ensuring the objective descends in expectation.
On the other hand, in practice a \emph{constant} stepsize, i.e. $\eta_k \equiv
\eta$, is often preferred and yields better performance. This presents a
challenge to classical theory, as it fails to capture the non-diminishing role
of noise in SGD dynamics. Additionally, it falls short in
differentiating between saddle points and local minima, as well as 
in identifying favorable stationary points --- those that not only are local
minima but also exhibit good generalization.

The above limitations underscore the need to understand a critical aspect
of SGD in nonconvex optimization: \emph{the transitions between different stationary
points}, beyond the traditional focus on convergence towards a stationary point.
The selection of stepsize plays a crucial role in SGD's transition behaviors:
\begin{itemize}
    \item \textbf{Escaping from stationary points.} At a stationary point, either a saddle point or a local minimum, increasing the stepsize raises the likelihood of SGD to escape from it.
    \item \textbf{Minima selection.}
Given a constant stepsize, SGD can rapidly exit local minima with a
large Hessian norm, also known as sharp minima, but dwell longer at flat minima where the norm is small.
\end{itemize}
The latter relates to the notable generalization capabilities of SGD,
which in prevailing belief, tends to favor flat minima over sharp ones~\citep{keskar2016large,neyshabur2017exploring,neyshabur2017pac}.
These phenomena have been empirically verified~\citep{zhang2021understanding} and theoretically
investigated~\citep{pemantle1990nonconvergence,brandiere1996algorithmes,mertikopoulos2020almost,liu2023almost}.
Indeed, analyzing the link between stepsize and the escaping dynamics of SGD has emerged as an important topic in non-convex
optimization.
In particular,
there is growing interest in using continuous-time dynamics, specifically
Stochastic Differential Equations (SDEs), to understand the mechanisms behind
escaping from saddle points and minima selection~\citep{hu2017diffusion,xie2022adaptive}. This involves quantifying the interplay between escaping speed, stepsizes, and
local sharpness~\citep{zhu2019anisotropic,hu2017diffusion,nguyen2019first,xie2020diffusion,ibayashi2023does}.
A more detailed discussion of these works is provided in 
\Cref{sec:related}.

\begin{figure}[t]
  \centering
  \begin{subfigure}{0.4\textwidth}
    \includegraphics[width=\linewidth]{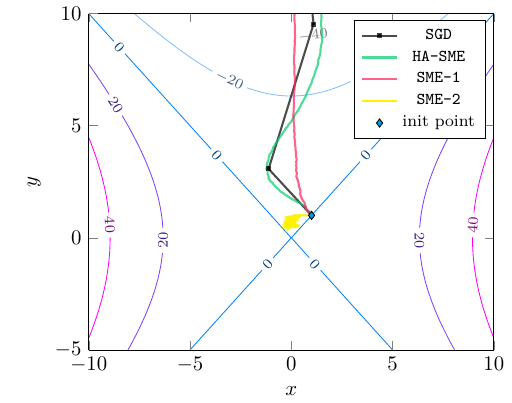}
    \caption{$f\left([x, y]^\T\right) = \frac{1}{2} \left(x^2 - y^2\right)$}
      \label{fig:quad_saddle}
  \end{subfigure}
  \hspace{1cm}
  \begin{subfigure}{0.4\textwidth}
    \includegraphics[width=\linewidth]{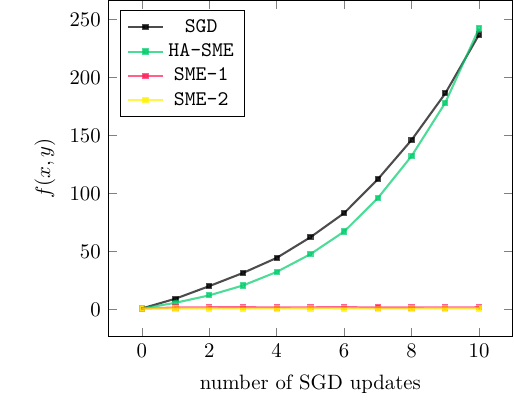}
      \caption{
      $f\left([x, y]^\T\right) = \frac{1}{2} \left(x^2 + y^2\right)$
    }
      \label{fig:quad_minimum}
  \end{subfigure}
  \caption{Escaping behaviors of SGD and SDEs on quadratic functions.
    The left sub-figure demonstrates the trajectories around a saddle point $(0,0)$,
    while the right sub-figure represents the evolution of the function value around a minimum. 
    We set $\eta=2.1$, the initial point
    to $(1, 1)$, and the covariance of the additive noise to $I$ for both cases.
    We conducted the experiments 1000 times to estimate the expectation.
    For the right sub-figure, we evaluated the SDEs at time stamps corresponding to
    $\eta$ times the number of SGD updates.
    In these two cases, the proposed \oursde\ behaves similarly as SGD. However, \sme-2 fails to escape
    the saddle point in case (a), and both \sme-1 and \sme-2 fail to escape the minimum in case (b).
    We further provide in \Cref{sec:exist_quad}
  mathematical proofs for the failure cases of these two existing \sme{s} in approximating SGD.}
  \label{fig:illustrate_escape}
\end{figure}

\paragraph{Existing continuous-time approximations of SGD and their failure modes.} 
Previous study on the escaping behavior of constant-stepsize SGD through its continuous-time
approximation typically entails two steps: (1) Proposing a proxy SDE model that
approximately tracks the distributional evolution of the SGD dynamics; (2)
Analyzing the escaping behavior of the proposed SDE model to predict or
interpret the behavior of SGD.
Evidently, whether the behavior of SGD can be precisely captured hinges on the accuracy of the proxy models.
Among existing SDE models, the second-order stochastic modified
equation
(\sme-2)~\citep{li2017stochastic} provides the order-best approximation guarantee:

\begin{align} \label{eq:SDE2}
dX_t = - \nabla \left(f(X_t) + \frac{\eta}{4} \|\nabla f(X_t)\|^2\right) dt + \sqrt{\eta \Sigma(X_t)} dW_t,
\end{align}
where $\Sigma(X_t)$ denotes the covariance matrix of the stochastic gradient at
$X_t$, $W_t$ represents the standard Wiener process, and $\eta$ denotes the constant stepsize of SGD.

However, we note that even for simple quadratic objectives, \sme-2 and SGD show different behaviors in escaping from saddle
points and minima: \sme-2 exhibits locally stabilizing behavior
while SGD with constant stepsize quickly escapes from the same local region, as illustrated in
\Cref{fig:illustrate_escape}.
\added{We further extend this observation to general objectives in \Cref{sec:escape_saddle}.}
Our study also identifies similar failure modes
for other existing SDE models, detailed in \Cref{sec:exist_quad}. This
disparity in behavior suggests the inadequacy of existing SDEs in
modeling the escaping dynamics of SGD, even with high-order approximation
such as \sme-2.
Consequently, this brings us to a central question:
\begin{quote} \emph{
How to accurately model SGD in continuous time while preserving its escaping behaviors?
} 
\end{quote}

\paragraph{Our contributions.} 
We develop a new Stochastic Backward Error Analysis (SBEA) to approximate the
dynamics of constant-stepsize SGD. 
Our analysis
yields an SDE, termed ``\oursdefullh'' (\oursde), that integrates Hessian 
information into both its drift and diffusion terms.
By incorporating local
curvature information, \oursde\ effectively tracks
the escape phenomena of SGD. The well-posedness of \oursde{} is
assured under mild regularity assumptions. Moreover, under similar
assumptions, the distribution evolution of \oursde\ \emph{exactly} matches that
of SGD on quadratic functions.

Our contributions are summarized as follows:
\begin{itemize}
  \item \textbf{Failure modes of existing SDE proxies.} 
    In this study, we 
    identify several shortcomings of current SDE approximations in capturing the escaping behaviors of SGD. We
    present concrete examples of optimization problems, 
    where SGD exhibits escaping behavior, while existing SDE
    approximations tend to stabilize around the stationary points. We further pinpoint that this failure mode stems from
    the negligence of higher-order terms involving Hessians in their derivations.
  \item \textbf{A new Hessian-aware SDE proxy.} 
    We propose the SBEA framework, which enables
    us to design \oursde{} capable of encompassing more error terms involving gradients and Hessians than existing SDEs in the error
    analysis.
   Specifically, the drift term and diffusion coefficient of \oursde{} are defined by two power series, respectively. Utilizing a
    generating function approach, we demonstrate that these power series
    converge when the stepsize is below a certain constant threshold, and
    their limits possess analytical expressions. We provide sufficient
    conditions to ensure \oursde{} is well-posed.
  \item \textbf{Approximation guarantees on general functions.}
    We establish that \oursde{} adheres to an order-2 weak approximation
    error guarantee, aligning with the best-known results among existing SDE
    approximations. To further differentiate \oursde{} from existing \sme{s}, we
    conduct a fine-grained analysis of the weak approximation error,
    taking into account of the smoothness parameter of the objective, which we denote as $\lambda$.
    For convex functions, our findings indicate that \oursde{} admits an
    approximation error of order $\mathcal{O}\left(\eta^2\right)$.
    In stark contrast, the error for \sme-1 model is $\mathcal{O}\left(\eta \lambda\right)$,
    and for \sme-2, it is $\mathcal{O}\left(\eta^2 \lambda^2\right)$. Our
    results mark a notable improvement over the
    dependence on the smoothness parameter compared to existing SDE approximations.
    \added{To validate these theoretical findings in practical scenarios,
      we conduct experiments on a classification task using neural networks.
      The result shows that \oursde{} aligns more closely with SGD, outperforming existing
      SDE models.
    }
  \item \textbf{Exact recovery on quadratic functions.}
  \added{
    We further examine quadratic functions, which serve as a crucial benchmark
    for SDE approximations.
  }
   We show that accurately mirroring the SGD dynamics (linear in this setting) with a linear SDE (OU process) is fundamentally impossible, due to a potential mismatch between the ranks of their covariance matrices.
   \added{
     Nonetheless, we show that under mild conditions---specifically, when the
     eigenvectors of the Hessian and noise covariance matrices align or when a
     small stepsize is used---\oursde{} is well-founded and \emph{exactly} replicates constant-stepsize
     SGD dynamics in a distributional sense.
     To the best of our knowledge, these findings represent the first instance
     of exact recovery of the distribution of SGD iterates using continuous-time
     models. }

  \item \added{\textbf{Escape Time Analysis.}
To showcase the potential of using \oursde{} as an analytical tool for
studying SGD, we provide an analysis of the expected hitting time at
stationary points. In the small stepsize regime, our analysis recovers
classical results about dynamical systems with perturbations.
Furthermore, in the large
stepsize regime, where the stepsize is inversely proportional to the local
Hessian eigenvalues, we demonstrate that \sme-2 exhibits exponentially long
escape times from saddle points, whereas \oursde{} escapes sub-exponentially,
better aligning with the behavior of SGD.
}

\end{itemize}

\subsection{Related Work}
\label{sec:related}

\paragraph{Continuous-time approximation of GD and SGD.}
The connection between differential equations and gradient-based methods has
been extensively studied previously, such as in~\citep{helmke2012optimization,schropp2000dynamical,fiori2005quasi,
su2016differential, muehlebach2019dynamical, muehlebach2021optimization,
shi2021understanding, attouch2022first,krichene2015accelerated,brown1989some,bloch1994hamiltonian}.
For GD, the most straightforward
continuous-time approximation is the gradient flow 
\begin{align} \label{eq:GF}
  dX_t = -\nabla f(X_t) dt.
\end{align}
While this flow approximates GD accurately as the stepsize
approaches zero, the approximation error becomes critical when the stepsize
is large. \citet{barrett2020implicit} took a step further by
employing backward error analysis (detailed in \Cref{sec:pre_bea}) to derive a second-order continuous-time
approximation for GD:
\begin{align} \label{eq:GF2}
  dX_t = - \nabla \left(f(X_t) + \frac{\eta}{4}\|\nabla f(X_t)\|^2\right) dt. 
\end{align}
However, this second-order model still has limitations, particularly in
capturing certain critical behaviors of GD in discrete-time, such as
overshooting, i.e., the iterates oscillate around the minimum or even diverge because of large stepsize.  
To tackle this issue,
\citet{rosca2022continuous} introduced the Principal Flow (PF),
a continuous-time approximation for GD that operates in complex space and can
exactly match the discrete-time dynamics of GD on quadratics. This work is closely
related to ours, and we will dive deeper in \Cref{sec:pre_pf}.

When taking noises into account, \citet{mandt2015continuous,mandt2016variational} introduced the following SDE dynamics for approximating SGD (\sme-1):
\begin{align} \label{eq:SDE1}
d X_t = -\nabla f(X_t) dt + \sqrt{\eta \Sigma(X_t)} dW_t.
\end{align}
Later, \citet{li2017stochastic} introduced \sme-2 (described in \Cref{eq:SDE2}) and established 
rigorous weak approximation for these SDEs.
These SDEs are named according to the order of their weak
approximation accuracy (see precise definition in \Cref{def:weak_approx}).
In validating the accuracy of \sme-1 with non-infinitesimal stepsize,
\citet{li2021validity} proposed SVAG, which has been empirically demonstrated to
closely resemble SGD and theoretically proven to approximate \sme-1.
Besides the two \sme{s}, \citet{orvieto2019continuous} and
\citet{fontaine2021convergence} examined SDEs where the drift and diffusion
terms are scaled by varying stepsizes. Their SDEs reduce to \sme{s} when
adopting constant stepsizes.
Another line of work has considered Homogenized SGD, which replaces the
diffusion term in \sme-1 with $\sqrt{\eta f(X_t) \nabla^2 f(X_t) /
B}$~\citep{paquette2021dynamics, paquette2022homogenization} or $\sqrt{\eta
f(X_t) \nabla^2 f(X^*) / B}$~\citep{mori2021power},
where $B$ denotes the batch size and $X^*$ the local minimum.
However, their analyses are
limited to the least square loss, lacking approximation guarantees for general
non-convex objectives.

\added{
  Continuous-time tools have also been used to study SGD with adaptive stepsizes.
  \citet{barakat2021convergence} modeled Adam with a non-autonomous ODE, proving
  convergence. \citet{zhou2020towards} and \citet{xie2022adaptive} used
  SDEs to explain Adam's poor generalization due to sharp minima preference.
  \citet{ma2022qualitative} linked Adam’s continuous-time limit to sign gradient
  flow, explaining fast early convergence and instabilities.
  \citet{malladi2022sdes} further used SDEs to propose a scaling
  rule for stepsize and batchsize tuning.
  Additionally, \citet{compagnoni2025adaptive} provided theoretical
  insights into the role of noise in optimization dynamics of SDEs for adaptive stepsizes.
}

\paragraph{Escaping Behaviors Analysis}
Back in the 90s, \citet{pemantle1990nonconvergence} and \citet{brandiere1996algorithmes}
showed that SGD with a small stepsize can avoid hyperbolic saddle points,
i.e., $\lambda_{\min}\left(\nabla^2 f(x)\right) < 0$ and 
$\det\left(\nabla^2 f(x)\right) \neq 0$. Recent studies~\citep{mertikopoulos2020almost,liu2023almost} extended this understanding
by showing SGD can escape strict saddle points ($\lambda_{\min}\left(\nabla^2
f(x)\right) < 0$).
Additionally, with slight modifications to the algorithm, such as adding
artificial noises,
\citet{ge2015escaping} and \citet{jin2017escape} established convergence of 
noisy SGD to approximate local minima.
\added{
These discrete-time studies primarily focus on convergence analysis to local
minima and show the escape from saddle points through function value descent
analysis. However, these analyses, compared to continuous-time approaches: (1)
fail to address escaping from local minima or minima selection, (2) lack a
systematic characterization of the escape trajectory, including the escaping
direction, and (3) are often case-specific, requiring modifications to SGD.
}

In terms of minima escaping, leveraging continuous-time frameworks,
\citet{jastrzebski2017three} highlighted three crucial factors that influence
the minima found by SGD: step-size, batchsize, and gradient covariance.
\citet{zhu2019anisotropic} focused on
the role of anisotropic noise and concluded that such noise enables SGD to evade
sharp minima more effectively.
Quantitatively, the time required for escaping has been theoretically shown to
exponentially depend on the inverse of the stepsize, using theory for random
perturbations of dynamical systems~\citep{hu2017diffusion}.
Subsequent work expanded on these ideas, considering heavy-tail noise and
utilizing L\'evy-driven SDEs to examine SGD's stability around minima~\citep{nguyen2019first}.
While the aforementioned research largely focused on parameter-independent gradient noise,
\citet{xie2020diffusion} analyzed parameter-dependent anisotropic noise. Using the analysis for the 
classic Kramers' escape problem, they showed that compared to
Stochastic Gradient Langevin Dynamics,
SGD can quickly escape sharp minima,
with escape speed exponentially depending on the determinant of the Hessian.
Their analysis assumes that the system first forms a stationary
distribution around a basin before escaping. However, SGD typically selects
minima dynamically. New insights
into the exponential escape time have been developed by employing Large
Deviation Theory in non-stationary settings~\citep{ibayashi2023does}.

\subsection{Paper Structure}
\added{
The paper is structured as follows. \Cref{sec:intro} introduces the problem of
modelling SGD with continuous-time tools, highlighting challenges and limitations
of existing SDEs.
\Cref{sec:prelim} provides preliminaries, including
notation, an overview of continuous-time approximations of GD and SGD,
and a discussion on BEA and the PF approach.
\Cref{sec:SBEA} develops SBEA as a framework for deriving our improved SDE models,
and identifying challenges in exact matching at higher orders. 
\Cref{sec:def_SPF}
derives our \oursde{}, demonstrating how Hessian information is integrated into its
drift and diffusion terms, ensuring better tracking of SGD’s escaping behaviors,
and establishing conditions for its well-posedness.
\Cref{sec:approx} analyzes the
approximation error of \oursde{}, proving that it achieves an order-2 weak
approximation guarantee and showing its advantages over existing models in terms
of smoothness dependence.
These theoretical results are then validated through experiments on neural networks.
\Cref{sec:quad} evaluates the exact recovery of SGD dynamics
by \oursde{} on quadratic objectives, demonstrating failure cases of existing SDE
proxies and proving that \oursde{} can match SGD distributions under mild
conditions. 
In \Cref{sec:escape_analysis}, we analyze the expected hitting time of \oursde{}
near stationary points, demonstrating its utility as an analytical tool.
Furthermore, we compare \oursde{} with \sme-2 in the large stepsize regime,
showing that \oursde{} more accurately reflects the behavior of SGD.
In the end, \Cref{sec:conclusion} concludes with a summary of contributions and
potential directions for future research.
}

\section{Preliminaries}
\label{sec:prelim}

In this section, we first introduce the notations
used throughout our analysis, and then review the Backward Error
Analysis (BEA), a technique instrumental to the development of our stochastic
backward error analysis. Additionally, we discuss the Principal Flow
(PF) introduced by \citet{rosca2022continuous}, which inspires the derivation of our \oursde\ model.

\subsection{Notations}
Uppercase variables, such as $X$, denote continuous-time variables, while
lowercase variables like $x$ represent discrete-time variables. The
function $X(x_0, t)$ indicates the solution of a continuous-time process
starting from the initial point $x_0$ after a time duration $t$. For matrix
notation, we use $[A]_{i, j}$ to specify the element at the $(i, j)$-th position
in matrix $A$. For two matrices, $A$ and $B$, of the same size, we denote $A\!:\!B = \tr(A^\top B)$.
Applying the logarithmic function to a matrix and dividing between two matrices
are by default element-wise. We use the notation $I$ to represent the identity
matrix and $\odot$ to denote the Hadamard product between matrices.
The square root of a matrix, i.e., $\sqrt{A}$, is the matrix $B$ such that
$BB^\T = A$. In the context of functions, $\|f\|_{C^m}$ is defined as
$\sum_{|\alpha| \leq m} |D^\alpha f|_{\infty}$, and we define the set
$C_b^m\left(\mathbb{R}^d\right) = \{f \in C^m\left(\mathbb{R}^d\right) \mid
\|f\|_{C^m} < \infty \}$, where $C^m\left(\mathbb{R}^d\right)$ is the set of all
$m$-th differentiable functions on $\mathbb{R}^d$. For complex number $x$,
$\Re(x)$ and $\Im(x)$ extract the real and imaginary parts, respectively, and
$\overline{x}$ denotes the complex conjugate of $x$. For random variables $x$ and $y$,
the covariance is represented as $\Cov{x, y} = \mathbb{E}\left[\left(x -
\mathbb{E}(x)\right)\left(y - \mathbb{E}(y)\right)^\T\right]$. Lastly, we
use $W_t$ to represent the standard Wiener process.

\subsection{Backward Error Analysis (BEA)}
\label{sec:pre_bea}

Consider the GD dynamics\footnote{Similar arguments can be made for other discrete-time dynamics.}
\begin{equation}\label{eq:bea_discrete}
    x_{k+1} = x_{k} - \eta \nabla f(x_{k}),
\end{equation}
with constant stepsize $\eta$. 
Let $\widetilde{X}:[0, \infty) \rightarrow \mathbb{R}^d$ be a continuous trajectory.
The idea of BEA is to identify a series of functions, denoted as $g_i:
\mathbb{R}^d \to \mathbb{R}^d$ for $i\in\{0, 1,\ldots\}$, which constitute a
modified equation as:
\begin{align} \label{eq:bea}
  \frac{d }{d t}\widetilde{X} = \widetilde{G}_\eta (\widetilde{X}) := \sum_{p=0}^\infty \eta^p g_p(\widetilde{X}),
\end{align}
such that $\widetilde{X}((k+1)\eta)$ closely approximates $x_{k+1}$, given $\widetilde{X}(k\eta) = x_k$.
To achieve this goal, we consider the Taylor expansion of the above continuous-time process at time $k\eta$,
\ifdefined\mpversion
\begin{align}
  \begin{split}
  \label{eq:bea_continu}
  \widetilde{X}((k+1)\eta)
  &= \sum_{j=0}^\infty \frac{\eta^j}{j!} \left. \frac{d^j\widetilde{X}}{(dt)^j} \right|_{\widetilde{X}= x_k} \\
  &=
  x_k + \eta \sum_{p=0}^\infty \eta^p g_p(x_k) + \frac{\eta^2}{2} \sum_{p=0}^\infty \eta^p \nabla g_p(x_k) \sum_{p=0}^\infty \eta^p g_p(x_k) + \cdots.
  \end{split}
\end{align}
\else
\begin{align}
  \begin{split}
  \label{eq:bea_continu}
  \widetilde{X}((k+1)\eta)
  &= \sum_{j=0}^\infty \frac{\eta^j}{j!} \left. \frac{d^j\widetilde{X}}{(dt)^j} \right|_{\widetilde{X}= x_k} =
  x_k + \eta \sum_{p=0}^\infty \eta^p g_p(x_k) + \frac{\eta^2}{2} \sum_{p=0}^\infty \eta^p \nabla g_p(x_k) \sum_{p=0}^\infty \eta^p g_p(x_k) + \cdots.
  \end{split}
\end{align}
\fi
By matching the powers of $\eta$ in \Cref{eq:bea_discrete} and \Cref{eq:bea_continu} up to the $(p+1)$-order, we solve for $g_i$ for $i=0, ..., p$. This procedure ensures that the leading error term between the
continuous-time variable {$\widetilde{X} ((k+1)\eta)$} and the discrete-time
variable $x_{k+1}$, given $\widetilde{X}(k\eta) = x_k$, is of the order $\cO\left( \eta^{p+2}\right)$. For a
more comprehensive exploration of BEA, please refer to
\citet[Section~IX]{hairer2006geometric}.

When $p=0$, the system reduces to gradient flow~(\Cref{eq:GF}) and
then $p=1$, we recover \Cref{eq:GF2}.
Increasing the order $p$ in the above BEA procedure improves the
accuracy of the approximation. The higher order modified flows
generated by BEA involve higher-order gradients of the objective
function $f$. For example, for $p=2$, we obtain the third-order Gradient Flow~\citep{rosca2022continuous}:
\ifdefined\mpversion
\begin{align*}
  dX_t/dt 
  &= - \nabla f(X_t) - \frac{\eta}{2}\nabla^2 f(X_t) \nabla f(X_t) \\
  &\fakeeq - \eta^2 \left( \frac{1}{3} \nabla^2 f(X_t) \nabla f(X_t) +
 \frac{1}{12} \nabla f(X_t)^\T \nabla^3 f(X_t) \nabla f(X_t)\right) .
\end{align*}
\else
\begin{align*}
  dX_t/dt = - \nabla f(X_t) - \frac{\eta}{2}\nabla^2 f(X_t) \nabla f(X_t)
 - \eta^2 \left( \frac{1}{3} \nabla^2 f(X_t) \nabla f(X_t) +
 \frac{1}{12} \nabla f(X_t)^\T \nabla^3 f(X_t) \nabla f(X_t)\right) .
\end{align*}
\fi

\subsection{Principal Flow}
\label{sec:pre_pf}
While theoretically, one could solve for all component function $g_i$'s to
achieve an arbitrary target error order, the resulting modified flow quickly
becomes cumbersome due to the involvement of higher-order gradients. Moreover,
the resulting power series (w.r.t. the exponent of $\eta$) may not even be
convergent \citep{hairer2006geometric}. In light of this
limitation of BEA, \citet{rosca2022continuous} suggested the following term
selection strategy:
\begin{principle}[Term selection of PF, BEA] \label{principle_PF}
    In BEA, when solving $g_i$'s in \Cref{eq:bea}, derivatives of $f(x)$ up to
    the second order, i.e., terms containing $\nabla^p f(x)$ for $p \geq 3$,
    are omitted.
\end{principle}

With this term selection scheme, the resulting power series not only converges
(with convergence radius $\eta\|\nabla^2 f(x)\| = 1$) but admits a concise form.
The corresponding modified equation is:
    \begin{align} \label{eq:PF}
    dX_t = 
    U(X_t) \frac{\log(I - \eta \Lambda(X_t))}{\eta \Lambda(X_t)} U(X_t)^\T \nabla f(X_t) d t,
    \end{align}
where for any fixed $x$, we denote $\nabla^2 f(x) = U(x) \Lambda(x)
U(x)^\T$ the eigen-decomposition of $\nabla^2 f(x)$.
When the stepsize is large, the logarithmic
function could result in complex numbers, as elements of $I-\eta \Lambda$ might
be negative. Operating in the complex space is beneficial for modeling the
divergent and oscillatory behaviors of gradient descent~\citet{rosca2022continuous}.
As only higher-order gradients are omitted in \Cref{principle_PF}, PF provides
an \emph{exact} characterization of the discrete GD dynamics on quadratic
objectives.

\section{Stochastic Backward Error Analysis}
\label{sec:SBEA}

We extend BEA to the stochastic domain, termed Stochastic Backward Error Analysis (SBEA),
providing a systematic approach for constructing SDEs that approximate stochastic discrete-time processes. 
While we focus on the analysis for SGD, similar argument can be generalized to other dynamics.
\subsection{Differences between BEA and SBEA}
Compared with BEA, there are two major modifications in
SBEA. First, SBEA includes an additional Brownian motion term in its ansatz, compared to the one in \Cref{eq:bea}, with the diffusion
coefficients to be determined; Second, to evaluate the quality of the
approximation, SBEA adopts the weak approximation error (see
\Cref{def:weak_approx}) as a metric in the distributional sense.

\paragraph{(1) Ansatz with Brownian motion.} In BEA,
we define the modified equation with a series of functions
$\{g_i\}_{i\in\mathbb{N}}$. To further model the randomness in SGD, we introduce another series $\{h_i\}_{i\in\mathbb{N}_{>0}}$, $h_i:
\mathbb{R}^d \to \mathbb{R}^{d \times d}$, to represent the diffusion coefficients
of the Brownian motion. Specifically, we consider a hypothesis continuous-time
dynamics as follows:
\begin{align} \label{eq:SBEA_unknown}
    d X =&\ \left(g_0(X) + \eta g_1(X) + \eta^2 g_2(X) + \cdots\right) dt
         + \sqrt{\eta h_1(X) + \eta^2 h_2 (X) + \cdots} dW_t.
\end{align}
The rationale stems from the observation that the noise introduced in the ansatz
undergoes continuous modification by the drift term. 
Consequently, it is natural to employ a
similar power series representation for the diffusion coefficient\footnote{We do not include $h_0$ in the ansatz as SGD degenerates to gradient flow when $\eta \rightarrow 0$, which implies $h_0\equiv 0$.}. 
While one can truncate the series $\{h_i\}_{i\in\mathbb{N}}$ to simplify the
analysis, as we illustrate in \Cref{fig:two_mode} and theoretically justify in
\Cref{sec:exist_quad}, such truncation can compromise the approximation quality
and lead to misaligned escaping behaviors.

\paragraph{(2) Weak approximation error as quality metric.}
Our goal is to solve for the functions $g_i$'s and $h_i$'s so that the ansatz in \Cref{eq:SBEA_unknown} approximates the SGD dynamics accurately. Instead of seeking a path-wise alignment like BEA, we focus
on aligning the continuous-time and discrete-time dynamics in a distributional
sense, characterized by the following weak approximation error.
\begin{definition}[Weak Approximation Error] \label{def:weak_approx}
  Fixing $T > 0$, we say a continuous-time process $X(t)_{t\in[0, T]}$ is an order $p$ weak approximation 
  to the sequence $\{x_k\}$ if for any $u \in C_b^{2(p+1)}\left(\mathbb{R}^d\right)$,
  there exists $C > 0$ and $\eta_0 > 0$ that are independent of $\eta$ (but may
  depend on $T$, $u$ and its derivatives) such that for all $x\in\mathbb{R}^d$,
  \begin{align} \label{eq:weak approximation error}
    \Big|\mathbb{E} [u\left(x_k\right) | x_0 = x] - \mathbb{E}[u\left(X(k \eta)\right) | X(0) = x]\Big|
    \leq C \eta^p,
  \end{align}
  for all $k = 1, \ldots, \floor{T/\eta}$ and all $0 < \eta < \eta_0$.
\end{definition}
This metric is commonly used for analyzing discretization errors of SDE, as referenced
in textbooks~\citep{milstein2013numerical,kloeden2011numerical} and
various studies on modeling SGD behavior with SDEs~\citep{li2017stochastic,
hu2017diffusion, feng2019uniform, li2022uniform}.
\added{
Specifically, we note that the error bound is uniform across all SGD iterations within
$\lfloor T/\eta \rfloor$ steps, rather than merely bounding the last iterate.
A small weak approximation error means that SGD and \oursde{} are
indistinguishable when evaluated with a broad class of test functions along
the entire trajectory.
}

\begin{remark}
We highlight the difference between SBEA and the weak backward error analysis
introduced by \citet{doi:10.1137/110831544}. The latter provides a
framework for constructing a sequence of modified Kolmogorov generators with
increasing orders to approximate discrete-time processes. However, since the high-order Kolmogorov
generators involve high-order gradients, it \emph{does not} yield an
equivalent SDE representation (recall that the Kolmogorov generator corresponding to an Itô
SDE contains only second-order gradients). Consequently, their framework is
fundamentally different from our approach in terms of the final outcome.
\end{remark}

\subsection{Procedures in SBEA}
We determine the components $\{h_i\}_{i\in\mathbb{N}}$ and
$\{g_i\}_{i\in\mathbb{N}}$ by first calculating the semi-group expansions of the
two conditional expectations involved in \Cref{eq:weak approximation error}
after one step, i.e. $k=1$, and then match the resulting terms according to the
power of $\eta$. Such a one-step approximation analysis can be translated to the weak approximation error guarantee via a
martingale argument, as we will elaborate in \Cref{sec:approx}.

\paragraph{Semi-group expansions.}
We expand the conditional expectation of the test function $u$ after one SGD step as:
\begin{align} \label{eq:SBEA_discrete}
  \mathbb{E} [u\left(x - \eta \nabla F(x; \xi) \right) \mid x] = u(x) + \eta \sum_{p=0}^\infty \eta^p \Phi_p(x),
\end{align}
where we use $\Phi_p$ to denote the aggregation of terms with $\eta^{p+1}$ in the expansion.
Simple calculation shows 
\begin{equation}
  \Phi_0(x) = - \nabla f(x)^\T \nabla u(x), \ \Phi_1(x) = \frac{1}{2}\nabla^2 u(x) : \left(\nabla f(x) \nabla f(x)^\T + \Sigma(x)\right), \ \cdots \added{,}
\end{equation}
where $\Sigma(x)$ denotes the covariance matrix of the stochastic gradient at $x$:
\begin{equation}
    \Sigma(x) = \mathbb{E}\left[\left(\nabla F(x; \xi) - \mathbb{E}[\nabla F(x; \xi)]\right)\left(\nabla F(x; \xi) - \mathbb{E}[\nabla F(x; \xi)]\right)^\T\right].
\end{equation}

For the continuous-time system, 
by the semi-group expansion~\citep{hille1996functional}, the conditional expectation of a test function $u$ under the SDE model yields:
\ifdefined\mpversion
\begin{align}
  \begin{split}
  \label{eq:SBEA_continuous}
   \mathbb{E} [u\left(X(\eta)\right) \mid X(0)=x] &= \left(e^{\eta \mathcal{L}} u\right) (x) \\
      &= u(x) + \eta \mathcal{L} u(x) 
       + \frac{1}{2} \eta^2 \mathcal{L}^2 u(x) + \frac{1}{6} \eta^3 \mathcal{L}^3 u(x) + \cdots,
  \end{split}
\end{align}
\else
\begin{align} \label{eq:SBEA_continuous}
   \mathbb{E} [u\left(X(\eta)\right) \mid X(0)=x] &= \left(e^{\eta \mathcal{L}} u\right) (x)
      = u(x) + \eta \mathcal{L} u(x) 
       + \frac{1}{2} \eta^2 \mathcal{L}^2 u(x) + \frac{1}{6} \eta^3 \mathcal{L}^3 u(x) + \cdots,
\end{align}
\fi
where $\cL$ is corresponding the infinitesimal generator~\citep{sarkka2019applied}.
For the SDE described in \Cref{eq:SBEA_unknown}, the generator is given explicitly by
\ifdefined\mpversion
\begin{align*}
  \mathcal{L} &= (g_0 + \eta g_1 + \cdots) \cdot \nabla + \frac{1}{2}(\eta h_1 +  \cdots) : \nabla^2 \\
              &= g_0 \cdot \nabla + \sum_{i=1}^{+\infty} \eta^i \left(g_i \cdot \nabla + \frac{1}{2} h_i : \nabla^2\right) = \sum_{i=0}^{+\infty} \eta^i \cL_i,
\end{align*}
\else
\begin{align*}
  \mathcal{L} = (g_0 + \eta g_1 + \cdots) \cdot \nabla + \frac{1}{2}(\eta h_1 +  \cdots) : \nabla^2 
              = g_0 \cdot \nabla + \sum_{i=1}^{+\infty} \eta^i \left(g_i \cdot \nabla + \frac{1}{2} h_i : \nabla^2\right) = \sum_{i=0}^{+\infty} \eta^i \cL_i,
\end{align*}
\fi
where for the ease of notation, we denote
\begin{equation*} 
    \cL_0 \coloneqq  g_0 \cdot \nabla \ \text{and}\
    \cL_i \coloneqq  g_i \cdot \nabla + \frac{1}{2} h_i : \nabla^2 \ \text{for}\ i \geq 1.
\end{equation*}
Gathering terms with power $\eta^{p}$ in \Cref{eq:SBEA_continuous} leads to the following lemma:
\begin{lemma} \label{lemma:semi_expand}
    Define the function 
    \begin{equation} \label{eq:expand_p}
      \Psi_{p}(x) = \sum_{n=1}^{p+1} \frac{1}{n!} \sum_{\{l^n_i \geq 0\}_{i=1}^n:\sum_{i=1}^n l^n_i = {p+1}-n} \cL_{l^n_1} \cL_{l^n_2} \cdots \cL_{l^n_n} u(x),
    \end{equation}
    where $\cL_{l^n_j} \cL_{l^n_{j+1}}$ denotes the composition of the operators $\cL_{l^n_j}$ and $\cL_{l^n_{j+1}}$.
    We have that 
    \begin{equation} \label{eq:SBEA_continuous_match}
        \left(e^{\eta \mathcal{L}} u\right) (x) = u(x) + \eta \sum_{p=0}^\infty \eta^p \Psi_p(x).
    \end{equation}
  
\end{lemma}
Next, we determine the components $g_p$'s and $h_p$'s in an iterative manner.

\paragraph{Identify components by iteratively matching terms.}
Our aim is to determine $g_i$'s and $h_i$'s by aligning the terms from
the discrete-time expansion (\Cref{eq:SBEA_discrete}) with those from
continuous time (\Cref{eq:SBEA_continuous_match}), according to the order of $\eta$.
Note that every application of the generator $\mathcal{L}$ on a function generates terms with all orders of $\eta$. Consequently, $p$-th order terms in the continuous-time expansion become complicated as $p$ grows, as evident in Lemma \ref{lemma:semi_expand}.
Fortunately, we can identify $g_p$'s and $h_p$'s in an iterative manner.
In the following discussion, $g_i$ and $h_i$ for $i=0,1,\dots,p-1$ are treated as \emph{known} from our iterative construction, and we solve the \emph{unknowns} $g_p$ and $h_p$ by matching $\Phi_p$ and $\Psi_p$. We make the following two observations:
\begin{itemize}
  \item \emph{$\Psi_{p}(x)$ (\Cref{eq:expand_p}) is linear on $g_p$ and $h_p$.} Recall that $g_p$ and $h_p$ are only contained linearly in $\cL_{p} u$, which appears only once in $\Psi_{p}(x)$ by choosing $n =1$ and $l^1_1 = p$.
  \item \emph{Terms in $\Psi_{p}(x)$ with $n > 1$ contains only known terms.} For
      $n > 1$, the terms in \Cref{eq:expand_p} will only include $g_i$ and $h_i$
      for $i\leq p-1$, since for any choice of the indices $\{l^{n+1}_j\}_{j=1}^n$,
      we must have $l_j^n \leq p-1$ if $n > 1$.
    These terms are already determined by our iterative strategy.
\end{itemize}

Following \citet{rosca2022continuous},
we also explore various term selection
strategies, starting with the exact matching case.
\begin{principle}[Term selection of exact-matching, SBEA] \label{principle_SBEA_exact}
    When determining the component functions $g_p$ and $h_p$, all terms in $\Phi_p$ from \Cref{eq:SBEA_discrete} must match the ones in $\Psi_p$ from \Cref{eq:SBEA_continuous_match}. In particular, given the arbitrariness of the test function, the coefficients of  $\nabla^p u$ ($p\geq 1$) should all be matched.
\end{principle}

With the above principle, we identify the unknowns, $g_p$ and $h_p$, for $p\leq 1$.
\begin{remark}[Recovering \sme-2 by SBEA] \label{remark:example_SBEA}
By matching $\Psi_0$ and $\Phi_0$, we
deduce that $g_0(x) = -\nabla f(x)$. 
By matching $\Psi_1$ with $\Phi_1$,
we solve that $g_1(x) = - \frac{1}{2} \nabla^2 f(x) \nabla f(x)$ and $h_1(x) = \Sigma(x)$. This leads to the formulation of
the \sme-2~(\Cref{eq:SDE2}).
\end{remark}

\subsection{Exact Matching Fails for Order $\eta^{3}$ in SBEA} \label{sec:overdetermination}

We have shown that for $p=0, 1$, it is possible to exactly solve $g_p$ and $h_p$.
However, for $p \geq 2$, adhering to \Cref{principle_SBEA_exact} can be infeasible.
Specifically,
following \Cref{principle_SBEA_exact} leads to an over-determined system when determining the unknown component functions $g_p$ and $h_p$ for $p\geq 2$.
Recall that the test function $u$ is chosen arbitrarily and let us consider all the possible occurrence of $\nabla^q u(x)$, $q=1, 2, \ldots$, in $\Phi_p$ and $\Psi_p$. 
\begin{itemize}
    \item For the continous-time dynamics, the term $\Psi_{p}$ in \Cref{eq:expand_p} can contain terms with factors $\nabla^q u(x)$ for $q$ ranging from $1$ to
$p+1$. For example, the term $\nabla^{p+1} u(x)$ can be
generated from \Cref{eq:expand_p} with the choice of $n=p+1$ and $l_j^n = 0$ for all $j\in\{1,\ldots, n\}$; the term $\nabla u(x)$ can be generated with the choice $n=1$ and $l_1^1 = 0$.
    \item For discrete-time expansion as per \Cref{eq:SBEA_discrete}, the
term $\Phi_p$ involves only $\nabla^{p} u(x)$. 
\end{itemize}

For exact matching, the coefficient of $\nabla^{p} u$ in $\Psi_p$ must match that in $\Phi_p$; besides, for any $q \neq p$, the coefficient of $\nabla^{q} u$ in $\Psi_p$ must be zero. This leads to an over-determined system comprising $p+1$ conditions for $2$ free variables ($g_p$ and $h_p$) when $p\geq 2$. In other words, the above iterative construction of the components $g_p$ and $h_p$ fails for $p\geq 2$.
The infeasibility result highlights a drastic difference between SBEA and BEA.
\citet{shardlow2006modified} also noted such infeasibility when attempting to
find stochastic modified flows to achieve higher orders of weak approximation
guarantees.

Inspired by PF, we introduce a relaxed term selection scheme, compared to \Cref{principle_SBEA_exact}, so that only a subset of terms in $\Psi_p$ (the continuous-time dynamics) and $\Phi_p$ (the discrete-time dynamics) match. Importantly, this term selection scheme should ensure that the resulting continuous-time dynamics still preserve the favorable escaping behaviors of the discrete-time dynamics of interest. 
This will be the focus of the following section.

\section{\oursdefull}
\label{sec:def_SPF}
In this section, we first derive the proposed SDE, \oursde{}, to model SGD
using the idea of SBEA. Then we discuss sufficient conditions for its
well-posedness.

Due to the lack of degree of freedom, exactly matching all terms in $\Phi_p$ and $\Psi_p$ for $p \geq 2$ is in general infeasible.
Instead, we adopt the following term selection principle.

\begin{principle}[Term selection in \oursde, SBEA] \label{principle_hessian_aware}
    When determining the components $g_p$ and $h_p$ in SBEA, ignore the terms that involve $\nabla^{(r)} f(x), \nabla^{(s)} u(x)$ and $\nabla^{(m)} \Sigma(x)$ for $r, s\geq 3$ and $m \geq 1$.
\end{principle}

Under this principle of term selection, we obtain the following result.
\begin{lemma} \label{lemma:SBEA_conditions}
  Under \Cref{principle_hessian_aware}, when applying SBEA on SGD, we have the following results:
  \begin{enumerate}
      \item The component functions $g_p$ and $h_p$ in \Cref{eq:SBEA_unknown} are uniquely determined;
      \item The component functions $g_p$ and $h_p$ in \Cref{eq:SBEA_unknown} admit the form
  \begin{align}
    g_p(x) =&\ c_p \cdot \left(\nabla^2 f(x)\right)^p \nabla f(x), \label{eq:gp} \\
    h_p(x) =&\ 
   \sum_{k = 0}^{p-1} a_{k, p-1-k} \cdot \left(\nabla^2 f(x)\right)^{k} \Sigma(x)
   \left(\nabla^2 f(x)\right)^{p-1-k}; \label{eq:hp}
  \end{align}
        \item The coefficients $\{a_{s, m}\}$ and $\{c_s\}$ in $\{g_p\}$ and $\{h_p\}$ satisfy
  \begin{gather*}
    \frac{\log(1-x)}{x} = \sum^{+\infty}_{s = 0} c_s x^s \text{ and }
   \frac{\log(1-x)(1-y)}{x y - (x + y)} = \sum^{+\infty}_{s, m \geq 0} a_{s, m} x^s y^m.
  \end{gather*}
  \end{enumerate}
\end{lemma}
We now derive the limit of the power series defined in \Cref{eq:SBEA_unknown}.

  \begin{theorem} \label{lemma:key}
  Let the components $\{g_p\}$ and $\{h_p\}$ be determined by the SBEA framework applied on SGD under \Cref{principle_hessian_aware}.
  Denote the limits
  \begin{equation} \label{eq:drifting and diffusion}
    b(x) = \sum_{p\geq 0} \eta^p g_p(x) \text{\ and\ } \mathcal{D}(x) = \sum_{p\geq 1} \eta^p h_p(x).
  \end{equation}
  Under the assumption that $\eta < 1/ \norm{\nabla^2 f(x)}$,
  we have that the limits $b(x)$ and $\mathcal{D}(x)$ exist with
  \begin{gather} \label{eq:HASME-drifiting-term}
      b(x) = U(x) \frac{\log(I - \eta \Lambda(x))}{\eta \Lambda(x)} U(x)^\T \nabla f(x), \\
 \label{eq:sPF_diff}
      \mathcal{D}(x) = U(x) S(x) U(x)^\top, \text{such that}\; \left[S(x)\right]_{i,j} = \frac{\left[U^\T \Sigma U \right]_{i,j} \log (1-\eta \lambda_i)
  (1-\eta \lambda_j)}{\eta \lambda_i \lambda_j - (\lambda_i + \lambda_j)},
  \end{gather}
  where $U(x)$ and $\Lambda(x)$ are defined through the eigen-decomposition $\nabla^2 f(x) = U(x) \Lambda(x) U(x)^\T$.
  The diagonal elements of $\Lambda(x)$ are denoted by $\lambda_i(x)$. For
  conciseness, we omit the dependence of $\lambda_i$, $U$ and $\Sigma$ on $x$.
\end{theorem}
The derivation of the above theorem is highly nontrivial, and represents one of the major contributions of our paper.
Its proof includes several critical steps:
(1) Identifying the structure of the components $g_p$ and $h_p$ in the SBEA
ansatz; (2) Determining the coefficient of the components $g_p$ and $h_p$; (3)
Computing the limit of the resulting power series, which we defer to \Cref{sec:construct_sPF}.

The following lemma provides sufficient conditions for the existence of the diffusion coefficient $\sqrt{\mathcal{D}(x)}$.
\begin{lemma} \label{prop:SPF_exist}
  $\mathcal{D}(x)$ is positive semi-definite at point $x$ if either of the
  following two conditions holds:
  \begin{enumerate}
    \item The Hessian $\nabla^2 f(x)$ and covariance matrix $\Sigma(x)$ commute,
      and $\eta < 1/\norm{\nabla^2 f(x)}$.
    \item The covariance matrix $\Sigma(x)$ is positive definite, and $\eta$ satisfies
    \begin{equation} \label{eq:bar_eta}
        \eta \leq 
        \frac{1}{\norm{\nabla^2 f(x)}} \min\left\{1 - \sqrt{1 - \frac{\lambda_{\min}\left(\Sigma(x)\right)}{\sqrt{d} 
      \lambda_{\max}\left(\Sigma(x)\right)}}, 1 - \frac{\sqrt{2}}{2}\right\},
    \end{equation}
    where $\lambda_{\max}(\cdot)$ and $\lambda_{\min}(\cdot)$ denote the maximum and minimum eigenvalues of a matrix.
  \end{enumerate}
\end{lemma}
The validity of the above two conditions is discussed at the end of this subsection.
With \Cref{lemma:key} and \Cref{prop:SPF_exist}, we now present the proposed SDE model.
\begin{mdframed}
\begin{definition} \label{def:SPF}
  Let the assumptions in \Cref{lemma:key} and \Cref{prop:SPF_exist} hold for any $x$, then the drift term $b(x)$
    and diffusion coefficient $D(x) \coloneqq \sqrt{\mathcal{D}(x)}$ in \Cref{eq:drifting and diffusion} are
    well-defined. We denote the SDE as the \oursdefullh~(\oursde):
\begin{gather}
  d X_t = b(X_t) d t
  + D(X_t) d W_t. \label{eq:sPF}
\end{gather}
\end{definition}
\end{mdframed}
Comparing with \sme-1\ in \Cref{eq:SDE1} and \sme-2\ in \Cref{eq:SDE2}, both the drift term and the diffusion coefficient of \oursde\ incorporate the Hessian information,
hence the name.

Notice that the drift term $b(x)$ of \oursde\ exactly corresponds to PF defined in \Cref{eq:PF}. 
A naive idea would be to combine PF with the conventional diffusion coefficient $\sqrt{\eta \Sigma}$ as in \sme-1 and \sme-2. The following remark comments on the drawback of such a construction.
\begin{remark}[A naive SDE derived from PF]
    Adding the diffusion coefficient $\sqrt{\eta \Sigma}$ to the PF in \Cref{eq:PF} gives rise to the following SDE:
\begin{align} \label{eq:PFwoC}
dX_t = 
U(X_t) \frac{\log(I - \eta \Lambda(X_t))}{\eta \Lambda(X_t)} U(X_t)^\T \nabla f(X_t) d t + \sqrt{\eta \Sigma(X_t)} dW_t,
\end{align}
which we refer to as \pfsfull~(\pfs).
Such a straightforward combination could lead to escaping behaviors different
from SGD. For example, when considering a
multi-mode function, as illustrated in \Cref{fig:two_mode}, SGD and our
proposed \oursde\ can easily escape the local minimum
and find the global one, whereas \pfs\ remains trapped in the valley of the
initial local minimum.
A theoretical comparison of \oursde\ and \pfs\ on general objective functions is provided in \Cref{sec:fine_grained_error_analysis}.
\end{remark}
\begin{remark}[Comparison with existing \sme s]
  By integrating Hessian into its diffusion coefficient, \oursde\ provides a
  nuanced correction to the SDE noise.
  This adjustment is crucial for capturing the true dynamics of SGD, even when
  considering additive, state-independent noise models, such as $\nabla F(x,
  \xi) = \nabla f(x) + \xi$, where $\xi \sim \mathcal{N}(0, \Sigma)$.
  While discrete-time models treat noise as state-independent, in continuous
  time, noise is introduced at infinitesimal time intervals subject to
  immediate transformation by the drift term, which includes Hessian information.
  The necessity of incorporating the Hessian matrix into the diffusion term for
  corrections is also underscored by the failure cases of existing
  SDE models discussed in \Cref{sec:exist_quad}.
\end{remark}

\begin{remark}[Discussion on conditions of \Cref{prop:SPF_exist}]
The first condition states that the eigenvectors of $\nabla^2 f(x)$ and
$\Sigma(x)$ are aligned. In the vicinity of minima, this condition finds support from both theoretical
analysis and empirical evidence, even in deep
learning. Theoretically, \citet{jastrzebski2017three} showed that $\Sigma(x^*) \approx \nabla^2 f(x^*)$
when the model fits all the data at $x^*$,
which is further empirically verified by \citet{xie2020diffusion}.
For mean-square loss, \citet{mori2021power} and \citet{paquette2022homogenization} derived $\Sigma(x) \approx \frac{2f(x)}{B} \nabla^2 f(x^*)$
near local minima $x^*$, where $B$ is the mini-batch size.
\citet{wang2024theoretical} further theoretically justified the approximation
 for nonlinear networks.
In addition, the approximation $\Sigma(x) \approx \nabla^2 f(x^*)$ is commonly used in
local escaping analysis of
SGD~\citep{zhu2019anisotropic,xie2020diffusion,xie2022adaptive,ibayashi2023does}.
Our requirement here is more relaxed, as we only need the eigenvectors to be
the same, regardless of eigenvalues. 

\end{remark}

\begin{figure}[t]
  \centering
  \begin{subfigure}{0.4\textwidth}
    \includegraphics[width=\linewidth]{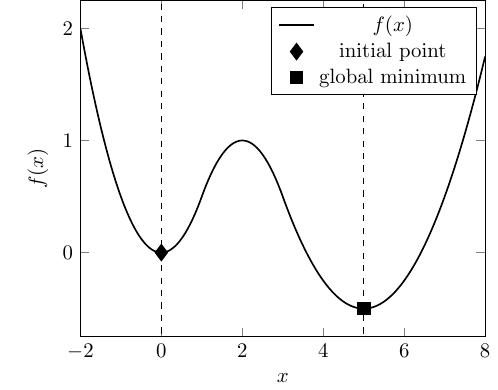}
    \caption{function $f(x)$}
      \label{fig:two_mode_func}
  \end{subfigure}
  \hspace{1cm}
  \begin{subfigure}{0.4\textwidth}
    \includegraphics[width=\linewidth]{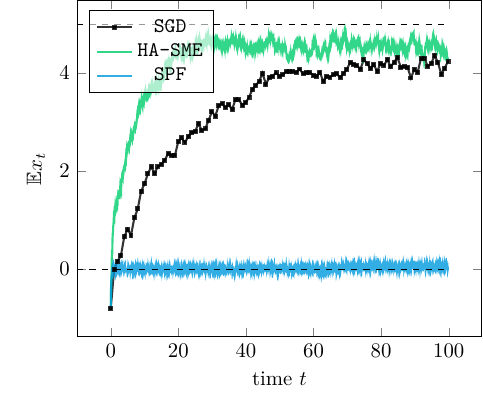}
      \caption{expected value of $x_t$ through time}
      \label{fig:two_mode_func_x}
  \end{subfigure}
  \caption{The illustration depicts the failure case for \pfs, a naive stochastic extension of PF. The objective function is
defined as $f(x) = \frac{1}{2} x^2$ for $x < 1$; $-\frac{1}{2}(x-2)^2 + 1$ for
$1 \leq x < 3$; and $\frac{1}{4}(x-5)^2 - \frac{1}{2}$ for $x \geq 3$.
We use dashed lines to indicate the two minima.
The initial point is set to 0, and we set $\eta = 0.999$ and the noise variance to
1 for all methods. We run the simulation 100 times to estimate the expectation.
SGD and \oursde\ escape the initial minimum and arrive near
the global minimum, while \pfs\ stays around the initial one.
A theoretical justification for the failure of \pfs\ is presented in \Cref{sec:exist_quad}.}
  \label{fig:two_mode}
\end{figure}

To guarantee the existence and uniqueness of a solution for an SDE, it is often
sufficient to impose regularity conditions on both the drift and diffusion
terms. In this context, we show that, when $\Sigma$ is positive definite,
conditions in \Cref{prop:SPF_exist} lead to the well-posedness of the proposed
\oursde\ model.
\begin{theorem} \label{prop:well_posed}
  Assume $f \in C^3_b\left(\mathbb{R}^d\right)$, \added{$\Sigma \in C^1_b \left(\mathbb{R}^{d \times d}\right)$}
  being positive definite,
  and that at least one of the two conditions
  in \Cref{prop:SPF_exist} is satisfied everywhere.
  \oursde\ has a unique strong solution.
\end{theorem}

\begin{remark}
  While for general smooth objectives, \Cref{prop:SPF_exist} requires a small stepsize,
    in the particular case of quadratic functions with a constant noise covariance, the requirement of small stepsize is not necessary, as will be discussed in \Cref{sec:quad}.
    In this case, the diffusion coefficient is constant, and the drift term depends linearly on $X_t$, naturally fulfilling the Lipschitz criterion.
    This is sufficient for proving the well-posedness of SDEs,
    independent of the stepsize.
\end{remark}

\section{Approximation Error Analysis of \oursde}
\label{sec:approx}

In this section, we analyze the approximation errors of \oursde{} in approximating SGD on general smooth functions.
Our first result establishes the order 2 weak approximation for \oursde{}, as defined in
\Cref{def:weak_approx}, matching the order-best guarantee in the literature.
Subsequently, we conduct a more fine-grained analysis, elucidating the explicit
dependence of the approximation error on the smoothness parameter of the
objective function. We observe a significant improvement of \oursde{} over the
existing \sme\ models. In particular, for convex
objectives, the leading error term of \oursde{} is independent of the smoothness
parameter.
\added{
To validate our theoretical findings, we conduct experiments on neural networks,
showing that \oursde{} more accurately approximates SGD dynamics compared to
existing SDE models.
}

\added{
In addition to directly establishing weak approximation error guarantees, we
also demonstrate the accuracy of \oursde{} through an escape time analysis in
\Cref{sec:escape_saddle}, showing that \oursde{} more reliably captures the
escape time near saddle points compared to \sme-2.
}

\subsection{Weak Approximation Error Guarantee}
\label{sec:weak_approx_theorem}

Below, we show a weak approximation error guarantee on sufficiently
smooth functions.

\begin{theorem} \label{thm:weak_approx}
  Assuming for any $\xi$, $F(\cdot; \xi) \in C^7_b\left(\mathbb{R}^d\right)$, 
  \oursde{}
  is an order 2 weak approximation of SGD.
\end{theorem}
{
\begin{remark}
  For all approximation error guarantees developed in this section, we assume
  that the (high-order) gradients of $F$ are uniformly bounded. This assumption is common
  in the analysis of weak approximation errors for SDEs that approximate
  SGD~\citep{hu2017diffusion, feng2017semi, feng2019uniform}.
Another common assumption is to consider polynomially bounded (high-order) gradients~\citep{li2017stochastic,li2019stochastic}.
However,
this leads to weaker, point-wise approximation guarantees with respect to the
initial point; specifically, the constant $C$ in \Cref{eq:weak approximation
error} depends on the initial point $x$. In contrast, our guarantees are uniform over all initial
points.

To relax the global boundedness assumption, one could restrict the approximation analysis to a compact set, where the (high-order) gradient bounds can be replaced with their local versions. To be more specific about how such results can be achieved, we outline two possible approaches:
1) modifying \oursde{} so that it remains confined within a bounded domain, or
2) studying the approximation error up to the hitting time of the boundary of a compact set.
The first approach has been employed in \citet{li2022uniform}, where, under mild
assumptions on the objective function, SGD itself remains within a compact set.
Consequently, it can be shown that the modified SDE retains the same
approximation guarantees as the original SDE, but with respect to local problem
parameters.
\end{remark}
}
An order 2 weak approximation is the order-best approximation guarantee known for SGD
in existing literature.
Typically, achieving weak approximation results requires bounding the (high-order)
derivatives of the drift and diffusion terms of the SDE.
While this condition is readily met for existing \sme{}s,
ensuring boundedness for \oursde{} presents a challenge,
as the drift and diffusion terms in \oursde{} are defined as limits
of power series and involve logarithmic components.
Our proof builds upon the following lemma.
\begin{lemma}[Regularity of the drift term and diffusion coefficient] \label{lemma:regul_drift_diff}
  Consider a fixed $n \geq 0$.
  Assume for that any $\xi$, we have $F(\cdot; \xi) \in C^{n+2}_b(\mathbb{R}^d)$.
  Then, there exists
  a constant $\eta_0 > 0$ such that, for any $\eta < \eta_0$,
  $\max_{0 \leq i \leq d} \norm{[b(x)]_i}_{C^n} < \infty$
  and $\max_{0 \leq i, j \leq d}\norm{\left[D(x)D(x)^\T\right]_{i, j}}_{C^n} < \infty$.
\end{lemma}
\begin{remark}
    A similar proof shows that \pfs{} also admits the order 2 weak approximation error.    
Although \sme-2, \pfs{}, and \oursde{} share the same order of weak approximation error, as demonstrated in \Cref{sec:quad}, these models can have drastically different behaviors near critical points.
This discrepancy between the practice and the above theory arises because the classical analysis solely emphasizes the dependence on the stepsize, while neglecting other crucial factors, such as the norm of the
Hessian matrix. 
To better differentiate these models, a more fine-grained analysis is necessary,
where the dependence on the problem-dependent parameters is explicitly accounted for.
\end{remark}

\subsection{Fine-Grained Error Analysis with Hessian Dependence} \label{sec:fine_grained_error_analysis}

\added{
The preceding weak approximation results primarily focus on the dependence on
the stepsize $\eta$, consistent with earlier research on SDEs in the context of
\sme-1 and \sme-2~\citep{li2017stochastic, hu2017diffusion, feng2017semi,
feng2019uniform}.
This suggests that SDEs provide accurate approximations when the stepsize is
small.
However, a more relevant scenario in practice involves larger stepsizes, which
are inversely proportional to $\lambda\coloneqq \sup_{x \in \mathbb{R}^d} \norm{\nabla^2 f(x)}$.
This regime is
supported by the ``edge of stability'' phenomenon observed in empirical studies
of deep neural networks~\citep{cohen2020gradient, cohen2022adaptive}. These
studies have shown that GD tends to converge to solutions
where the maximum eigenvalue of the Hessian matrix approaches $2/\eta$. A
similar behavior has been noted in SGD, albeit
with a smaller final eigenvalue. This phenomenon underscores the significant
practical and theoretical importance of considering stepsizes on the order of
$\Theta(1/\lambda)$.
\red{
Consequently, prior approximation error estimates may be imprecise if they
consider only the dependence on $\eta$ while ignoring the role of $\lambda$.
}
}

Our subsequent analysis explicitly examines the dependence of the error on
\added{problem-dependent parameters, the smoothness parameter
$\lambda$ and
Lipschitz parameter $s \coloneqq \sup_{x \in \mathbb{R}^d} \norm{\nabla f(x)}$.}
Through this detailed
analysis, we differentiate the approximation guarantees of \oursde{} from \sme-2 and \pfs{}, highlighting its
advantage in modeling SGD.

\begin{theorem} \label{thm:weak_approx_lambda}
  Assume for any $\xi$, we have $F(\cdot; \xi) \in C_b^8\left(\mathbb{R}^d\right)$,
  and for any $0 \leq i, j \leq d$, we have $\left[\Sigma(\cdot)\right]_{i,j} \in C_b^6\left(\bR^d\right)$.
  Let $X(t)$ be the stochastic process
  described by \oursde{} and $\{x_k\}$ be the sequence generated by SGD.
  There exists $\eta_0 > 0$ such that for any $\eta < \eta_0$ and $T > 0$, it holds that for all $x\in\mathbb{R}^d$,
\ifdefined\mpversion
  \begin{gather}
    \begin{split} \label{eq:weak_approx_lambda_general}
    \sup_{k=1,\dots,\floor{T/\eta}}
    \abs{\mathbb{E} [u\left(x_k\right) | x_0 = x] - \mathbb{E}[u\left(X(k
    \eta)\right) | X(0) = x]} \\
    \leq \cO\left(\left(\eta^2 s^3 + \eta^3 s^4 \lambda^3\right) M(T) \right),
    \end{split}
  \end{gather}
\else
  \begin{gather}  \label{eq:weak_approx_lambda_general}
    \sup_{k=1,\dots,\floor{T/\eta}}
    \abs{\mathbb{E} [u\left(x_k\right) | x_0 = x] - \mathbb{E}[u\left(X(k
    \eta)\right) | X(0) = x]} \leq \cO\left(\left(\eta^2 s^3 + \eta^3 s^4 \lambda^3\right) M(T) \right),
  \end{gather}
\fi
  where for any $p \geq 1$, $M(T) \coloneqq \eta \sum_{k=0}^{\floor{T/\eta} - 1}
  \sum_{1 \leq \abs{J} \leq 8} \abs{D^J u^k(x)}_{\infty}$ with $u^k(x) = \Ep{u(x_k) \mid x_0 = x}$.
\end{theorem}
Note that $M(T)$ 
characterizes how the regularity of the test
function $u$ deteriorates along the SGD trajectory.
By definition, this quantity is solely determined by the SGD dynamics, thus independent of the SDE models.

Similarly, we derive the following results for \sme-2 and \pfs{}.
\begin{theorem} \label{thm:weak_approx_lambda_other_SDEs}
  Under the same settings as $\Cref{thm:weak_approx_lambda}$,
\ifdefined\mpversion
  \begin{enumerate}
    \item when $X(t)$ is described by \sme-2, it holds that for all $x\in\mathbb{R}^d$,
  \begin{align*}
    \sup_{k=1,\dots,\floor{T/\eta}}
    \abs{\mathbb{E} [u\left(x_k\right) | x_0 = x] - \mathbb{E}[u\left(X(k
    \eta)\right) | X(0) = x]} \\
    \leq \cO\left(\left(\eta^2 \left(s^3 + s \lambda^2\right) + \eta^3 s^4 \lambda^3 \right) M(T) \right),
  \end{align*}
\item when $X(t)$ is described by \pfs{}, it holds that for all $x\in\mathbb{R}^d$,
  \begin{align*}
    \sup_{k=1,\dots,\floor{T/\eta}}
    \abs{\mathbb{E} [u\left(x_k\right) | x_0 = x] - \mathbb{E}[u\left(X(k
    \eta)\right) | X(0) = x]} \\
    \leq \cO\left( \left( \eta^2 \left(s^3 + \lambda \right) + \eta^3 s^4 \lambda^3 \right) M(T) \right).
  \end{align*}
  \end{enumerate}
\else
  \begin{enumerate}
    \item when $X(t)$ is described by \sme-2, it holds that for all $x\in\mathbb{R}^d$,
  \begin{align*}
    \sup_{k=1,\dots,\floor{T/\eta}}
    \abs{\mathbb{E} [u\left(x_k\right) | x_0 = x] - \mathbb{E}[u\left(X(k
    \eta)\right) | X(0) = x]} \leq \cO\left(\left(\eta^2 \left(s^3 + s \lambda^2\right) + \eta^3 s^4 \lambda^3 \right) M(T) \right),
  \end{align*}
\item when $X(t)$ is described by \pfs{}, it holds that for all $x\in\mathbb{R}^d$,
  \begin{align*}
    \sup_{k=1,\dots,\floor{T/\eta}}
    \abs{\mathbb{E} [u\left(x_k\right) | x_0 = x] - \mathbb{E}[u\left(X(k
    \eta)\right) | X(0) = x]} \leq \cO\left( \left( \eta^2 \left(s^3 + \lambda \right) + \eta^3 s^4 \lambda^3 \right) M(T) \right).
  \end{align*}
  \end{enumerate}
\fi
\end{theorem}

The above theorems implies that \oursde{} demonstrates a clear improvement in terms of dependency on $\lambda$. In
particular, it eliminates the dependence on $\lambda$ in the leading error term
involving $\eta^2$.

As is common in the finite-time approximation error analysis of
SDE~\citep{li2017stochastic,feng2017semi,hu2017diffusion}, we point out that $M(T)$
could exhibit exponential growth in terms of $T$ and $\lambda$ in general.
However, in cases when the
function is convex, the following result shows that these constants remain
independent of $\lambda$.

\begin{lemma} \label{lemma:strong_sc}
  Assuming that for any $\xi$, $F(\cdot; \xi) \in C^{9}_b\left(\bR^d\right)$ is 
  convex, i.e., $F(y; \xi) - F(x; \xi) \geq \nabla F(x; \xi)^\T (y-x)$ for any $x, y \in \bR^d$. 
  There exists a constant $\eta_0$ such that, for any $\eta < \eta_0$, it holds that
  \begin{align*}
    M (T) \leq \cO\left( \norm{u}_{C^8}\right),
  \end{align*}
  where $M(T)$ is defined in \Cref{thm:weak_approx_lambda}, and $\cO(\cdot)$ hides terms that do not
  depend on $\eta$, $\lambda$ or $s$.
\end{lemma}

The above lemma allows us to develop the following approximation guarantee.
\begin{theorem} \label{thm:weak_approx_sc}
  Fixing $T > 0$, assume that the test function satisfies $u \in C_b^8\left(\mathbb{R}^d\right)$, $F(\cdot;
  \xi) \in C_b^9\left(\mathbb{R}^d\right)$ is convex for any $\xi$,
  and for any $0 \leq i, j \leq d$, $\left[\Sigma(\cdot)\right]_{i,j} \in C_b^6\left(\bR^d\right)$.
  Let $X(t)$ be the stochastic process
  described by \oursde{} and $\{x_k\}$ be the sequence generated by SGD.
  There exists a constant $\eta_0 > 0$ such that for any $\eta < \eta_0$,
  it holds that for all $x\in\mathbb{R}^d$,
  \begin{align*} \sup_{k=1,\dots,\floor{T/\eta}}
    \abs{\mathbb{E} [u\left(x_k\right) | x_0 = x] - \mathbb{E}[u\left(X(k
    \eta)\right) | X(0) = x]} \leq \cO\left(\eta^2 s^3 \right).
  \end{align*}
\end{theorem}

Following a similar proof together with \Cref{thm:weak_approx_lambda_other_SDEs},
we can establish that the upper bounds
for \sme-2 and \pfs{} are $\cO\left(\eta^2 \left(s^3 + s \lambda^2 \right)\right)$ and
$\cO\left(\eta^2 (s^3 + \lambda)\right)$,
respectively.
This difference partially explains
the limitations of existing SDE models in capturing the escape dynamics as
illustrated in \Cref{fig:illustrate_escape,fig:two_mode}.

Intuitively, the advantage of \oursde{} is inherent in its construction by SBEA.
Even under \Cref{principle_hessian_aware} of term selection, \oursde{} incorporates infinite sequences of terms associated
with $\lambda$, $\left\{\left(\nabla^2 f(x)\right)^p\right\}_{p=0}^\infty$, in both its drift and diffusion terms.
This is in stark contrast to \sme-1, \sme-2, and \pfs{}, all of which truncate the error series at a certain
point. All three of these models can be recovered by SBEA with a proper truncated series of $\{g_p\}$ and $\{h_p\}$:
\sme-1 includes $g_0$, $h_0$, and $h_1$; \sme-2 includes $g_0$,
$g_1$, $h_0$, and $h_1$; and \pfs{} includes $\{g_i\}_{i=0}^{\infty}$, $h_0$, and
$h_1$.
We conjecture that the $\lambda$-dependence in the error analysis for \sme-2
and \pfs{} cannot be improved. This limitation arises because error terms that
exhibit such $\lambda$-dependence are inherently truncated during their
construction within the SBEA framework.
\begin{remark}
  Since \Cref{thm:weak_approx_sc} addresses the global point-wise approximation,
  the dependence on the global gradient norm of  $f$  appears unavoidable and
  cannot be neglected. However, to study the escaping behavior of SGD near a
  critical point, we believe that the global gradient norm  $s$  can be relaxed to
  its local version. This would be an interesting result because, in the
  vicinity of a critical point of a smooth function, the local gradient norm can
  be regarded as a negligible constant. Consequently, the leading term in
  \Cref{thm:weak_approx_sc} would be independent of both  $s$ and  $\lambda$
  (whereas errors of existing SDE proxies always depend on  $\lambda$). 
\end{remark}
\begin{remark} \label{remark:sc}
  If we further assume that $f(\cdot, \xi)$ is strongly-convex for any $\xi$,
  then $M(T)$ can be shown to be uniformly bounded w.r.t. $T$ (the proof is provided
  in \Cref{sec:apx_proofs_approx}), leading to
  a uniform-in-time approximation guarantee.
  A similar 
  guarantee for \sme-2 was established in the literature under the strong-convexity
  assumption~\citep{li2022uniform}, but without considering the explicit dependence on the problem parameters $s$ and $\lambda$. 

\end{remark}

{
\subsection{Numerical Experiments}

To validate the theoretical results and demonstrate the potential of \oursde{}
for modeling SGD on practical applications, we conduct numerical experiments comparing
\sme{}s and \oursde{} in terms of their alignment with SGD dynamics.
Specifically, we run both SGD and the SDE models on a two-hidden-layer neural
network with a sigmoid activation function and cross-entropy loss.
The network consists of 10 nodes per hidden layer, and we use the Iris
dataset~\citep{fisher1936use} for classification. Since computing the full
covariance for the typical mini-batch noise is computationally expensive,
particularly for SDEs where each SGD
iteration is subdivided into multiple numerical discretization steps, we
instead choose Gaussian noise with a covariance of
$\Sigma(x) \equiv \sigma^2 I$ as the gradient noise.

\begin{table}[]
  \centering
  \caption{Weak approximation error of SDE models approximating SGD}
  \label{tab:exp_weak_approx_error}
\begin{tabularx}{0.8\linewidth}{@{}XXXX@{}}
\toprule
SDE Models       & \sme{}-1 & \sme{}-2 & \oursde{}     \\ \midrule
\multicolumn{4}{c}{Fixing $\sigma$ to be $10^{-3}$ and changing $\eta$}                                               \\ \midrule
$\eta=0.1$       & $1.4 \times 10^{-3}$  & $5.8 \times 10^{-5}$  & \textbf{$ 4.6 \times 10^{-5}$} \\
$\eta=0.2$       & $3.0 \times 10^{-3}$  & $3.1 \times 10^{-4}$  & \textbf{$ 1.1 \times 10^{-4}$}          \\
$\eta=0.5$       & $8.5 \times 10^{-3}$  & $ 2.4 \times 10^{-3}$ & \textbf{$ 5.3 \times 10^{-4}$} \\ \midrule
\multicolumn{4}{c}{Fixing $\eta$ to be $0.5$ and changing $\sigma$}                                                     \\ \midrule
$\sigma=10^{-2}$ & $8.5 \times 10^{-3}$  & $ 2.5 \times 10^{-3}$ & \textbf{$ 4.9 \times 10^{-4}$} \\
$\sigma=10^{-1}$ & $ 4.1 \times 10^{-3}$ & $ 3.9 \times 10^{-3}$ & \textbf{$ 1.8 \times 10^{-3}$} \\ \bottomrule
\end{tabularx}
\end{table}

The training starts with the same random initialization for all models, and
we compute the objective function value at the first 10 SGD iteration and the
corresponding continuous timestamp $\eta k$. After averaging over 100 runs, we plot the
maximum absolute difference between the objective function values of SGD and the
SDE across all iterations, given by
$\max_k \left| \mathbb{E} f(x_k) - \mathbb{E} f(X(k\eta)) \right|$.
This measure corresponds exactly to the weak approximation error, with the test
function being the training objective.
As shown in \Cref{tab:exp_weak_approx_error}, \oursde{} consistently outperforms  
existing SDE models across all stepsize and noise level setups.  
\sme-1 exhibits the worst performance, as it is only a order-1 weak
approximation. For \sme-2, \oursde{} can achieve up to an order of magnitude
improvement when the stepsize is large,  which aligns with our theoretical
findings.
}

\section{Exact Recovery of SGD by \oursde\ on Quadratics}
\label{sec:quad}

The quadratic objective, despite its simplicity, holds significance in studying the
behaviors of SGD. 
Its relevance comes from not only its application in linear regression but also its use in
modeling the local curvature of complex models through second-order Taylor
expansions.
Moreover, insights from Neural Tangent Kernel~\citep{arora2019exact,jacot2018neural}
suggests that in regression tasks, the objectives of sufficiently wide
neural networks closely resemble quadratic functions.
The convergence of SGD on quadratic models has been extensively
studied, for example, for constant stepsizes~\citep{10.1214/19-AOS1850}
and stepsize schedulers~\citep{ge2019step,pan2021eigencurve}. Additionally,
the exploration of locally escaping behaviors of SGD often relies on a local quadratic
assumption~\citep{zhu2019anisotropic,hu2017diffusion,xie2020diffusion,xie2022adaptive,ibayashi2023does}.

In this section, we aim to compare different SDEs in terms of their ability to
approximate SGD on quadratic functions. 
\added{
We believe that accurately capturing the behavior of SGD on simple quadratic  
functions serves as a touchstone for any SDE model,
as local quadratic approximations frequently arise in optimization.
}
In this context, We first discuss the failure cases of
existing SDE proxies (including \pfs\ in
\Cref{eq:PFwoC}) and then illustrate how \oursde{} can accurately
match SGD under certain conditions.

\subsection{Failure Cases for Existing SDEs}
\label{sec:exist_quad}

We now offer a theoretical justification for the phenomena
observed in \Cref{fig:illustrate_escape,fig:two_mode} --- specifically, why \sme-1, \sme-2, and \pfs\ fail in modeling
SGD on quadratic functions. We restrict the noise in this analysis to
be additive and state-independent.
\begin{assumption}[quadratic objective] \label{assume:quad}
  The objective function $f$ is defined as $f(x) = \frac{1}{2} x^\T A x$,
  where $A \in \mathbb{R}^{d \times d}$ is a symmetric real matrix\footnote{We focus on the simply quadratic $\frac{1}{2} x^\T A x$, while noting that all
results in this section can easily generalize to any quadratic function of the
form $\frac{1}{2} x^\T A x + bx + c$.}. We denote its 
  eigen-decomposition by $A = U \Lambda U^\T$, where $U\in\mathbb{R}^{d\times d}$ satisfies $U^\top U = I_d$ and $\Lambda \in \mathbb{R}^{d\times d}$ is diagonal.
\end{assumption}
\begin{assumption}[Gaussian noise] \label{assume:add_indep_noise}
The stochastic gradient satisfies
\begin{align*}
  \nabla F(x, \xi) = \nabla f(x) + \xi, \quad \xi \sim \mathcal{N}(0, \Sigma),
\end{align*}
where $\Sigma$ is a constant positive semi-definite matrix. 
\end{assumption}
With state-independent noise, the aforementioned SDEs applied to
quadratics can be represented as Ornstein–Uhlenbeck (OU) processes.
For simplicity, this subsection considers
isotropic noise, i.e. $\Sigma = \sigma^2 I$ for some scalar $\sigma$.
We show that even within this highly simplified noise
setting, existing SDEs may struggle to accurately capture the escaping behaviors
exhibited by SGD.
The following proposition computes the distributions of the iterates of SGD and the SDE proxies.
\begin{proposition} \label{prop:other_SDE_quad}
  Under \Cref{assume:quad,assume:add_indep_noise} with $\Sigma = \sigma^2 I$ for some
  scalar $\sigma$,
  it holds that
\begin{enumerate}
  \item The iterates of SGD (\Cref{eq:sgd}) with stepsize $\eta_k \equiv \eta$ satisfy
    \begin{align*}
  x_k \sim \mathcal{N}\left(U\left(I - \eta \Lambda\right)^k U^\T x_0, \;
    \eta^2 \sigma^2 \sum_{m=0}^{k-1} U \left(I - \eta \Lambda \right)^{2m} U^\T
  \right).
    \end{align*}
  \item The solution of \sme-1 (\Cref{eq:SDE1}) satisfies
\begin{align*}
  X(x_0, t) \sim \mathcal{N}\left(\exp(-At) x_0, \;
  \eta \sigma^2 U \frac{I-\exp(-2\Lambda t)}{2 \Lambda} U^\T \right).
\end{align*}

  \item The solution of \sme-2 (\Cref{eq:SDE2}) satisfies
\begin{align*}
  X(x_0, t) \sim \mathcal{N}\left(\exp\left(-\left(A + \frac{\eta}{2}A^2\right)t\right) x_0, \;
  \eta \sigma^2 U \frac{I-\exp\left(-\left(2\Lambda + \eta \Lambda^2\right) t\right)}{2 \Lambda + \eta \Lambda^2} U^\T \right).
\end{align*}

  \item The solution of \pfs\ (\Cref{eq:PFwoC}), if $\eta < 1/ \|A\|$, satisfies
\begin{align*}
  X(x_0, t) \sim \mathcal{N}\left( U \left(1- \eta \Lambda \right)^{t/\eta} U^\T x_0, \;
  \eta^2 \sigma^2 U \frac{1 - \left(1 - \eta \Lambda \right)^{2t/\eta}}{-2\log \left(1-\eta \Lambda\right)} U^\T \right).
\end{align*}

\end{enumerate}
\end{proposition}

Now we delve into the analysis of different escaping scenarios.

\paragraph{Escaping Saddle Points in \Cref{fig:quad_saddle}} 
Consider the function defined in \Cref{fig:quad_saddle}, i.e.,
$f\left([x, y]^\T\right) = \frac{1}{2} \left(x^2 - y^2\right)$, where the saddle point
is at $(0, 0)$. According to
\Cref{prop:other_SDE_quad}, along the direction of $y$, the iterate $x_k$ generated by
SGD has a mean of $(1+\eta \lambda_2)^k y_0$ and a variance of $\eta^2
\sigma^2 \sum_{m=0}^{k-1}(1+\eta\lambda_2)^{2m}$, which drifts
away from $0$ for any $\eta > 0$. Thus, SGD can successfully escape for any stepsize $\eta > 0$.
However, for \sme-2, when $\eta \geq 2$, as $t \to \infty$, its stationary distribution
is a Gaussian distribution with mean 0 and covariance $\eta \sigma^2 \left(2\Lambda + \eta \Lambda^2 \right)^{-1}$.
In such a regime of constant stepsize, the dynamics of \sme-2
converges to a stationary distribution around $0$, while SGD
escapes exponentially fast. This observation confirms the numerical simulation
presented in \Cref{fig:quad_saddle}.

\paragraph{Escaping Minimum in \Cref{fig:quad_minimum}}
Consider the objective $f\left([x, y]^\T\right) = \frac{1}{2} \left(x^2 + y^2\right)$
where all eigenvalues of the Hessian are positive. 
From \Cref{prop:other_SDE_quad}, SGD will escape from $(0, 0)$ if $\eta >
2$ --- both the mean and covariance of the iterates grow
exponentially. However, for \sme-1 and \sme-2, their stationary distributions
are zero-mean Gaussians with covariance $\eta \sigma^2 U(2\Lambda)^{-1}U^\T$
and $\eta \sigma^2 \left(2\Lambda + \eta \Lambda^2 \right)^{-1}$, respectively.
This elucidates why, in \Cref{fig:quad_minimum}, their dynamics oscillate around the minimum.

\paragraph{Escaping Behavior on Bimodal Function in \Cref{fig:two_mode}}
Consider the one-dimensional piece-wise quadratic function depicted in
\Cref{fig:two_mode_func}. When we initialize the point within the range $[-1,
1]$, the local function is quadratic.
At least for the first step, the behaviors of SGD and \pfs\ should
be accurately predicted by \Cref{prop:other_SDE_quad}.
We consider small stepsize close to $1/\|\nabla^2 f(x)\|$ (which
is 1 in this case).
Due to the additive noise, the iterates of SGD always have a variance no smaller
than $\eta^2 \sigma^2$.
Such a variance allows the iterates to escape
from the local basin and reach the basin around the global minimum at $x=5$.
However, after the first step, \pfs\ would have a mean around 0
and variance also near 0 after time $\eta$. It can be shown by
substituting $t=\eta$, setting $\eta \Lambda$ close to 1 in
\Cref{prop:other_SDE_quad} and by noting that $\lim_{z \to 1} \frac{1-(1-z)^2}{\log (1-z)} = 0$.
This explains why \pfs\ tends to stay closely around
the initial minimum in our simulation, as shown in \Cref{fig:two_mode}. This
phenomenon underscores the need for a more nuanced correction of the diffusion
term to accurately model the escaping behaviors of SGD in such settings.

\subsection{Hardness of Approximating SGD with OU Process on Quadratics}
Ideally, we would like to approximate SGD on quadratics exactly using SDEs, i.e., the iterates
of SGD share the same distribution as its continuous-time counterpart
at time stamps $k\eta$ for any natural number $k$. However, this is not possible for
general quadratic objective with arbitrary gradient noises, even if they are additive
and state-independent Gaussian noises. 
We provide a hard instance, demonstrating that no OU process (even extended to
complex domain as elaborated below) can achieve this.
We focus on the class of OU process because the corresponding transition kernel is Gaussian, which matches the transition kernel of 
SGD (when viewed as a Markov chain) under \Cref{assume:quad,assume:add_indep_noise}.

\paragraph{Complex OU Process} 
We note that the iterates of SGD in this context is
always \emph{real-valued} and follows a Gaussian distribution, as outlined in
\Cref{prop:other_SDE_quad}\footnote{This proposition, while initially framed for isotropic
noises, can be easily generalized to anisotropic noises.}. 
However, to allow the continuous-time proxy to have the maximum capability of approximating SGD, we allow it to be \emph{complex-valued}.
The benefit of extending to complex space has been observed by \citet{rosca2022continuous} when matching the dynamics of GD by PF.

A random vector, denoted by $z = x + i y$, is a complex Gaussian random vector if $x$ and $y$ are two (possibly correlated) real Gaussian random vectors.
In the real-valued case, a Gaussian random vector can be characterized by its mean and covariance, while in the complex-valued case, one additional parameter called \emph{pseudo-covariance} is involved.
Formally, the covariance is defined as 
$\Gamma(z) = \Ep{\left(z - \Ep{z}\right)\left(z - \Ep{z}\right)^H}$ and pseudo-covariance
is defined as $C(z) = \Ep{\left(z - \Ep{z}\right)\left(z - \Ep{z}\right)^\T}$.
We refer to \Cref{sec:pre_complex}
for a detailed discussion of properties of complex Gaussian random vectors.

We now state the desideratum when matching the distribution of a real normal variable, denoted as 
$\widetilde z$, and a complex normal variable $z$.
\begin{desideratum} \label{lemma:dist_match}
  The distribution of a complex Gaussian random vector $z$ is said to \emph{match} the distribution
  of real normal variable $\widetilde z$ if the following conditions hold:
  \begin{align*}
    \Ep{\Re(z)} = \Ep{\widetilde{z}}, \quad \quad \Ep{\Im(z)} = 0, \quad \quad
    \Gamma(z) = C(z) = \Cov{\widetilde{z}, \widetilde{z}}.
  \end{align*}
  The above result is equivalent to the statement: $\widetilde z$ and $\Re(z)$ have the same distribution, and $\Im(z)$ have 0 mean and 0 variance.
    Moreover, consider the setting described by \Cref{assume:quad,assume:add_indep_noise} so that the SGD iterates are all real Gaussian random vectors. 
    A complex OU process
  \begin{align*}
    dX_t = B X_t dt + D dW_t,
  \end{align*}
  is said to \emph{match} the iterates of the SGD if the \added{marginal} distributions of $X_t$ matches those of the SGD iterates at all corresponding time stamps, i.e. $t = k\eta$ where $k$ is the iteration index of SGD.
  Here $B \in \mathbb{C}^{d \times d}$, $D \in \mathbb{C}^{d \times m}$ and $W_t$ represents an $m$-dimensional standard Wiener process.
\end{desideratum}

\begin{proposition} \label{thm:quad_general}
  Consider the setting described by \Cref{assume:quad,assume:add_indep_noise}.
  One can construct real matrices $A$ and $\Sigma$ such that, for any given stepsize $\eta > 0$, there exists no complex OU process\footnote{
    In our formulation of complex OU process, we use real Brownian motion $W_t$.
    It should be noted, however, that the framework can also include complex Brownian
    motion~\citep{perret2010stability}. Consider a complex Brownian motion defined by $\widetilde{W_t}
    = W^r_t + i W^i_t$, where $W^r_t$ and $W^i_t$ are two independent real
    Brownian motions. Consequently, any SDE of the form $dX_t = B X_t dt + D
    d\widetilde{W_t}$ can be equivalently expressed as $dX_t = B X_t dt + D
    \begin{bmatrix} I & iI \end{bmatrix} \begin{bmatrix} dW^r_t \ dW^i_t \end{bmatrix}^\T$. }
    that matches the iterates of SGD, in the sense of \Cref{lemma:dist_match}.

\end{proposition}

The hard instance is constructed by choosing a \emph{full-rank} matrix $A$ and a \emph{degenerate} covariance matrix $\Sigma$ such that the
eigenvectors of $\Sigma$ are not aligned with those of $A$.
The distributional mismatch of the SDE and SGD can be intuitively understood through the following argument:
Starting from a deterministic initialization, after one SGD step, the covariance of the iterates
is $\eta^2 \Sigma$, which is rank-deficient;
In contrast, in continuous-time dynamics, the noise injected is rotated by the misaligned linear transformations
from the drift term and the covariance matrix quickly becomes full-rank. Please find the rigorous proof in \Cref{proof_hard_instance}.
\subsection{Exact Approximation From \oursde}

While in general, complex OU process cannot exactly model SGD on quadratic functions,
\oursde\, when extended to complex-valued,
achieves an exact match when the matrices $A$ and $\Sigma$ commute or when the stepsize
is small enough.
These conditions are slightly weaker than the sufficient conditions
for the existence of our \oursde{} (\Cref{prop:SPF_exist}).

The following lemma outline conditions for a complex OU process to match SGD on quadratics. 

\begin{lemma} \label{thm:approx_quad}
  Consider the same setting as \Cref{thm:quad_general}, and
  denote the diagonal elements of $\Lambda$ as $\lambda_1, \lambda_2, \dots, \lambda_d$.
After time $k \eta$ for $k\geq 0$, the mean of $X(k \eta)$
equals the mean of $x_k$, i.e., $\Ep{X(k \eta)} = \Ep{x_k}$, if and only if 
$B = U \frac{\log\left(1 - \eta \Lambda\right)}{\eta} U^\T.$ %
In addition,
\begin{enumerate}
\item the covariance of $X(k \eta)$ equals the covariance of $x_k$
if and only if for all $0 \leq i, j \leq d$,
\begin{align} \label{eq:match_cov}
  \left[\left(U^\T D\right) \left(U^\T D\right)^H\right]_{i, j} = \frac{\left[U^\T \Sigma U \right]_{i,j} \left(\log (1-\eta \lambda_i)
  + \overline{\log (1-\eta \lambda_j)}\right)}{\eta \lambda_i \lambda_j - (\lambda_i + \lambda_j)}.
\end{align}
\item pseudo-covariance of $X(k \eta)$ equals the covariance of $x_k$
if and only if for all $0 \leq i, j \leq d$,
\begin{align} \label{eq:match_psu_cov}
  \left[\left(U^\T D\right) \left(U^\T D\right)^\T\right]_{i, j} = \frac{\left[U^\T \Sigma U \right]_{i,j} \left(\log (1-\eta \lambda_i)
  + \log (1-\eta \lambda_j)\right)}{\eta \lambda_i \lambda_j - (\lambda_i + \lambda_j)}.
\end{align}
\end{enumerate}
\end{lemma}

Satisfying any of the existence conditions for \oursde{} in \Cref{prop:SPF_exist} ensures
the existence of matrix $D$ that adheres to the two required conditions in \Cref{thm:approx_quad}.
Under these existence conditions, the matrices, whose elements are defined in the
RHS of \Cref{eq:match_cov} and \Cref{eq:match_psu_cov}, become the same real
positive semi-definite matrix. 
By taking the square roots of this matrix's eigenvalues, we can construct
$D$ that satisfies both \Cref{eq:match_cov,eq:match_psu_cov},
thereby achieving an exact match for SGD on quadratic functions.
\begin{theorem} \label{cor:spf_match_quad}
  Under \Cref{assume:quad,assume:add_indep_noise}, the solution of \oursde\ exactly
  matches the iterates of SGD if either of the following two
  conditions holds:
  \begin{enumerate}
    \item $A$ and $\Sigma$ commute, and $\eta \lambda_i \neq 1$ for all eigenvalues
  $\lambda_i$ of $A$.
\item $\Sigma$ is positive definite, and $\eta \leq  \frac{1}{\norm{A}} \min\left\{1 - \sqrt{1 - \frac{\lambda_{\min}\left(\Sigma\right)}{\sqrt{d} 
      \lambda_{\max}\left(\Sigma\right)}}, 1 - \frac{\sqrt{2}}{2}\right\}.$
  \end{enumerate}
\end{theorem}
To the best of our knowledge, this marks the first instance of an SDE that precisely
mirrors the distribution of SGD, albeit restricted to quadratic functions. 
\begin{remark}
  The first condition in this theorem relaxes the constraints from
  \Cref{prop:SPF_exist} to allow larger stepsizes.
  As noted in \citet{rosca2022continuous},
  complex flow is helpful for capturing the instabilities caused by large stepsizes
  in GD on quadratics. This is also the case for
  \oursde{}. When \(\eta < 1/\norm{A}\), \oursde{} operates in real space.
  However, when \(\eta > 1/\norm{A}\), imaginary components emerge due to the
  logarithmic function.
\end{remark}

{

\section{Escape Analysis Beyond Quadratics}
\label{sec:escape_analysis}

Our proposed \oursde{} serves as an analytical tool that can be seamlessly
integrated into standard SDE analysis to better understand the escaping behavior
of SGD. A particularly relevant line of work is the expected hitting time
analysis, which focuses on the following expectation:
\begin{align*} \Ep{\tau_S}
= \Ep{\inf\{ t > 0 \mid X(t) \in \partial S,\, X(0) = x_0 \in S \}},
\end{align*}
where $S$ is a compact set and $\partial S$ denotes its boundary.
The expected hitting time characterizes how quickly a stochastic process escapes
from a given region. This quantity has been extensively studied in the context
of randomly perturbed dynamical
systems~\citep{day1995exit,kifer1981exit,vent1973some,FreidlinWentzell2012,pavliotis2014stochastic}.
In deep learning, expected hitting time has been used to compare the escape
rates from saddle points under different optimization
algorithms~\citep{xie2022adaptive}, as well as to study minima
selection~\citep{xie2020diffusion,xie2022adaptive,ibayashi2023does,zhou2020towards}.

In this section, we conduct such analysis beyond quadratic functions and consider
general function classes.
We first start with a one-dimensional local
minimum/maximum setting under \oursde{}, illustrating its potential to be used
as an analytical tool for capturing the escaping behavior of SGD.
We then turn our focus to the regime of larger stepsizes, where existing
SDEs begin to fail, as discussed in \Cref{sec:exist_quad}. Motivated by
the insights from that section, we analyze the behavior near
multi-dimensional saddle points. In this setting, we show that \sme-2 requires
exponentially long escape times, whereas \oursde{} escapes sub-exponentially fast---closely
aligning with the empirical behavior observed in SGD.

\subsection{Escaping of \oursde{} Near One-Dimensional Local Minimum/Maximum}

To illustrate how \oursde{} can be leveraged to analyze escape
times, we begin with a simple one-dimensional setting involving a local minimum
or maximum.

\begin{theorem}
  \label{thm:escape_analysis}
  Assume that there exists a compact domain $S_0 \subset \bR$, and let $f: \bR \to \bR$ be a $C^4(S_0)$ function. Further, assume the following conditions hold for all $x \in S_0$:
\begin{enumerate}
  \item There exists $x_0 \in S_0$ such that $f'(x_0) = 0$, i.e., $x_0$ is a stationary point.
  \item The condition in \Cref{prop:well_posed} holds, ensuring that \oursde{} is well-posed.
  \item The function $f(x)$ is symmetric around $x_0$\footnote{
      \red{
The symmetry assumption for $f$ was introduced for analytical convenience
and can be relaxed. If $f$ is not symmetric, the expected hitting time will
be dominated by the less challenging half of the domain.
}
    }.
\end{enumerate}
Let \oursde{} be initialized at $x_0$. Then there exists a compact set $S_1 \subseteq S_0$ with $x_0 \in S_1$ such that:
\begin{enumerate}
  \item \emph{(Local minimum)} If $f''(x_0) > 0$, then there exists $c > 0$ such that
    \begin{align} \label{eq:escape_local_min}
    \lim_{\eta \to 0} \eta \log \Ep{\tau_{S_1}} = c.
  \end{align}
  \item \emph{(Local maximum)} If $f''(x_0) < 0$, then there exist $c_1, c_2 > 0$ such that
  \begin{align} \label{eq:escape_local_max}
    c_1 \leq \lim_{\eta \to 0} \frac{\Ep{\tau_{S_1}}}{\log(1/\eta)} \leq c_2.
  \end{align}
\end{enumerate}

\end{theorem}
It is known that for the one-dimensional SDE
\begin{align*}
  dX_t = b(X_t)\,dt + \sqrt{\cD(X_t)}\,dW_t,
\end{align*}
the expected hitting time admits an analytical expression~\citep{pavliotis2014stochastic}:
\begin{align*}
  \Ep{\tau_{S_1}} = 2 \int_0^a \frac{\exp\left(-\psi(z)\right)}{\cD(z)}\,dz
  \int_z^a \exp\left(\psi(y)\right)\,dy,
\end{align*}
where $\psi'(x) = -2b(x)/\cD(x)$ and $S_1 = [-a, a]$.
\red{
Our results are obtained by analyzing the double integral and characterizing its
dependence on the stepsize $\eta$. The analysis shows that the expected hitting
time at a local minimum scales exponentially with $\eta$, and the constant $c$
in \Cref{eq:escape_local_min} corresponds to the Freidlin–Wentzell
quasi-potential~\citep{FreidlinWentzell2012}, which represents the minimal ``energy cost'' or action required
for the SDE to reach the boundary.
In contrast, the expected hitting time at a local maximum scales logarithmically
with $\eta$, indicating fast escape dynamics consistent with that of SGD.
}

\red{
One can obtain for \sme-1 and \sme-2 the same asymptotic
result as stated in \Cref{thm:escape_analysis}.
This outcome is expected as
the analysis assumes the stepsize is small compared to the local curvature. 
We will demonstrate in the next section that \sme-2 can fail to capture the
correct escape behavior when the stepsize becomes large.
}
\begin{remark} \label{rem:multi_dim}
  While the above result is stated for one-dimensional functions, it can
  potentially be extended to multi-dimensional settings.
  There is a rich body of classical work on the expected hitting time in the
  multi-dimensional case, particularly when the drift $b$ is independent of
  $\eta$ and the diffusion $\cD$ is linear in
  $\eta$~\citep{day1995exit,kifer1981exit,vent1973some,FreidlinWentzell2012}. 
  For \oursde{}, the dependence of both the drift and diffusion on $\eta$ is
  more intricate. However, recent studies have begun to address settings with
  general $\eta$-dependence in both
  terms~\citep{chiarini2014large,aleksian2025freidlin}. In particular, we can
  directly apply the result from \citet[Corollary~1.9]{aleksian2025freidlin} to
  derive multi-dimensional escape time estimates for local minima under
  \oursde{}.
\end{remark}

\subsection{Escaping Near Multi-Dimensional Saddle Points}
\label{sec:escape_saddle}

In this subsection, we consider the case where the stepsize is inversely
proportional to the Hessian eigenvalues, a situation that can naturally arise
locally in practice during deep learning training, as motivated in
\Cref{sec:fine_grained_error_analysis}. 
To model this scenario, we assume that, locally, the Hessian can be approximated
as $\nabla^2 f(x) \approx g(x) / \eta$, where $g(x): \mathbb{R}^d \to
\mathbb{R}^{d \times d}$ is a symmetric matrix-valued function independent of
$\eta$. Under this assumption, the rescaled quantity $\eta \nabla^2 f(x)$
remains non-negligible, i.e., $\eta$ is relatively large compared to the local
curvature. For analytical tractability, we neglect contributions from terms that
are asymptotically smaller than $1/\eta$ in the expansion of $\nabla^2 f(x)$.

\begin{theorem}
  \label{thm:escape_analysis_multi_dim}
  Assume there exists a compact domain $S_0 \subset \bR^d$, and let $f:
  \bR^d \to \bR$ be a $C^4(S_0)$ function. For all $x \in S_0$, suppose the
  following conditions hold:
\begin{enumerate}
  \item There exists $x_0 \in S_0$ such that $\nabla f(x_0) = 0$, and $\nabla^2
    f(x_0)$ has both positive and negative eigenvalues; that is, $x_0$ is a
    saddle point.
  
  \item The Hessian satisfies $\nabla^2 f(x) = g(x) / \eta$, where $g(x)$ is
    independent of $\eta$. Furthermore, there exists an absolute constant $c_0 >
    0$ such that all negative eigenvalues $\lambda_i$ of $\nabla f(x_0)$ satisfy
    $\eta \lambda_i \leq -2 - c_0$, i.e., the stepsize is large enough compared
    to the Hessian eigenvalues.
  
  \item The condition in \Cref{prop:well_posed} is satisfied so that \oursde{}
    is well-posed.\footnote{To ensure \oursde{} is well-posed, we require
    $\abs{\eta \lambda_i(x)} < 1$ for all $1 \leq i \leq d$ and $x$, where
  $\lambda_i(x)$ denotes the $i$-th eigenvalue of $\nabla^2 f(x)$, as noted in
\Cref{prop:SPF_exist}. However, this can be relaxed to an upper bound condition $\eta
\lambda_i(x) < 1$ along with a mild lower bound ensuring that $\eta \lambda_i(x)
> m$ for some constant $m < 0$. Thus, the second condition required here in this
theorem does not contradict well-posedness.}
  
  \item The gradient noise covariance $\Sigma(x)$ satisfies $\norm{\Sigma(x)} =
    \Theta(1)$ uniformly, and the derivatives $\abs{\partial
    [\Sigma(x)]_{i,j} / \partial [x]_k}$ are uniformly $\cO(1)$ for $i, j, k \in [d]$.
\end{enumerate}

Then, there exists a compact set $S_1 \subseteq S_0$ with $x_0 \in S_1$,
independent of $\eta$, and constants $\eta_0, c_1, c_2 > 0$ such that for all
$\eta < \eta_0$\footnote{
  The assumption that $\eta$ is smaller than some constant $\eta_0$ does not
  contradict the large stepsize assumption. The notion of a ``large'' stepsize is
  relative to the local Hessian eigenvalues, whereas $\eta_0$ is an absolute
  constant.
  In fact, in practice, even with a fixed stepsize $\eta$, there exist local
  regions where the Hessian eigenvalues can grow as large as
  $2/\eta$~\citep{cohen2020gradient}.
}:
\begin{enumerate}
  \item For \sme-2:
  \begin{align*}
    \Ep{\tau_{S_1}} \geq \exp\left(\frac{c_1}{\eta^2}\right).
  \end{align*}
  \item For \oursde{}:
  \begin{align*}
    \Ep{\tau_{S_1}} \leq c_2 \eta \log\left(\frac{1}{\eta}\right).
  \end{align*}
\end{enumerate}
\end{theorem}

\red{
Since we assume that locally $\eta \nabla^2 f$ is independent of $\eta$, the
dependence on $\eta$ in the drift and diffusion terms of both \sme-2 and
\oursde{} is simply linear. As a result, classical results on exit
times~\citep{vent1973some,kifer1981exit} apply directly.
If we instead consider a broader class of functions where $f$ and $\eta$ exhibit
more intricate local dependencies, additional care is needed. As mentioned in
\Cref{rem:multi_dim}, recent works have studied settings where the drift and
diffusion terms can depend on $\eta$ in arbitrary
ways~\citep{chiarini2014large,aleksian2025freidlin}. These advances suggest that
our result can potentially be extended to a much wider class of functions.
}

\red{
The above theorem demonstrates that \sme-2 requires exponentially long time to
escape a saddle point, whereas \oursde{} escapes in sub-exponential time. For
SGD, even a small amount of noise is sufficient to perturb the iterates away
from a stationary saddle point. Once perturbed, the iterates quickly escape
along the direction of negative curvature, as the gradient naturally points away
from the saddle. \oursde{} accurately captures this rapid escape behavior,
aligning closely with that of SGD, while \sme-2 fails to reflect this dynamic.
}

\begin{remark}
Near a local saddle point, we note that the drift in \sme-2 is the negative gradient  
of the function $f + \frac{\eta}{4} \| \nabla f \|^2$.  
The regularizer $\frac{\eta}{4} \| \nabla f \|^2$ can transform a negative  
curvature direction---where the second-order gradient is negative, causing  
SGD to escape---into a positive one.
Consequently, the process may stabilize at the saddle point rather than escape,
causing an exponentially long escape time as indicated in the theorem.
According to our SBEA framework, we observe that while \sme-2 attempts to  
capture more $\Theta\left(\eta^2\right)$ terms compared to \sme-1, it still truncates
many higher-order terms related to the stepsize and Hessian,
thus, the local curvature information is not well captured---even wrongly captured in this case.
In contrast, \oursde{} incorporates as many Hessian-related terms as possible,  
providing a more accurate curvature estimation.  
\end{remark}

}

\section{Conclusion and Future Work}
\label{sec:conclusion}
In this work, we present \oursde{}, an advancement over existing SDEs for
approximating the dynamics of SGD. Specifically, \oursde{} offers improved
theoretical approximation guarantees for general smooth objectives over
current SDEs in terms of the dependence on the smoothness parameter of the objective.
For quadratic objectives, \oursde{}
exactly recovers the dynamics of SGD under mild conditions.
The primary innovation lies in integrating Hessian information into both the
drift and diffusion terms of the SDE, achieved by extending backward error
analysis to the stochastic setting. This integration preserves the interplay of
noise and local curvature in approximating SGD, allowing to better capture
its escaping behaviors.

\red{
We anticipate that the current results will deepen our understanding of SGD’s
escape dynamics and minima selection, enhance hyperparameter tuning, and
inspire the development of more effective optimization algorithms. 
Furthermore, applying our SBEA framework to other optimization algorithms, such
as momentum methods and adaptive gradient methods, would be an intriguing
direction for future research. It would also be interesting to explore SDEs
driven by alternative noise processes, such as L\'evy noise.
In deriving the weak approximation error, we established a uniform-in-time
guarantee for strongly convex objectives. A natural question for future work
is whether similar guarantees can be obtained for broader function classes,
such as those satisfying the Polyak–Łojasiewicz (PL) condition.
}

\medskip
\bibliographystyle{plainnat}
\bibliography{ref.bib}

\begin{thebibliography}{90}
\providecommand{\natexlab}[1]{#1}
\providecommand{\url}[1]{\texttt{#1}}
\expandafter\ifx\csname urlstyle\endcsname\relax
  \providecommand{\doi}[1]{doi: #1}\else
  \providecommand{\doi}{doi: \begingroup \urlstyle{rm}\Url}\fi

\bibitem[Aleksian and Villeneuve(2025)]{aleksian2025freidlin}
Ashot Aleksian and St{\'e}phane Villeneuve.
\newblock Freidlin-wentzell type exit-time estimates for time-inhomogeneous diffusions and their applications.
\newblock \emph{arXiv}, page 2501.11797, 2025.

\bibitem[Arora et~al.(2019)Arora, Du, Hu, Li, Salakhutdinov, and Wang]{arora2019exact}
Sanjeev Arora, Simon~S Du, Wei Hu, Zhiyuan Li, Russ~R Salakhutdinov, and Ruosong Wang.
\newblock On exact computation with an infinitely wide neural net.
\newblock In \emph{Advances in Neural Information Processing Systems}, 2019.

\bibitem[Attouch et~al.(2022)Attouch, Chbani, Fadili, and Riahi]{attouch2022first}
Hedy Attouch, Zaki Chbani, Jalal Fadili, and Hassan Riahi.
\newblock First-order optimization algorithms via inertial systems with {Hessian} driven damping.
\newblock \emph{Mathematical Programming}, 193:\penalty0 113–155, 2022.

\bibitem[Barakat and Bianchi(2021)]{barakat2021convergence}
Anas Barakat and Pascal Bianchi.
\newblock Convergence and dynamical behavior of the {Adam} algorithm for nonconvex stochastic optimization.
\newblock \emph{SIAM Journal on Optimization}, 31\penalty0 (1):\penalty0 244--274, 2021.

\bibitem[Barrett and Dherin(2021)]{barrett2020implicit}
David~GT Barrett and Benoit Dherin.
\newblock Implicit gradient regularization.
\newblock In \emph{International Conference on Machine Learning}, 2021.

\bibitem[Bengio(2012)]{bengio2012practical}
Yoshua Bengio.
\newblock Practical recommendations for gradient-based training of deep architectures.
\newblock In \emph{Neural Networks: Tricks of the Trade: Second Edition}, pages 437--478. Springer, 2012.

\bibitem[Bloch(1994)]{bloch1994hamiltonian}
Anthony Bloch.
\newblock \emph{Hamiltonian and gradient flows, algorithms and control}, volume~3.
\newblock American Mathematical Soc., 1994.

\bibitem[Brandi\`ere and Duflo(1996)]{brandiere1996algorithmes}
Odile Brandi\`ere and Marie Duflo.
\newblock Les algorithmes stochastiques contournent-ils les pi\`eges ?
\newblock \emph{Annales de l'I.H.P. Probabilit\'es et statistiques}, 32\penalty0 (3):\penalty0 395--427, 1996.

\bibitem[Brown and Bartholomew-Biggs(1989)]{brown1989some}
Andrew~A Brown and Michael~C Bartholomew-Biggs.
\newblock Some effective methods for unconstrained optimization based on the solution of systems of ordinary differential equations.
\newblock \emph{Journal of Optimization Theory and Applications}, 62:\penalty0 211--224, 1989.

\bibitem[Cauchy(1821)]{cauchy1821analyse}
Augustin Louis~Baron Cauchy.
\newblock \emph{Analyse alg{\'e}brique}, volume~1.
\newblock Debure, 1821.

\bibitem[Chiarini and Fischer(2014)]{chiarini2014large}
Alberto Chiarini and Markus Fischer.
\newblock On large deviations for small noise it{\^o} processes.
\newblock \emph{Advances in Applied Probability}, 46\penalty0 (4):\penalty0 1126--1147, 2014.

\bibitem[Cohen et~al.(2020)Cohen, Kaur, Li, Kolter, and Talwalkar]{cohen2020gradient}
Jeremy Cohen, Simran Kaur, Yuanzhi Li, J~Zico Kolter, and Ameet Talwalkar.
\newblock Gradient descent on neural networks typically occurs at the edge of stability.
\newblock In \emph{International Conference on Learning Representations}, 2020.

\bibitem[Cohen et~al.(2022)Cohen, Ghorbani, Krishnan, Agarwal, Medapati, Badura, Suo, Cardoze, Nado, Dahl, and Gilmer]{cohen2022adaptive}
Jeremy~M Cohen, Behrooz Ghorbani, Shankar Krishnan, Naman Agarwal, Sourabh Medapati, Michal Badura, Daniel Suo, David Cardoze, Zachary Nado, George~E Dahl, and Justin Gilmer.
\newblock Adaptive gradient methods at the edge of stability.
\newblock \emph{arXiv}, page 2207.14484, 2022.

\bibitem[Compagnoni et~al.(2025)Compagnoni, Liu, Islamov, Proske, Orvieto, and Lucchi]{compagnoni2025adaptive}
Enea~Monzio Compagnoni, Tianlin Liu, Rustem Islamov, Frank~Norbert Proske, Antonio Orvieto, and Aurelien Lucchi.
\newblock Adaptive methods through the lens of {SDE}s: Theoretical insights on the role of noise.
\newblock In \emph{International Conference on Learning Representations}, 2025.

\bibitem[Day(1995)]{day1995exit}
Martin~V Day.
\newblock On the exit law from saddle points.
\newblock \emph{Stochastic processes and their applications}, 60\penalty0 (2):\penalty0 287--311, 1995.

\bibitem[Debussche and Faou(2012)]{doi:10.1137/110831544}
Arnaud Debussche and Erwan Faou.
\newblock Weak backward error analysis for {SDEs}.
\newblock \emph{SIAM Journal on Numerical Analysis}, 50\penalty0 (3):\penalty0 1735--1752, 2012.

\bibitem[Dieuleveut et~al.(2020)Dieuleveut, Durmus, and Bach]{10.1214/19-AOS1850}
Aymeric Dieuleveut, Alain Durmus, and Francis Bach.
\newblock {Bridging the gap between constant step size stochastic gradient descent and Markov chains}.
\newblock \emph{The Annals of Statistics}, 48\penalty0 (3):\penalty0 1348 -- 1382, 2020.

\bibitem[Drori and Shamir(2020)]{10.5555/3524938.3525187}
Yoel Drori and Ohad Shamir.
\newblock The complexity of finding stationary points with stochastic gradient descent.
\newblock In \emph{International Conference on Machine Learning}, 2020.

\bibitem[Fang et~al.(2019)Fang, Lin, and Zhang]{pmlr-v99-fang19a}
Cong Fang, Zhouchen Lin, and Tong Zhang.
\newblock Sharp analysis for nonconvex sgd escaping from saddle points.
\newblock In \emph{Conference on Learning Theory}, 2019.

\bibitem[Feng et~al.(2017)Feng, Li, and Liu]{feng2017semi}
Yuanyuan Feng, Lei Li, and Jian-Guo Liu.
\newblock Semi-groups of stochastic gradient descent and online principal component analysis: properties and diffusion approximations.
\newblock \emph{arXiv}, page 1712.06509, 2017.

\bibitem[Feng et~al.(2019)Feng, Gao, Li, Liu, and Lu]{feng2019uniform}
Yuanyuan Feng, Tingran Gao, Lei Li, Jian-Guo Liu, and Yulong Lu.
\newblock Uniform-in-time weak error analysis for stochastic gradient descent algorithms via diffusion approximation.
\newblock \emph{arXiv}, page 1902.00635, 2019.

\bibitem[Fiori and Bengio(2005)]{fiori2005quasi}
Simone Fiori and Yoshua Bengio.
\newblock Quasi-geodesic neural learning algorithms over the orthogonal group: A tutorial.
\newblock \emph{Journal of Machine Learning Research}, 6\penalty0 (5), 2005.

\bibitem[Fisher(1936)]{fisher1936use}
Ronald~A Fisher.
\newblock The use of multiple measurements in taxonomic problems.
\newblock \emph{Annals of eugenics}, 7\penalty0 (2):\penalty0 179--188, 1936.

\bibitem[Fontaine et~al.(2021)Fontaine, De~Bortoli, and Durmus]{fontaine2021convergence}
Xavier Fontaine, Valentin De~Bortoli, and Alain Durmus.
\newblock Convergence rates and approximation results for {SGD} and its continuous-time counterpart.
\newblock In \emph{Conference on Learning Theory}, 2021.

\bibitem[Freidlin and Wentzell(2012)]{FreidlinWentzell2012}
Mark~I. Freidlin and Alexander~D. Wentzell.
\newblock \emph{Random Perturbations of Dynamical Systems}.
\newblock Springer, 2012.

\bibitem[Ge et~al.(2015)Ge, Huang, Jin, and Yuan]{ge2015escaping}
Rong Ge, Furong Huang, Chi Jin, and Yang Yuan.
\newblock Escaping from saddle points—online stochastic gradient for tensor decomposition.
\newblock In \emph{Conference on Learning Theory}, 2015.

\bibitem[Ge et~al.(2019)Ge, Kakade, Kidambi, and Netrapalli]{ge2019step}
Rong Ge, Sham~M Kakade, Rahul Kidambi, and Praneeth Netrapalli.
\newblock The step decay schedule: A near optimal, geometrically decaying learning rate procedure for least squares.
\newblock In \emph{Advances in Neural Information Processing Systems}, 2019.

\bibitem[Ghadimi and Lan(2013)]{ghadimi2013stochastic}
Saeed Ghadimi and Guanghui Lan.
\newblock Stochastic first-and zeroth-order methods for nonconvex stochastic programming.
\newblock \emph{SIAM Journal on Optimization}, 23\penalty0 (4):\penalty0 2341--2368, 2013.

\bibitem[Hairer et~al.(2006)Hairer, Hochbruck, Iserles, and Lubich]{hairer2006geometric}
Ernst Hairer, Marlis Hochbruck, Arieh Iserles, and Christian Lubich.
\newblock Geometric numerical integration.
\newblock \emph{Oberwolfach Reports}, 3\penalty0 (1):\penalty0 805--882, 2006.

\bibitem[Helmke and Moore(2012)]{helmke2012optimization}
Uwe Helmke and John~B Moore.
\newblock \emph{Optimization and dynamical systems}.
\newblock Springer Science \& Business Media, 2012.

\bibitem[Hille and Phillips(1996)]{hille1996functional}
Einar Hille and Ralph~Saul Phillips.
\newblock \emph{Functional analysis and semi-groups}, volume~31.
\newblock American Mathematical Society, 1996.

\bibitem[Hu et~al.(2019)Hu, Li, Li, and Liu]{hu2017diffusion}
Wenqing Hu, Chris~Junchi Li, Lei Li, and Jian-Guo Liu.
\newblock On the diffusion approximation of nonconvex stochastic gradient descent.
\newblock \emph{Annals of Mathematical Sciences and Applications}, 4\penalty0 (1):\penalty0 3--32, 2019.

\bibitem[Ibayashi and Imaizumi(2023)]{ibayashi2023does}
Hikaru Ibayashi and Masaaki Imaizumi.
\newblock Why does {SGD} prefer flat minima?: Through the lens of dynamical systems.
\newblock In \emph{When Machine Learning meets Dynamical Systems: Theory and Applications}, 2023.

\bibitem[Jacot et~al.(2018)Jacot, Gabriel, and Hongler]{jacot2018neural}
Arthur Jacot, Franck Gabriel, and Cl{\'e}ment Hongler.
\newblock Neural tangent kernel: Convergence and generalization in neural networks.
\newblock In \emph{Advances in Neural Information Processing Systems}, 2018.

\bibitem[Jastrzebski et~al.(2017)Jastrzebski, Kenton, Arpit, Ballas, Fischer, Bengio, and Storkey]{jastrzebski2017three}
Stanis{\l}aw Jastrzebski, Zachary Kenton, Devansh Arpit, Nicolas Ballas, Asja Fischer, Yoshua Bengio, and Amos Storkey.
\newblock Three factors influencing minima in {SGD}.
\newblock \emph{arXiv}, page 1711.04623, 2017.

\bibitem[Jin et~al.(2017)Jin, Ge, Netrapalli, Kakade, and Jordan]{jin2017escape}
Chi Jin, Rong Ge, Praneeth Netrapalli, Sham~M Kakade, and Michael~I Jordan.
\newblock How to escape saddle points efficiently.
\newblock In \emph{International Conference on Machine Learning}, 2017.

\bibitem[Keskar et~al.(2017)Keskar, Mudigere, Nocedal, Smelyanskiy, and Tang]{keskar2016large}
Nitish~Shirish Keskar, Dheevatsa Mudigere, Jorge Nocedal, Mikhail Smelyanskiy, and Ping Tak~Peter Tang.
\newblock On large-batch training for deep learning: Generalization gap and sharp minima.
\newblock In \emph{International Conference on Learning Representations}, 2017.

\bibitem[Khaled and Richt{\'a}rik(2023)]{khaled2022better}
Ahmed Khaled and Peter Richt{\'a}rik.
\newblock Better theory for {SGD} in the nonconvex world.
\newblock \emph{Transactions on Machine Learning Research}, 2023.

\bibitem[Kifer(1981)]{kifer1981exit}
Yuri Kifer.
\newblock The exit problem for small random perturbations of dynamical systems with a hyperbolic fixed point.
\newblock \emph{Israel Journal of Mathematics}, 40:\penalty0 74--96, 1981.

\bibitem[Kloeden and Platen(2011)]{kloeden2011numerical}
P.E. Kloeden and E.~Platen.
\newblock \emph{Numerical Solution of Stochastic Differential Equations}.
\newblock Stochastic Modelling and Applied Probability. Springer Berlin Heidelberg, 2011.

\bibitem[Krichene et~al.(2015)Krichene, Bayen, and Bartlett]{krichene2015accelerated}
Walid Krichene, Alexandre Bayen, and Peter~L Bartlett.
\newblock Accelerated mirror descent in continuous and discrete time.
\newblock In \emph{Advances in Neural Information Processing Systems}, 2015.

\bibitem[Li and Wang(2022)]{li2022uniform}
Lei Li and Yuliang Wang.
\newblock On uniform-in-time diffusion approximation for stochastic gradient descent.
\newblock \emph{arXiv}, page 2207.04922, 2022.

\bibitem[Li et~al.(2017)Li, Tai, and E]{li2017stochastic}
Qianxiao Li, Cheng Tai, and Weinan E.
\newblock Stochastic modified equations and adaptive stochastic gradient algorithms.
\newblock In \emph{International Conference on Machine Learning}, 2017.

\bibitem[Li et~al.(2019)Li, Tai, and E]{li2019stochastic}
Qianxiao Li, Cheng Tai, and Weinan E.
\newblock Stochastic modified equations and dynamics of stochastic gradient algorithms i: Mathematical foundations.
\newblock \emph{Journal of Machine Learning Research}, 20\penalty0 (40):\penalty0 1--47, 2019.

\bibitem[Li and Orabona(2020)]{li2020high}
Xiaoyu Li and Francesco Orabona.
\newblock A high probability analysis of adaptive sgd with momentum.
\newblock \emph{arXiv}, page 2007.14294, 2020.

\bibitem[Li et~al.(2021)Li, Malladi, and Arora]{li2021validity}
Zhiyuan Li, Sadhika Malladi, and Sanjeev Arora.
\newblock On the validity of modeling sgd with stochastic differential equations (sdes).
\newblock In \emph{Advances in Neural Information Processing Systems}, 2021.

\bibitem[Lin and Maji(2017)]{lin2017improved}
Tsung-Yu Lin and Subhransu Maji.
\newblock Improved bilinear pooling with {CNN}s.
\newblock \emph{arXiv}, page 1707.06772, 2017.

\bibitem[Liu and Yuan(2023)]{liu2023almost}
Jun Liu and Ye~Yuan.
\newblock Almost sure saddle avoidance of stochastic gradient methods without the bounded gradient assumption.
\newblock \emph{arXiv}, page 2302.07862, 2023.

\bibitem[Ma et~al.(2022)Ma, Wu, and E]{ma2022qualitative}
Chao Ma, Lei Wu, and Weinan E.
\newblock A qualitative study of the dynamic behavior for adaptive gradient algorithms.
\newblock In \emph{Mathematical and Scientific Machine Learning}. PMLR, 2022.

\bibitem[Malladi et~al.(2022)Malladi, Lyu, Panigrahi, and Arora]{malladi2022sdes}
Sadhika Malladi, Kaifeng Lyu, Abhishek Panigrahi, and Sanjeev Arora.
\newblock On the sdes and scaling rules for adaptive gradient algorithms.
\newblock In \emph{Advances in Neural Information Processing Systems}, 2022.

\bibitem[Mandt et~al.(2015)Mandt, Hoffman, and Blei]{mandt2015continuous}
Stephan Mandt, Matthew~D Hoffman, and David~M Blei.
\newblock Continuous-time limit of stochastic gradient descent revisited.
\newblock In \emph{OPT workshop, NIPS}, 2015.

\bibitem[Mandt et~al.(2016)Mandt, Hoffman, and Blei]{mandt2016variational}
Stephan Mandt, Matthew Hoffman, and David Blei.
\newblock A variational analysis of stochastic gradient algorithms.
\newblock In \emph{International Conference on Machine Learning}, 2016.

\bibitem[Mertikopoulos et~al.(2020)Mertikopoulos, Hallak, Kavis, and Cevher]{mertikopoulos2020almost}
Panayotis Mertikopoulos, Nadav Hallak, Ali Kavis, and Volkan Cevher.
\newblock On the almost sure convergence of stochastic gradient descent in non-convex problems.
\newblock In \emph{Advances in Neural Information Processing Systems}, 2020.

\bibitem[Milstein(2013)]{milstein2013numerical}
Grigorii~Noikhovich Milstein.
\newblock \emph{Numerical integration of stochastic differential equations}, volume 313.
\newblock Springer Science \& Business Media, 2013.

\bibitem[Mori et~al.(2022)Mori, Ziyin, Liu, and Ueda]{mori2021power}
Takashi Mori, Liu Ziyin, Kangqiao Liu, and Masahito Ueda.
\newblock Power-law escape rate of {SGD}.
\newblock In \emph{International Conference on Machine Learning}, 2022.

\bibitem[Muehlebach and Jordan(2019)]{muehlebach2019dynamical}
Michael Muehlebach and Michael~I Jordan.
\newblock A dynamical systems perspective on {Nesterov} acceleration.
\newblock In \emph{International Conference on Machine Learning}, 2019.

\bibitem[Muehlebach and Jordan(2021)]{muehlebach2021optimization}
Michael Muehlebach and Michael~I Jordan.
\newblock Optimization with momentum: Dynamical, control-theoretic, and symplectic perspectives.
\newblock \emph{Journal of Machine Learning Research}, 22\penalty0 (73):\penalty0 1--50, 2021.

\bibitem[Nar and Sastry(2018)]{nar2018step}
Kamil Nar and Shankar Sastry.
\newblock Step size matters in deep learning.
\newblock In \emph{Advances in Neural Information Processing Systems}, 2018.

\bibitem[Neyshabur et~al.(2017)Neyshabur, Bhojanapalli, Mcallester, and Srebro]{neyshabur2017exploring}
Behnam Neyshabur, Srinadh Bhojanapalli, David Mcallester, and Nati Srebro.
\newblock Exploring generalization in deep learning.
\newblock In \emph{Advances in Neural Information Processing Systems}, 2017.

\bibitem[Neyshabur et~al.(2018)Neyshabur, Bhojanapalli, and Srebro]{neyshabur2017pac}
Behnam Neyshabur, Srinadh Bhojanapalli, and Nathan Srebro.
\newblock A {PAC}-bayesian approach to spectrally-normalized margin bounds for neural networks.
\newblock In \emph{International Conference on Learning Representations}, 2018.

\bibitem[Nguyen et~al.(2019)Nguyen, Simsekli, Gurbuzbalaban, and Richard]{nguyen2019first}
Thanh~Huy Nguyen, Umut Simsekli, Mert Gurbuzbalaban, and Ga{\"e}l Richard.
\newblock First exit time analysis of stochastic gradient descent under heavy-tailed gradient noise.
\newblock In \emph{Advances in Neural Information Processing Systems}, 2019.

\bibitem[Orvieto and Lucchi(2019)]{orvieto2019continuous}
Antonio Orvieto and Aurelien Lucchi.
\newblock Continuous-time models for stochastic optimization algorithms.
\newblock In \emph{Advances in Neural Information Processing Systems}, 2019.

\bibitem[Pan et~al.(2021)Pan, Ye, and Zhang]{pan2021eigencurve}
Rui Pan, Haishan Ye, and Tong Zhang.
\newblock Eigencurve: Optimal learning rate schedule for {SGD} on quadratic objectives with skewed hessian spectrums.
\newblock \emph{arXiv}, page 2110.14109, 2021.

\bibitem[Paquette and Paquette(2021)]{paquette2021dynamics}
Courtney Paquette and Elliot Paquette.
\newblock Dynamics of stochastic momentum methods on large-scale, quadratic models.
\newblock In \emph{Advances in Neural Information Processing Systems}, 2021.

\bibitem[Paquette et~al.(2022)Paquette, Paquette, Adlam, and Pennington]{paquette2022homogenization}
Courtney Paquette, Elliot Paquette, Ben Adlam, and Jeffrey Pennington.
\newblock Homogenization of {SGD} in high-dimensions: Exact dynamics and generalization properties.
\newblock \emph{arXiv}, page 2205.07069, 2022.

\bibitem[Pavliotis(2014)]{pavliotis2014stochastic}
Grigorios~A Pavliotis.
\newblock Stochastic processes and applications.
\newblock \emph{Texts in applied mathematics}, 60, 2014.

\bibitem[Pemantle(1990)]{pemantle1990nonconvergence}
Robin Pemantle.
\newblock Nonconvergence to unstable points in urn models and stochastic approximations.
\newblock \emph{The Annals of Probability}, 18\penalty0 (2):\penalty0 698--712, 1990.

\bibitem[Perret(2010)]{perret2010stability}
Christian Perret.
\newblock \emph{The stability of numerical simulations of complex stochastic differential equations}.
\newblock PhD thesis, ETH Zurich, 2010.

\bibitem[Robbins and Monro(1951)]{10.1214/aoms/1177729586}
Herbert Robbins and Sutton Monro.
\newblock {A Stochastic Approximation Method}.
\newblock \emph{The Annals of Mathematical Statistics}, 22\penalty0 (3):\penalty0 400 -- 407, 1951.

\bibitem[Rosca et~al.(2022)Rosca, Wu, Qin, and Dherin]{rosca2022continuous}
Mihaela Rosca, Yan Wu, Chongli Qin, and Benoit Dherin.
\newblock On a continuous time model of gradient descent dynamics and instability in deep learning.
\newblock \emph{Transactions on Machine Learning Research}, 2022.

\bibitem[Ross(1997)]{ross1997generalized}
Peter Ross.
\newblock Generalized hockey stick identities and $n$-dimensional blockwalking.
\newblock \emph{The College Mathematics Journal}, 28\penalty0 (4):\penalty0 325, 1997.

\bibitem[S{\"a}rkk{\"a} and Solin(2019)]{sarkka2019applied}
Simo S{\"a}rkk{\"a} and Arno Solin.
\newblock \emph{Applied stochastic differential equations}, volume~10.
\newblock Cambridge University Press, 2019.

\bibitem[Schropp and Singer(2000)]{schropp2000dynamical}
Johannes Schropp and I~Singer.
\newblock A dynamical systems approach to constrained minimization.
\newblock \emph{Numerical functional analysis and optimization}, 21\penalty0 (3-4):\penalty0 537--551, 2000.

\bibitem[Sebbouh et~al.(2021)Sebbouh, Gower, and Defazio]{sebbouh2021almost}
Othmane Sebbouh, Robert~M Gower, and Aaron Defazio.
\newblock Almost sure convergence rates for stochastic gradient descent and stochastic heavy ball.
\newblock In \emph{Conference on Learning Theory}, 2021.

\bibitem[Shabat(1992)]{shabat1992introduction}
Boris~V Shabat.
\newblock \emph{Introduction to complex analysis {Part II.} Functions of Several Variables}, volume 110 of \emph{Translations of Mathematical Monographs}.
\newblock American Mathematical Society, 1992.

\bibitem[Shardlow(2006)]{shardlow2006modified}
Tony Shardlow.
\newblock Modified equations for stochastic differential equations.
\newblock \emph{BIT Numerical Mathematics}, 46:\penalty0 111--125, 2006.

\bibitem[Shi et~al.(2022)Shi, Du, Jordan, and Su]{shi2021understanding}
Bin Shi, Simon~S Du, Michael~I Jordan, and Weijie~J Su.
\newblock Understanding the acceleration phenomenon via high-resolution differential equations.
\newblock \emph{Mathematical Programming}, 195:\penalty0 79–148, 2022.

\bibitem[Steeb and Hardy(2011)]{steeb2011matrix}
Willi-Hans Steeb and Yorick Hardy.
\newblock \emph{Matrix calculus and Kronecker product: a practical approach to linear and multilinear algebra}.
\newblock World Scientific Publishing Company, 2011.

\bibitem[Su et~al.(2016)Su, Boyd, and Candes]{su2016differential}
Weijie Su, Stephen Boyd, and Emmanuel~J Candes.
\newblock A differential equation for modeling {Nesterov's} accelerated gradient method: Theory and insights.
\newblock \emph{Journal of Machine Learning Research}, 17\penalty0 (153):\penalty0 1--43, 2016.

\bibitem[Tops{\o}e(2007)]{topsoe12007some}
Flemming Tops{\o}e.
\newblock Some bounds for the logarithmic function.
\newblock \emph{Inequality theory and applications}, 4:\penalty0 137, 2007.

\bibitem[Vent-Tsel’ and Freidlin(1973)]{vent1973some}
AD~Vent-Tsel’ and MI~Freidlin.
\newblock Some problems concerning stability under small random perturbations.
\newblock \emph{Theory of Probability \& Its Applications}, 17\penalty0 (2):\penalty0 269--283, 1973.

\bibitem[Wang and Wu(2024)]{wang2024theoretical}
Mingze Wang and Lei Wu.
\newblock A theoretical analysis of noise geometry in stochastic gradient descent.
\newblock \emph{arXiv}, page 2310.00692, 2024.

\bibitem[Wilson et~al.(2017)Wilson, Roelofs, Stern, Srebro, and Recht]{wilson2017marginal}
Ashia~C Wilson, Rebecca Roelofs, Mitchell Stern, Nati Srebro, and Benjamin Recht.
\newblock The marginal value of adaptive gradient methods in machine learning.
\newblock In \emph{Advances in Neural Information Processing Systems}, 2017.

\bibitem[Xie et~al.(2020)Xie, Sato, and Sugiyama]{xie2020diffusion}
Zeke Xie, Issei Sato, and Masashi Sugiyama.
\newblock A diffusion theory for deep learning dynamics: Stochastic gradient descent exponentially favors flat minima.
\newblock In \emph{International Conference on Learning Representations}, 2020.

\bibitem[Xie et~al.(2022)Xie, Wang, Zhang, Sato, and Sugiyama]{xie2022adaptive}
Zeke Xie, Xinrui Wang, Huishuai Zhang, Issei Sato, and Masashi Sugiyama.
\newblock Adaptive inertia: Disentangling the effects of adaptive learning rate and momentum.
\newblock In \emph{International Conference on Machine Learning}, 2022.

\bibitem[Yang et~al.(2024)Yang, Li, Fatkhullin, and He]{yang2024two}
Junchi Yang, Xiang Li, Ilyas Fatkhullin, and Niao He.
\newblock Two sides of one coin: the limits of untuned {SGD} and the power of adaptive methods.
\newblock In \emph{Advances in Neural Information Processing Systems}, 2024.

\bibitem[Zhang et~al.(2021)Zhang, Bengio, Hardt, Recht, and Vinyals]{zhang2021understanding}
Chiyuan Zhang, Samy Bengio, Moritz Hardt, Benjamin Recht, and Oriol Vinyals.
\newblock Understanding deep learning (still) requires rethinking generalization.
\newblock \emph{Commun. ACM}, 64\penalty0 (3):\penalty0 107–115, 2021.

\bibitem[Zhou et~al.(2020)Zhou, Feng, Ma, Xiong, Hoi, and E]{zhou2020towards}
Pan Zhou, Jiashi Feng, Chao Ma, Caiming Xiong, Steven Chu~Hong Hoi, and Weinan E.
\newblock Towards theoretically understanding why {SGD} generalizes better than {Adam} in deep learning.
\newblock In \emph{Advances in Neural Information Processing Systems}, 2020.

\bibitem[Zhou et~al.(2022)Zhou, Liang, and Zhang]{Zhou2022}
Yi~Zhou, Yingbin Liang, and Huishuai Zhang.
\newblock Understanding generalization error of {SGD} in nonconvex optimization.
\newblock \emph{Machine Learning}, 111\penalty0 (1):\penalty0 345--375, 2022.

\bibitem[Zhu et~al.(2019)Zhu, Wu, Yu, Wu, and Ma]{zhu2019anisotropic}
Zhanxing Zhu, Jingfeng Wu, Bing Yu, Lei Wu, and Jinwen Ma.
\newblock The anisotropic noise in stochastic gradient descent: Its behavior of escaping from sharp minima and regularization effects.
\newblock In \emph{International Conference on Machine Learning}, 2019.

\end{thebibliography}

\newpage

\tableofcontents

\appendix

\section{Construction of \oursde{}} \label{sec:construct_sPF}
\subsection{Proof of \Cref{lemma:semi_expand}}
\begin{proof}
In the semi-group expansion of \Cref{eq:SBEA_continuous}, consider the term
$\frac{1}{n!}\eta^{n} \cL^{n} u(x)$ for $1 \leq n \leq p+1$. Since we already have
$\eta^{n}$, to get $\eta^{p+1}$ , $\cL^{n} u(x)$ must contribute $\eta^{p+1-n}$.
As $\cL_i$ contains $\eta^i$, we obtain the conclusion.
\end{proof}

\subsection{Proof of \Cref{lemma:SBEA_conditions}}
We start with a sketch of the proof for \Cref{lemma:SBEA_conditions}, where we first determine the functional structure of the components $g_i$'s and $h_i$'s in the SBEA ansatz in \Cref{eq:SBEA_unknown} and then determine their exact expression using a generating function based approach. We prove via induction.
The complexity in our analysis emerges from identifying all possible terms generated by the expansion of
$\cL_{l_1} \cL_{l_2}  \cdots \cL_{l_n}$ in \Cref{eq:expand_p}.
The collective expansion of these operators unfolds as a multinomial series.
This sets our work apart from the PF, as in the latter, in the absence of diffusion terms, the expansion generates only a monomial.
The rigorous proofs of the statements made in the induction are provided afterwards.

Before providing the inductive proof, statement (1) in \Cref{lemma:SBEA_conditions}, i.e. the uniqueness of $g_p$ and $h_p$, can be easily obtained from the following argument.
\begin{proof}[Proof of point (1) in \Cref{lemma:SBEA_conditions}]
We first give examples for solving the first few $g_i$ and $h_i$, then
we use an argument of induction to finish the proof.
For the terms associated with $\eta^0$,
\Cref{eq:SBEA_discrete} and \Cref{eq:SBEA_continuous} already match, i.e.,
$u(x)$.
For terms associated with $\eta$, we can solve that $g_0(x) = -\nabla f(x)$
and $h_0(x) = 0$.
Considering terms with $\eta^2$, we have for discrete-time (\Cref{eq:SBEA_discrete})
\begin{align*}
  \eta^2 \cdot \frac{1}{2} \nabla^2 u(x) : \left(\nabla f(x) \nabla f(x)^\T + \Sigma(x)\right),
\end{align*}
and for continuous-time(\Cref{eq:SBEA_continuous})
\begin{align*}
  \eta^2 \cdot \left( \underbrace{g_1(x)^\T \nabla u(x)  + \frac{1}{2}h_1(x) : \nabla^2 u(x) }_{\text{from $\eta \cL u(x)$}} 
  + \underbrace{ \frac{1}{2} (-\nabla f \cdot \nabla (-\nabla f \cdot \nabla u)) (x)}_{\text{from $\frac{1}{2}\eta^2 \cL^2 u(x)$}} \right).
\end{align*}
Note that this holds for any proper $u$, so we group terms associated with 
$\nabla u$ and $\nabla^2 u$ and match the terms in discrete-time and 
continuous-time. In this way, we obtain 2 equations, which gives us the following
$g_1$ and $h_1$:
\begin{align} \label{eq:f}
  g_1(x) = -\frac{1}{2} \nabla^2 f(x) \nabla f(x) \quad \text{and} \quad
  h_1(x) = \Sigma(x).
\end{align}

Now assume we already know $\left\{g_i\right\}_{1 \leq i < p}$ and $\left\{h_i\right\}_{1 \leq i < p}$.
We will proceed to solve $g_p$ and $h_p$. This is done by solving terms with
$\eta$ of order $p+1$.
The terms associated with $\eta^{p+1}$ are shown in \Cref{lemma:semi_expand}.
When $n=1$, we will have linear terms with $g_p$ and $h_p$, respectively.
When $n \geq 2$, since $l_1 + l_2 + \cdots + l_n$ could only sum to $p-1$, we
only have $g_i$ and $h_i$ for $i < p$. Given that we assume
$\left\{g_i\right\}_{1 \leq i < p}$ and $\left\{h_i\right\}_{1 \leq i < p}$ are
already solved, by grouping terms with $\nabla u$ and $\nabla^2 u$, we obtain 
two linear equations and can solve for $g_p$ and $h_p$. Note that the
coefficients for $g_p$ and $h_p$ in these two linear equations are non-zero,
therefore, the solution exists and is unique.
\end{proof}

We now give a sketch of the inductive proof of the second statement in \Cref{lemma:SBEA_conditions}, on the functional structure of the components $h_p$'s and $g_p$'s. 
Recall point (2) in \Cref{lemma:SBEA_conditions}: The components $g_p$ and $h_p$ admit the form
\begin{align}
  g_p(x) =&\ c_p \cdot \left(\nabla^2 f(x)\right)^p \nabla f(x),  \\
  h_p(x) =&\ 
 \sum_{k = 0}^{p-1} a_{k, p-1-k} \cdot \left(\nabla^2 f(x)\right)^{k} \Sigma(x)
 \left(\nabla^2 f(x)\right)^{p-1-k} ,
\end{align}
where $\{c_k\}$ and $\{a_{k, p-1-k}\}$ are constants to be determined.
Let $q$ be the induction index and let \Cref{eq:gp,eq:hp} be the induction hypothesis for $q=p \geq 1$.
Such a hypothesis clearly holds for $q=1$ (recall the calculation in \Cref{remark:example_SBEA}).
The following lemma shows that the nice functional structures in \Cref{eq:gp,eq:hp} are preserved for $q = p+1$.
\begin{lemma} \label{lemma_induction}
Under \Cref{principle_hessian_aware}, suppose that \Cref{eq:gp,eq:hp} hold in step $q = p$. They also hold for $q = p+1$.
\end{lemma}
\begin{proof}[Proof sketch.]
    To establish the induction, we need to enumerate all possible terms generated by the expansion in \Cref{eq:expand_p}.
    Viewing the application of each $\cL_{l_i}$ as a layer, this process entails a dual-stage selection mechanism: 
    \begin{enumerate}
  \item Within each layer we consider the terms generated from either $\eta^{l_i} g_{l_i} \cdot \nabla$ or $\frac{\eta^{l_i}}{2} h_{l_i}: \nabla^2$
    in $\cL_{l_i}$. \label{proof_sketch_lemma_induction_point_1}
  \item For a fixed layer,  we consider how the operators $\nabla$ and $\nabla^2$ are applied on the subsequent layer. For example, for the operator $\nabla$, it could be acting on $g_{l_{i+1}}$ to yield a concrete function $\nabla g_{l_{i+1}}$, or it can engage with the operator $\nabla^2$ from the subsequent layer to form the operator $\nabla^3$. 
  \label{proof_sketch_lemma_induction_point_2}
\end{enumerate}
Finally, for every possible \texttt{sequence-of-selections} made in the above mechanism, any remaining differential operators $\nabla^j$ will be applied on the test function $u$.
Hereafter, we use the typewriter font to emphasize that the \texttt{sequence-of-selections} is with respect to the above selection mechanism.

Unraveling all potential selection in
\Cref{eq:expand_p} is inherently difficult. However, \Cref{principle_hessian_aware} allows us to ensures the particular structures of $g_p$
and $h_p$. For the following discussion, we introduce the concept of \texttt{free-nabla-$j$}: During unraveling of the operator compositions in \Cref{eq:expand_p}, a \texttt{free-nabla-$j$} is generated if we obtain an operator $\nabla^j$ during any stage of the above selection mechanism.

We have the following result.  
\begin{lemma} \label{lemma_free_nabla_j}
    Once \texttt{free-nabla-$j$}, for $j \geq 3$, is generated, one can terminate the subsequent operator expansion as all the resulting terms will be excluded by \Cref{principle_hessian_aware}.
\end{lemma}
See an elaborated discussion in \Cref{sec:apdx_proof_sec_ha}. From this observation, the expansion of \Cref{eq:expand_p} can be significantly simplified and one can show that if the induction in \Cref{eq:gp,eq:hp} holds for $q=p$, $g_{p+1}$ and $h_{p+1}$ admit the following form:
\begin{align}
  g_{p+1}(x) =&\ c_{p+1} \cdot \left(\nabla^2 f(x)\right)^{p+1} \nabla f(x) \label{eq:g_form}  \\
  h_{p+1}(x) =&\ 
 \sum_{k = 0}^{p} (a_{k, p-k} + b_{k, p-k} \cdot \|\nabla f(x)\|^2) \cdot \left(\nabla^2 f(x)\right)^{k} \Sigma(x)
 \left(\nabla^2 f(x)\right)^{p-k}.
  \label{eq:h_form}
\end{align}
Further, a detailed examination of the combinations reveals the following result.
\begin{lemma} \label{lemma_b_is_0}
    $b_{k, p-k} \equiv 0$.
\end{lemma}
The above results together establish the induction step of $q=p+1$.
\end{proof}

Once we have determine the functional structure of $g_p$ and $h_p$,
the next step is to determine the coefficients $a_{i,j}$ and $c_{p}$. As we
obtain the structures of $g_p$ and $h_p$ by induction, the coefficients admit
recursion forms. To resolve these coefficients, we utilize an argument based on the generating functions. As a result, we obtain point (3) in \Cref{lemma:SBEA_conditions}.

\subsubsection{Proof of \Cref{lemma_induction}}
\label{sec:apdx_proof_sec_ha}
To establish the inductive step in \Cref{lemma_induction}, we prove the following points in this section.
\begin{enumerate}[label=(\alph*)]
    \item We first prove \Cref{lemma_free_nabla_j}, which allows us to significantly simplify the expansion of the compositions of the generators $\cL_i$'s. \label{lemma_induction_proof_point_a}
    \item We analyze all possible terms generated without incurring \texttt{free-nabla-$j$}, for $j \geq 3$, and find that there are only three possibilities. This leads to the functional structure described in \Cref{eq:g_form,eq:h_form}. \label{lemma_induction_proof_point_b}
    \item We further provide the recursive definitions of the coefficients $\{c_{i}\}$, $\{a_{i, j}\}$, and $\{b_{i, j}\}$. This allows us to conclude that $b_{k, p-k} \equiv 0$ (\Cref{lemma_b_is_0}) and also it allows us to calculate the limit of the power series $\{g_p\}$ and $\{h_p\}$, i.e. point (3) in \Cref{lemma:SBEA_conditions}. \label{lemma_induction_proof_point_c}
\end{enumerate}

\subsubsection{Proof of \Cref{lemma_free_nabla_j}}
Let us recall the concept of \texttt{free-nabla-$j$}: During unraveling of the operator compositions in \Cref{eq:expand_p}, a \texttt{free-nabla-$j$} is generated if we obtain an operator $\nabla^j$ during any stage of the selection mechanism mentioned in the proof sketch of \Cref{lemma_induction}.

In the following, we will show that once \texttt{free-nabla-$j$} for $j\geq 3$ is generated in a \texttt{sequence-of-selections} mentioned in the proof sketch of \Cref{lemma_induction}, the term resulting from this sequence is excluded by \Cref{principle_hessian_aware}, i.e. it contains at least one of the terms $\nabla^{(r)} f(x), \nabla^{(s)} u(x)$ and $\nabla^{(m)} \Sigma(x)$ for $r, s\geq 3$ and $m \geq 1$.
To establish this result, there are two possibilities once \texttt{free-nabla-$j$} for $j\geq 3$ is generated:
\begin{enumerate}
    \item The \texttt{free-nabla-$j$} operator is to be directly applied on the test function $u$. In this case, all the resulting term from this \texttt{sequence-of-selections} contains $\nabla^{(j)} u(x)$, $j\geq 3$, which is excluded by \Cref{principle_hessian_aware}.
    \item The \texttt{free-nabla-$j$} operator is to be applied on a generator $\cL_i$, i.e. we encounter the term
    \begin{equation} \label{eqn_free_nabla_j_generator}
        \nabla^{(j)} (\eta^i g_i \cdot \nabla + \frac{1}{2}\eta^i h_i : \nabla^ 2).
    \end{equation}
    One can prove that (see the discussion below) all terms in the expansion of the above expression either is excluded by \Cref{principle_hessian_aware} or it contains \texttt{free-nabla-$q$} for $q \geq j$\footnote{For simplicity, we take $j=3$ as example. For $j> 3$, the logic is similar. One has for the first term in \Cref{eqn_free_nabla_j_generator}
    \begin{align*}
        \nabla^{(3)} \left( g_i \cdot \nabla\right) = \underbrace{g_i \cdot \nabla^{(4)}}_{\texttt{free-nabla-$4$}} + \underbrace{\nabla g_i \cdot \nabla^{(3)}}_{\texttt{free-nabla-$3$}} + \underbrace{\nabla^2 g_i \cdot \nabla^{(2)} + \nabla^3 g_i \cdot \nabla}_{\text{every term contains $\nabla^{r} f$, $r\geq 3$}}.
    \end{align*}
    By the product rule of the differential and based on the induction in \Cref{eq:g_form}, one can check that in the expansions of $\nabla^2 g_i$ and $\nabla^3 g_i$, all terms contain $\nabla^{r} f$, for $r\geq 3$. Hence, the last two terms in the above expression are excluded by \Cref{principle_hessian_aware}. The second term in \Cref{eqn_free_nabla_j_generator} can be excluded with a similar argument.}. Consequently, the order $j$ does not decrease after the operator \texttt{free-nabla-$j$} being applied on a generator $\cL_i$. 
\end{enumerate}
One can repeat the case 2 until the \texttt{free-nabla-$j$} operator is directly applied on the test function $u$, which reduces to case 1. From the above argument, all terms generated are to be excluded by \Cref{principle_hessian_aware}.

\subsubsection{Proof of the Functional Structures in \Cref{eq:g_form,eq:h_form}}
According to \Cref{lemma_free_nabla_j}, we only need to consider the \texttt{sequence-of-selections} mentioned in the proof sketch of \Cref{lemma_induction} such that no \texttt{free-nabla-$j$} operator is generated, for $j\geq 3$. It turns out that there are only three possible \texttt{sequence-of-selections}, as depicted in \Cref{fig:construct_spf}.

In this section, we elaborate on these three possible sequences, based on which we then establish the functional structures in \Cref{eq:g_form,eq:h_form}. In the meantime, we also derive the recursive definitions of the coefficients $\{c_{i}\}$, $\{a_{i, j}\}$, and $\{b_{i, j}\}$.

\paragraph{Notation} For the ease of the proof, we define for $n \geq 1$, $m \geq 0$ and sequence $\{c_i\}_{i=0}^{+\infty}$,
\begin{align*}
  \rho\left(n, m, \{c_i\}_{i=0}^{+\infty}\right) = \sum_{l_1 + l_2 + \dots l_n = m} c_{l_1} c_{l_2}
  \cdots c_{l_n},
\end{align*}
i.e., sum of all possible combinations of $n$ items from sequence $\{c_i\}_{i=0}^{+\infty}$
such that the sum of indices is $m$.
A way to look at this is through generating functions. Let $c(x) = \sum_{i=0}^{+\infty} c_i x^i$,
then $\rho\left(n, m, \{c_i\}_{i=0}^{+\infty}\right)$ is equivalent to
the coefficient of $x^m$ in $c(x)^n$.

\begin{figure}[t]
  \centering
  \begin{subfigure}{0.3\textwidth}
    \includegraphics[width=\textwidth]{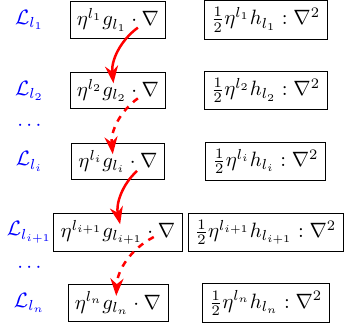}
    \caption{Sequence Type I}
    \label{fig:construct_spf_p1}
  \end{subfigure}
  \hfill
  \begin{subfigure}{0.3\textwidth}
    \includegraphics[width=\textwidth]{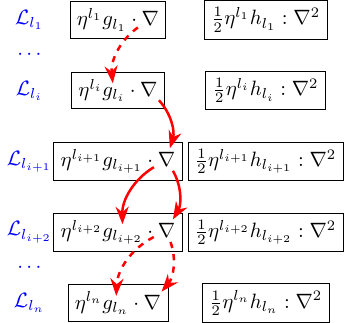}
    \caption{Sequence Type II}
    \label{fig:construct_spf_p2}
  \end{subfigure}
  \hfill
  \begin{subfigure}{0.3\textwidth}
    \includegraphics[width=\textwidth]{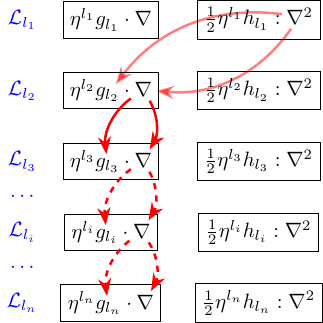}
    \caption{Sequence Type III}
    \label{fig:construct_spf_p3}
  \end{subfigure}
  \caption{Illustrations of possible combinations in the construction of \oursde.
In constructing \oursde, we examine the product of
$\cL_{l_1}\cL_{l_2}\cdots\cL_{l_n}$, where each $\cL_{l_i}$ contains two terms.
The red arrows indicate the remaining Sequences after filtering out those with
higher-order derivatives of $f$, $u$, or $\Sigma$, which are excluded in the
construction rules of \oursde. These arrows originate from $\nabla$ or $\nabla^2$
and lead to $g_{l_i}$ or $\nabla$, demonstrating how the $\nabla$ from the
previous operator is applied in the subsequent operator. The dashed arrows represent
the exclusion of repeated operations.}%
  \label{fig:construct_spf}
\end{figure}

  \paragraph{Solve for $g_p$}
  We claim that we
  can solve $g_p$ by matching terms with $\nabla u$. To show the claim, we
  observe that $\nabla u$ terms could only be achieved by choosing $g_{l_i}$
  and never selecting $h_{l_i}$ for for all $1 \leq i \leq n$. This is the case
  of our Sequence Type I illustrated in \Cref{fig:construct_spf_p1}.
  To see why this
  is the case, we start with $\cL_{l_n}$, since we want terms with $\nabla
  u(x)$, we must select $g_{l_n}$ instead of $h_{l_n}$. Next, when selecting
  from $\cL_{l_{n-1}}$, we could also only select $g_{l_{n-1}}$, since if
  $h_{l_{n-1}}$ is chosen, we would have $h_{l_{n-1}} : \nabla^2 \left(g_{l_n} \cdot \nabla u\right)$.
  The $\nabla^2$ operation must be applied to $g_{l_n}$ otherwise we will have
  $\nabla^2 u$ instead of $\nabla u$. However, when we apply the gradient operator
  twice in $g_{l_n}$, there must be a third-order gradient factor, which we
  choose to exclude, popping up as we have $g_{l_n}$ of form \Cref{eq:gp}. Following this
  logic recursively, we have that $\nabla u$ terms contain only $g_i$ (with $0
  \leq i \leq p -1$). Also note that the gradient operator in $g_{l_i} \cdot \nabla$
  must be applied to $g_{l_{i+1}}$ except when $i=n$ where the gradient operator
  is applied to $u$, otherwise we will have $\nabla^2$ operator and face the 
  same issue as when we selecting $h_{l_i}$.

  Since now in the $\nabla u$ terms only contains $g_p$ (with non-zero
  coefficient) without $h_p$, we can solve for $g_p$. Again, since the gradient
  operator from $g_{l_i}$ (except $i=n$) could only be used for the next
  $g_{l_{i+1}}$, specifically applied on the $\nabla f(x)$ factor of
  $g_{l_{i+1}}$ otherwise we would have third-order gradient of $f(x)$, we have
  $n - 1$ factors of $\nabla f(x)$ changed to $\nabla^2 f(x)$ and only one
  remains. Therefore, $g_p$ is of form \Cref{eq:gp}.
  
  The remaining proof for finding $c_p$ is almost the same as the proof
  for Theorem~A.2 of \citet{rosca2022continuous}. The resulting $c_i$ is the
  coefficients of the Taylor expansion of $c(x) = \frac{\log(1-x)}{x}$, i.e.,
  $c(x) = \sum_{i=0}^{+\infty} c_i x^i$.

  \paragraph{Solve for $h_p$}
  In the proof $g_p$, we have shown that the $\nabla u$
  terms allow us to solve $g_p$, and now we proceed to solve the form of $h_p$
  by considering $\nabla^2 u$ terms. Again, to solve $h_p$, we need terms
  with the order of $\eta$ summing up to $p+1$ and consider \Cref{lemma:semi_expand}.
  Let us discuss what are the valid possibilities of selections from such an 
  expansion.

  Let us call the selection between $g_{l_i}$ and $h_{l_i}$ the $i$-th step
  of selection. Assume we have the gradient operator $\nabla^q$ after the $i$-th step.
  If in the next step, we 
  select $g_{l_{i+1}} \cdot \nabla$, since $g_{l_{i+1}}$ has only one factor of
  $\nabla f$, which means we can only apply the gradient operator once on this $g_{l_{i+1}}$,
  the order of the remaining gradient operator passed to the next step is still at least $q$.
  If in the $i+1$-th step, we select $h_{l_{i+1}} : \nabla^2$ instead, all the gradient operators
  should go to the next step (otherwise, if applied on $h_{l_{i+1}}$, factors we exclude
  will emerge), resulting in a $\nabla^{q+2}$. Note that by our
  construction, in the end, the operator applied on $u$ should not exceed 2.
  Therefore no matter what we select from the first step, which produces at least
  $\nabla$, we should only select $g_{l_i}$ instead $h_{l_i}$ for $i > 1$.
  Because otherwise we would have at least $\nabla^3$ passing to $u$.
  According to the above reasoning, we have only two cases left:
  \begin{enumerate}
    \item \textbf{If we select $g_{l_1}$ in the first step}, which corresponds to
      \Cref{fig:construct_spf_p2}, then in order
      to get $\nabla^2 u$ at the end, we will need at a step $i$ with $2 \leq i \leq n$,
      the gradient operator is not applied to $g_{l_i}$ but passed to the next
      step. For other steps, the operator is applied to $g$. This is the only
      way to have $\nabla^2 u$. Plugging in the solution of $g_i$,
      we know that the resulting term in \Cref{eq:SBEA_continuous} is of form
      \begin{align*}
        \eta^{p+1} \sum_{s=0}^{p-1} b_{s, p-1-s}
        \Tr{ \nabla f(x) \left(\nabla^2 f(x)\right)^{s} \nabla^2 u(x) \left(\nabla^2 f(x)\right)^{p-1-s} \nabla f(x)^\T},
      \end{align*}
      where $b_{s, p-1-s}$ can be written as
      \begin{align}
        b_{s, p-1-s} 
        & = \sum_{n=2}^{p+1} \frac{1}{n!} \sum_{i=2}^n
        \sum_{q=0}^{n-i} \binom{n-i}{q} \rho\left(q+1, p-1-s-q, \{c_k\}_{k=0}^{+\infty}\right) \nonumber \\
        &\fakeeq \rho\left(n-1-q, s+q-n+2, \{c_k\}_{k=0}^{+\infty}\right), \label{eq:ae}
      \end{align}
      where $c_k = -\frac{1}{1+k}$, the coefficient associated with $g_i$.
      Note that the coefficient is for terms with $\nabla^2 f(x)$ to the power
      of $s$ on the left of $\nabla^2 u$ and $p-1-s$ on the right.
      To see why this is the case, we have the following selection procedure:
      \begin{enumerate}
        \item Such terms can be found in $\frac{\eta^n}{n!}\cL^n u$
          for $n=2$ to $p+1$.
        \item As mentioned before, we have at step $2 \leq i \leq n$, 
          the gradient operator is not applied to $g_{l_i}$.
          Note that in this case the $\nabla^2 f(x)$ coming from $g_{l_i}$
          will contribute to the RHS of $\nabla^2 u$.
          Then we will have for the remaining steps of selection,
          \begin{align*}
            \nabla^2 \left( g_{l_{i+1}} \cdot \nabla \left( g_{l_{i+2}} \cdot \nabla 
            \left(\cdots g_{l_{n}} \cdot \nabla u \right) \right) \right).
          \end{align*}
        \item \label{item:step_3} According to \Cref{lemma:nabla_2}, there are $\binom{n-i}{q}$
          combinations such that the number of $\nabla g$ resulting in the
          RHS of $\nabla^2 u$ is $q$.
        \item For the RHS of $\nabla^2 u$, we have $g_{l_i}$ and $q$ 
          of $\nabla g$ selected from the last step.
          Note that $g_k$ would contains $\left(\nabla^2 f(x)\right)^{k}$ for
          $k \geq 0$, and for the $q$ of $\nabla g$, we obtain
          additional $\left(\nabla^2 f(x)\right)^q$ because of the $\nabla$
          operator. Therefore, the total number of $\nabla^2 f(x)$ on the
          RHS of $\nabla^2 u$ comes from $g_{l_i}$ and $q$ of $\nabla g$
          and an additional $q$. Therefore, we select from $q+1$ of $g$,
          whose indices sum up to $p-1-s-q$ (so that the total order of
          $\nabla f(x)$ is $p-1-s-q + q = p-1-s$).
        \item For the LHS of $\nabla^2 u$, we have $g_{1}$, $\{\nabla g_{l_k}\}_{k=2}^{i-1}$ and $n-i-q$ of
          $\nabla g$ selected according to Step~\ref{item:step_3}.
          Following similar logic as the last step, in total we select $1 + i - 2
          + n - i - q = n - 1 - q$ of $g$, whose indices sum up to $s + q - n + 2$.
      \end{enumerate}

      We note that for any matrix $A$, $\Tr{A} = \Tr{A^\T}$, so
      the traces associated with $b_{s, m}$ and $b_{m, s}$ are the same for 
      $s, m \geq 0$. According to \Cref{lemma:coef_h_1}, $b_{s, m} + b_{m, s} = 0$
      for $m + s \geq 1$. Therefore, the coefficients cancel each other.
      For the case of $b_{s, s}$ with $s \geq 1$, \Cref{lemma:coef_h_1}
      also implies $b_{s, s} + b_{s, s} = 0 \implies b_{s, s} = 0$.

      Therefore, in summary, in this whole case, there is no term generated.

    \item \textbf{If we select $h_{l_1}$ in the first step}, which corresponds to
      \Cref{fig:construct_spf_p3}, then in order
      to get $\nabla^2 u$ at the end, we will in the following steps apply
      the gradient at $g_{l_k}$ for $k \geq 2$ at once and pass to the next
      step the other gradient operator. 
      In this way, all terms of \Cref{eq:SBEA_continuous} with order of $\eta$ summing to $p+1$
      and coming from $\frac{1}{n!} \cL^n u$, has the following form:
      \begin{align} \label{eq:ad}
        \frac{1}{2}\eta^{p+1} \tilde{a}_{i, j, k} \Tr{\left(\nabla^2 f(x)\right)^i \Sigma(x) 
        \left(\nabla^2 f(x)\right)^j \nabla^2 u(x) \left(\nabla^2 f(x)\right)^k},
      \end{align}
      for some absolute constants $\tilde{a}_{i, j, k}$ with $i + j + k = p-1$.
      The reason why we have this form is because of our induction assumption
      and \Cref{lemma:nabla_2}.
      The reason why $i + j + k = p-1$ is that to have $\eta^{p+1}$, we must have
      $l_1 + l_2 + \cdots + l_n = p + 1 - n$. For the first step, we have $h_{l_1}$.
      According to the induction assumption, it can offer $(l_1 - 1)$-th order of
      $\nabla^2 f(x)$. For the remaining steps, we have $\nabla g_{l_s}$,
      which can offer $(l_s + 1)$-th order of $\nabla^2 f(x)$. Therefore, in total
      we have $\nabla f(x)$ to the power of $l_1 - 1 + l_2 + l_3 + \cdots + l_n + n - 1 = p + 1 - n - 1 + n - 1 = p - 1$.
      According to the property of the trace operator, \Cref{eq:ad} can also
      be written as 
      \begin{align*}
        \frac{1}{2}\eta^{p+1} \tilde{a}_{i, j, k} \Tr{\left(\nabla^2 f(x)\right)^{k + i} \Sigma(x) 
        \left(\nabla^2 f(x)\right)^j \nabla^2 u(x)}.
      \end{align*}

      Now we try to match the $\eta^{p+1}$ terms in \Cref{eq:SBEA_discrete}
      and \Cref{eq:SBEA_continuous}.
      Note that in \Cref{eq:SBEA_discrete}, the term associated with
      $\eta^{p+1}$ for $p \geq 2$ is $0$, therefore we have the following
      equality for solving $h_p$:
      \begin{align*}
        \frac{1}{2}\eta^{p+1}\left(
        \Tr{h_{p}(x)^\T \nabla^2 u(x)} +
          \sum_{n=2}^{p+1} 
        \tilde{a}_{i, j, k} \Tr{\left(\nabla^2 f(x)\right)^{k + i} \Sigma(x) 
        \left(\nabla^2 f(x)\right)^j \nabla^2 u(x)} \right) = 0,
      \end{align*}
      which implies (since this should hold for any $u$)
      \begin{align*}
        h_{p}(x) = - \sum_{n=2}^{p+1} 
        \tilde{a}_{i, j, k} \left(\nabla^2 f(x)\right)^{j} \Sigma(x) 
        \left(\nabla^2 f(x)\right)^{k+i}.
      \end{align*}
      So far, by induction, we have proved that $h_p$ has the form of \Cref{eq:gp}
      (regardless of the constants). Next we will try to find these constants.
      We will stick to the notation in \Cref{lemma:SBEA_conditions}, i.e., using $a_{i, j}$
      for the rest of the proof, i.e., 
      \begin{align*}
  h_p(x) = 
 \sum_{k = 0}^{p-1} a_{k, p-1-k} \cdot \left(\nabla^2 f(x)\right)^{k} \Sigma(x)
  \left(\nabla^2 f(x)\right)^{p-1-k}.
      \end{align*}
      Similar to the proof of \Cref{eq:ae}, we have
      \begin{align}
        a_{k, p-1-k}  \nonumber
        &= - \sum_{n=2}^p \frac{1}{n!} \sum_{l=0}^k \sum_{r=0}^{p-1-k}
        a_{l, r} \sum_{q=0}^{n-1} \binom{n-1}{q} \rho\left(q, k-q-l, \{c_k\}_{k=0}^{+\infty}\right) \\
        &\fakeeq \rho\left(n-1-q, p-k-n+q-r, \{c_k\}_{k=0}^{+\infty}\right). \label{eqn:definition_a}
      \end{align}
      The reason is that we have the following selection process:
      \begin{enumerate}
        \item Such terms can be found in $\frac{\eta^n}{n!}\cL^n u$
          for $n=2$ to $p+1$.
        \item We try to find term with
          \begin{align*}
            \Tr{\left(\nabla^2 f(x)\right)^{p-1-k} \Sigma \left(\nabla^2 f(x)\right)^{k} \nabla^2 u},
          \end{align*}
          which can be found from Sequence
          \begin{align*}
            \underbrace{\left(\nabla^2 f(x)\right)^l \Sigma \left(\nabla^2 f(x)\right)^r}_{h_{l_1}} :
            \left(\nabla^2 f(x)\right)^{k-l} \nabla^2 u \left(\nabla^2 f(x)\right)^{p-1-k-r}.
          \end{align*}
        \item For $\left(\nabla^2 f(x)\right)^{k}$, we can have $\left(\nabla^2 f(x)\right)^{l}$ coming from
          $h_{l_1}$. Also $h_{l_1}$ can contribute to $\left(\nabla f(x)\right)^{p-1-k}$
          with $\left(\nabla^2 f(x)\right)^{r}$.
          That is why we have summations over $l$ and $r$ with coefficients $a_{l, r}$.
        \item \label{item:step_3_2}  As mentioned before, we have for the selection at step $2 \leq k \leq n$,
          \begin{align*}
            \nabla^2 \left( g_{l_2} \cdot \nabla \left( g_{l_3} \cdot \nabla 
            \left(\cdots g_{l_{n}} \cdot \nabla u \right) \right) \right).
          \end{align*}
        According to \Cref{lemma:nabla_2}, there are $\binom{n-1}{q}$
          combinations such that the number of $\nabla g$ resulting in the
          LHS of $\nabla^2 u$ is $q$.
        \item For the LHS of $\Sigma(x)$, we have $q$ of $\nabla g$ selected from the last step.
          The indices should sum up to $k - q - l$, since we will generate
          additional $q$ of $\nabla^2 f(x)$ and $h_{l_1}$ provides $l$.
          Therefore, we have coefficient $\rho\left(q, k-q-l, \{c_k\}_{k=0}^{+\infty}\right)$
          similarly for the RHS.

      \end{enumerate}

  \end{enumerate}
\subsubsection{Proof of \Cref{lemma_b_is_0} and the Recursive Expressions of $\{c_{i}\}$ and $\{a_{i, j}\}$}
\begin{lemma} \label{lemma:coef_h_1}
  Recall the recursive definition of $b_{s, m}$ in \Cref{eq:ae},
  \begin{align*}
    b_{s, m}
    = \sum_{n = 2}^{s+m+2} \frac{1}{n!} \sum_{i=2}^{n} \sum_{q=0}^{n-i}
    \binom{n-i}{q} \rho\left(q+1, m-q, \{c_k\}_{k=0}^{+\infty}\right)
    \rho\left(n-1-q, s+q-n+2, \{c_k\}_{k=0}^{+\infty}\right),
  \end{align*}
  where $s, m \geq 0$, $s+m \geq 1$, $c_k = -\frac{1}{k+1}$. We have
  $b_{s, m} + b_{m, s} = 0$.
\end{lemma}
\begin{proof}
  Let $c(x) = \sum_{k=0}^{+\infty} c_k x^k$, and we know that $c(x) = \frac{\log(1 - x)}{x}$.
  Define
  \begin{align*}
    b(x, y) = \sum_{s, m \geq 0}^{+\infty} b_{s, m} x^s  y^m.
  \end{align*}
  Then we have
\ifdefined\mpversion
  \begin{align*}
    &\fakeeq b(x, y) \\
    &= \sum_{s, m \geq 0}^{+\infty} x^s  y^m 
    \sum_{n = 2}^{s+m+2} \frac{1}{n!} \sum_{i=2}^{n} \sum_{q=0}^{n-i}
    \binom{n-i}{q} \rho\left(q+1, m-q, \{c_k\}_{k=0}^{+\infty}\right)
    \rho\left(n-1-q, s+q-n+2, \{c_k\}_{k=0}^{+\infty}\right) \\
    &= \frac{1}{y} \sum_{n = 2}^{+\infty} \frac{1}{n!} \sum_{s + m \geq n - 2}^{+\infty} x^s  y^{m+1}  \sum_{q=1}^{n-2}
    \binom{n-1}{q} \rho\left(q, m-q+1, \{c_k\}_{k=0}^{+\infty}\right)
    \rho\left(n-q, s+q-n+1, \{c_k\}_{k=0}^{+\infty}\right),
  \end{align*}
\else
  \begin{align*}
    &\fakeeq b(x, y) \\
    &= \sum_{s, m \geq 0}^{+\infty} x^s  y^m 
    \sum_{n = 2}^{s+m+2} \frac{1}{n!} \sum_{i=2}^{n} \sum_{q=0}^{n-i}
    \binom{n-i}{q} \rho\left(q+1, m-q, \{c_k\}_{k=0}^{+\infty}\right)
    \rho\left(n-1-q, s+q-n+2, \{c_k\}_{k=0}^{+\infty}\right) \\
    &= \sum_{n = 2}^{+\infty} \frac{1}{n!} \sum_{s + m \geq n - 2}^{+\infty} x^s  y^m \sum_{i=2}^{n} \sum_{q=0}^{n-i}
    \binom{n-i}{q} \rho\left(q+1, m-q, \{c_k\}_{k=0}^{+\infty}\right)
    \rho\left(n-1-q, s+q-n+2, \{c_k\}_{k=0}^{+\infty}\right) \\
    &= \sum_{n = 2}^{+\infty} \frac{1}{n!} \sum_{s + m \geq n - 2}^{+\infty} x^s  y^m  \sum_{q=0}^{n-2} \sum_{i=2}^{n-q}
    \binom{n-i}{q} \rho\left(q+1, m-q, \{c_k\}_{k=0}^{+\infty}\right)
    \rho\left(n-1-q, s+q-n+2, \{c_k\}_{k=0}^{+\infty}\right) \\
    &= \sum_{n = 2}^{+\infty} \frac{1}{n!} \sum_{s + m \geq n - 2}^{+\infty} x^s  y^m  \sum_{q=0}^{n-2}
    \binom{n-1}{q+1} \rho\left(q+1, m-q, \{c_k\}_{k=0}^{+\infty}\right)
    \rho\left(n-1-q, s+q-n+2, \{c_k\}_{k=0}^{+\infty}\right) \\
    &= \sum_{n = 2}^{+\infty} \frac{1}{n!} \sum_{s + m \geq n - 2}^{+\infty} x^s  y^m  \sum_{q=1}^{n-2}
    \binom{n-1}{q} \rho\left(q, m-q+1, \{c_k\}_{k=0}^{+\infty}\right)
    \rho\left(n-q, s+q-n+1, \{c_k\}_{k=0}^{+\infty}\right) \\
    &= \frac{1}{y} \sum_{n = 2}^{+\infty} \frac{1}{n!} \sum_{s + m \geq n - 2}^{+\infty} x^s  y^{m+1}  \sum_{q=1}^{n-2}
    \binom{n-1}{q} \rho\left(q, m-q+1, \{c_k\}_{k=0}^{+\infty}\right)
    \rho\left(n-q, s+q-n+1, \{c_k\}_{k=0}^{+\infty}\right),
  \end{align*}
\fi
  where in the forth equality, we used hockey-stick identity~\citep{ross1997generalized}, and in the
  fifth equality, we replace $q$ with $q-1$.
  Note that
  \begin{align} \label{eq:aa}
    \sum_{s + m \geq n - 2}^{+\infty} x^s  y^{m+1}  \sum_{q=1}^{n-2}
    \binom{n-1}{q} \rho\left(q, m-q+1, \{c_k\}_{k=0}^{+\infty}\right)
    \rho\left(n-q, s+q-n+1, \{c_k\}_{k=0}^{+\infty}\right)
  \end{align}
  is equivalent to
  \begin{align} \label{eq:ab}
    \left(y c(y) + x c(x) \right)^{n-1} c(x) - \left(x c(x)\right)^{n-1} c(x).
  \end{align}
  To see why this is the case, let us first looked at $\left(y c(y) + x c(x) \right)^{n-1} c(x)$.
  To have $y^{m+1}$, we can select $yc(y)$ from $\left(y c(y) + x c(x) \right)^{n-1}$
  for $q$ times, which results in coefficients $\binom{n-1}{q}$ and 
 $\rho\left(q, m-q+1, \{c_k\}_{k=0}^{+\infty}\right)$. For $x^s$, we select
 $n-1-q$ of $xc(x)$ from $\left(y c(y) + x c(x) \right)^{n-1}$,
 and obtain $\rho\left(n-q, s+q-n+1, \{c_k\}_{k=0}^{+\infty}\right)$.
 Note that in \Cref{eq:aa}, the sum of $q$ starts from $1$, which means we
 have not considered the case where $yc(y)$ is never selected in
$\left(y c(y) + x c(x) \right)^{n-1}$.
That is why in \Cref{eq:ab}, we subtract $\left(x c(x)\right)^{n-1} c(x)$.

Next, we proceed with
\begin{align}
  b(x, y)
    &= \frac{1}{y} \sum_{n = 2}^{+\infty} \frac{1}{n!} 
    \left(\left(y c(y) + x c(x) \right)^{n-1} c(x) - \left(x c(x)\right)^{n-1} c(x)\right) \nonumber \\
    &= \frac{1}{y} \sum_{n = 2}^{+\infty} \frac{1}{n!} \left(y c(y) + x c(x) \right)^{n-1} c(x)
    - \frac{1}{y} \sum_{n = 2}^{+\infty} \frac{1}{n!} \left(x c(x)\right)^{n-1} c(x). \label{eq:ac}
\end{align}
For the first part, we have
\begin{align*}
  \frac{1}{y} \sum_{n = 2}^{+\infty} \frac{1}{n!} \left(y c(y) + x c(x) \right)^{n-1} c(x)
  &= \frac{c(x)}{y \left(y c(y) + x c(x) \right)} \sum_{n = 2}^{+\infty} \frac{1}{n!} \left(y c(y) + x c(x) \right)^{n} \\
  &= \frac{c(x)}{y \left(y c(y) + x c(x) \right)} \left(e^{y c(y) + x c(x)} - 1 - \left(y c(y) + x c(x) \right)\right),
\end{align*}
where in the last equality we used Taylor expansion 
$e^x = \sum_{n=0}^{+\infty} \frac{1}{n!} x^n$.
Plugging in the definition of $c(x)$, we have
\ifdefined\mpversion
\begin{align*}
  &\fakeeq \frac{c(x)}{y \left(y c(y) + x c(x) \right)} \left(e^{y c(y) + x c(x)} - 1 - \left(y c(y) + x c(x) \right)\right) \\
  &= \frac{\log(1-x)}{x y \log\left((1-x)(1-y)\right) } \left(x y - x - y  - \log\left((1-x)(1-y)\right)\right).
\end{align*}
\else
\begin{align*}
  &\fakeeq \frac{c(x)}{y \left(y c(y) + x c(x) \right)} \left(e^{y c(y) + x c(x)} - 1 - \left(y c(y) + x c(x) \right)\right) \\
  &= \frac{\log(1-x)}{x y \log\left((1-x)(1-y)\right) } \left((1 - x)(1-y) - 1 - \log\left((1-x)(1-y)\right)\right) \\
  &= \frac{\log(1-x)}{x y \log\left((1-x)(1-y)\right) } \left(x y - x - y  - \log\left((1-x)(1-y)\right)\right).
\end{align*}
\fi
For the second part of \Cref{eq:ac},
\begin{align*}
    \frac{1}{y} \sum_{n = 2}^{+\infty} \frac{1}{n!} \left(x c(x)\right)^{n-1} c(x)
   &= \frac{1}{x y} \sum_{n = 2}^{+\infty} \frac{1}{n!} \left(x c(x)\right)^{n}
   = \frac{1}{x y} \left(e^{xc(x)} - 1 - xc(x)\right) \\
   &= \frac{1}{x y} \left(1 - x - 1 - \log(1-x)\right)
   = \frac{1}{x y} \left( - x - \log(1-x)\right).
\end{align*}
Combining the two parts, we get
\begin{align*}
  b(x, y) 
  &= \frac{\log(1-x)}{x y \log\left((1-x)(1-y)\right) } \left(x y - x - y  - \log\left((1-x)(1-y)\right)\right)
   + \frac{1}{x y} \left( x + \log(1-x)\right) \\
  &=\frac{\log(1-x)}{x y \log\left((1-x)(1-y)\right) } \left(x y - x - y  \right)
   + \frac{1}{y}.
\end{align*}
Now to prove the required result, we consider the generator function
\begin{align*}
  &\fakeeq b(x, y) + b(y, x) \\
  &= \frac{\log(1-x)}{x y \log\left((1-x)(1-y)\right) } \left(x y - x - y  \right)
   + \frac{1}{y}
   + \frac{\log(1-y)}{x y \log\left((1-x)(1-y)\right) } \left(x y - x - y  \right)
   + \frac{1}{x} \\
  &= \frac{1}{xy} (xy - x - y ) + \frac{1}{x} + \frac{1}{y}
  = 1.
\end{align*}
Note that the coefficients of $b(x, y) + b(y, x)$ for $x^s y^m$ with $s + m \geq 1$
is $b_{s, m} + b_{m, s}$, which means $b_{s, m} + b_{m, s} = 0$.

\end{proof}
\begin{lemma} \label{lemma:coef_h_2}
  Define $a_{0, 0} = 1$ and for $s, m \geq 0$ and $s + m \geq 1$, recall the recursive definition in \Cref{eqn:definition_a},
  \begin{align*}
    a_{s, m} 
    &= - \sum_{n=2}^{s+m+1} \frac{1}{n!} \sum_{l=0}^{s} \sum_{r=0}^m
    a_{l, r} \sum_{q=0}^{n-1} \binom{n-1}{q} \rho\left(q, s-q-l, \{c_k\}_{k=0}^{+\infty}\right) \\
     &\fakeeq    \rho\left(n-1-q, m+1-n+q-r, \{c_k\}_{k=0}^{+\infty}\right),
  \end{align*}
  where $c_k = -\frac{1}{1+k}$. Then we have $a(x, y) = \sum^{+\infty}_{s, m \geq 0} a_{s, m} x^s y^m = \frac{\log((1-x)(1-y))}{x y - x - y}$.
\end{lemma}
\begin{proof}
 We begin with 
 \begin{align*}
    &\fakeeq a(x, y)
    = \sum_{s, m \geq 0}^{+\infty} a_{s, m} x^s y^m \\
    &= 1 - \sum_{s, m \geq 1}^{+\infty} x^s y^m \sum_{n=2}^{s+m+1} \frac{1}{n!} \sum_{l=0}^{s} \sum_{r=0}^m
    a_{l, r} \sum_{q=0}^{n-1} \binom{n-1}{q} \rho\left(q, s-q-l, \{c_k\}_{k=0}^{+\infty}\right) \\
    &\fakeeq \fakeeq \cdot 
        \rho\left(n-1-q, m+1-n+q-r, \{c_k\}_{k=0}^{+\infty}\right)\\
    &= 1 - \sum_{n=2}^{+\infty} \frac{1}{n!} 
    \sum_{s + m \geq n - 1}^{+\infty} x^s y^m \sum_{l=0}^{s} \sum_{r=0}^m
    a_{l, r} \sum_{q=0}^{n-1} \binom{n-1}{q} \rho\left(q, s-q-l, \{c_k\}_{k=0}^{+\infty}\right) \\
    &\fakeeq \fakeeq \cdot
        \rho\left(n-1-q, m+1-n+q-r, \{c_k\}_{k=0}^{+\infty}\right)\\
 \end{align*}
 Similar to the proof technique used in \Cref{lemma:coef_h_1}, we note that
 \begin{align*}
       \sum_{s + m \geq n - 1}^{+\infty} x^s y^m \sum_{l=0}^{s} \sum_{r=0}^m
    a_{l, r} \sum_{q=0}^{n-1} \binom{n-1}{q} \rho\left(q, s-q-l, \{c_k\}_{k=0}^{+\infty}\right)
        \rho\left(n-1-q, m+1-n+q-r, \{c_k\}_{k=0}^{+\infty}\right)
 \end{align*}
 is equivalent to $a(x, y) \left(x c(x) + y c(y)\right)^{n-1}$, where $c(x) = \sum_{k=0}^{+\infty} c_k x^k = \frac{\log(1-x)}{x}$.
 To constitute $x^s y^m$, we first select $a_{l, r} x^l y^r$ from $a(x, y)$,
 and then select $q$ times $xc(x)$ from $\left(x c(x) + y c(y)\right)^{n-1}$.

 Then we have
 \begin{align*}
   a(x, y) 
   &= 1 - \sum_{n=2}^{+\infty} \frac{1}{n!} a(x, y) \left(x c(x) + y c(y)\right)^{n-1} \\
   &= 1 - \frac{a(x, y)}{x c(x) + y c(y)} \sum_{n=2}^{+\infty} \frac{1}{n!} \left(x c(x) + y c(y)\right)^{n} \\
   &= 1 - \frac{a(x, y)}{x c(x) + y c(y)} \left(e^{x c(x) + y c(y)} - 1 - \left(x c(x) + y c(y)\right)\right).
 \end{align*}
 Plugging in the definition of $c(x)$, we obtain
 \begin{align*}
   a(x, y) 
   &= 1 - \frac{a(x, y)}{\log\left((1-x)(1-y)\right)} \left( (1 - x)(1-y) - 1 - \log\left((1-x)(1-y)\right) \right) \\
   &= 1 - \frac{a(x, y)}{\log\left((1-x)(1-y)\right)} \left( xy - x - y - \log\left((1-x)(1-y)\right) \right).
 \end{align*}
 Solving this equation for $a(x, y)$ gives us the desired result.
\end{proof}

  \subsection{Proof of \Cref{lemma:key}}
  Next, we need to show that $b(x)$ and $\mathcal{D}(x)$ is convergent.
Solving for $b(x)$ is similar to the proof of \citet[Theorem~A.2]{rosca2022continuous}.
Here, we present the derivation of the diffusion term. We have
\begin{align*}
  \mathcal{D}(x) = \sum_{p=0}^{+\infty} \eta^p h_p(x) = \sum_{p=0}^{+\infty} 
 \eta^p \sum_{k = 0}^{p-1} a_{k, p-1-k} \cdot \left(\nabla^2 f(x)\right)^{k} \Sigma(x)
  \left(\nabla^2 f(x)\right)^{p-1-k},
\end{align*}
where $a_{k, p-1-k}$ is determined in \Cref{lemma:SBEA_conditions}.
The above equation implies
\begin{align*}
  U^\T \mathcal{D}  U
  &= U^\T \left( \sum_{p=0}^{+\infty} 
 \eta^p \sum_{k = 0}^{p-1} a_{k, p-1-k} \cdot \left(\nabla^2 f(x)\right)^{k} \Sigma(x)
  \left(\nabla^2 f(x)\right)^{p-1-k} \right) U \\
  &= \sum_{p=0}^{+\infty} 
 \eta^p \sum_{k = 0}^{p-1} a_{k, p-1-k} \cdot \Lambda^{k} U^\T \Sigma(x)
  U \Lambda^{p-1-k}.
\end{align*}
Since $\Lambda$ is a diagonal matrix, we have
\ifdefined\mpversion
\begin{align*}
  \left[U^\T \mathcal{D}  U\right]_{i, j} 
  &= \sum_{p=0}^{+\infty} 
 \eta^p \sum_{k = 0}^{p-1} a_{k, p-1-k} \cdot \lambda_i^{k} \left[U^\T \Sigma(x)
 U\right]_{i, j} \lambda_j^{p-1-k} \\
  &= \eta  \left[U^\T \Sigma(x) U\right]_{i, j} \sum_{k=0}^{+\infty} 
 \sum_{q = 0}^{+\infty} a_{k, q} \cdot (\eta \lambda_i)^{k} (\eta \lambda_j)^{q},
\end{align*}
\else
\begin{align*}
  \left[U^\T \mathcal{D}  U\right]_{i, j} 
  &= \sum_{p=0}^{+\infty} 
 \eta^p \sum_{k = 0}^{p-1} a_{k, p-1-k} \cdot \lambda_i^{k} \left[U^\T \Sigma(x)
 U\right]_{i, j} \lambda_j^{p-1-k} \\
  &= \eta \sum_{p=0}^{+\infty} 
 \sum_{k = 0}^{p-1} a_{k, p-1-k} \cdot (\eta \lambda_i)^{k} \left[U^\T \Sigma(x)
 U\right]_{i, j} (\eta \lambda_j)^{p-1-k} \\
  &= \eta \sum_{k=0}^{+\infty} 
 \sum_{p = k+1}^{+\infty} a_{k, p-1-k} \cdot (\eta \lambda_i)^{k} \left[U^\T \Sigma(x)
 U\right]_{i, j} (\eta \lambda_j)^{p-1-k} \\
  &= \eta  \left[U^\T \Sigma(x) U\right]_{i, j} \sum_{k=0}^{+\infty} 
 \sum_{q = 0}^{+\infty} a_{k, q} \cdot (\eta \lambda_i)^{k} (\eta \lambda_j)^{q},
\end{align*}
\fi
where in the last equality we let $q = p - 1 -k$. Then according to \Cref{lemma:SBEA_conditions},
we have
\begin{align*}
  \left[U^\T \mathcal{D} U\right]_{i, j} = \eta \left[U^\T \Sigma(x) U\right]_{i, j} a(\eta \lambda_i, \eta \lambda_j),
\end{align*}
where $a(x, y) = \frac{\log(1-x)(1-y)}{x y - (x + y)}$.

\subsection{Proof of \Cref{prop:SPF_exist}}
In the first condition, the matrix $U^\T \Sigma U$ becomes diagonal.
We can check that
the diagonal elements are always positive definite,
and by taking square root of the eigenvalues, we obtain
\begin{align} \label{eq:analytic_D}
  D = U \sqrt{\widetilde{\Lambda} \eta \frac{\log\left(1 - \eta \Lambda\right)^2}{\left(1 - \eta \Lambda\right)^2 - 1}} U^\T,
\end{align}
where $\Sigma = U \widetilde{\Lambda} U^\T$ and $\nabla^2 f(x) = U \Lambda U^\T$.

For the second condition, we consider small stepsize regime.
We will proceed to show that the RHS matrix of \Cref{eq:sPF_diff}, denoted as $M$
is positive semi-definite, so that such a square root always exists.
  Consider matrix $K$, whose $(i, j)$-th element is defined as
  \begin{align*}
    K_{i, j} = \frac{\log(1-\eta \lambda_i) + \log(1-\eta \lambda_j)}{(1 - \eta \lambda_i)(1- \eta \lambda_j) - 1}.
  \end{align*}
  We consider small $\eta$, i.e., $\eta < \frac{1 - \frac{\sqrt{2}}{2}}{\norm{\nabla^2 f(x)}}$, 
  therefore, $K$ is real.
  The matrix $M$ can then be written as
  \begin{align*}
    M = \eta U^\T \Sigma U \odot K = \eta U^\T \Sigma U \odot \left(\mathbf{1} + K - \mathbf{1}\right)
     = \eta \left(U^\T \Sigma U + U^\T \Sigma U \odot \left(K - \mathbf{1}\right)\right),
  \end{align*}
  where $\mathbf{1}$ is a all-one matrix.
  Now since $\eta > 0$, to determine the positive semi-definiteness of M, it is 
  sufficient to determine positive semi-definiteness of $U^\T \Sigma U + U^\T \Sigma U \odot \left(K - \mathbf{1}\right)$.

  Note that for matrix $U^\T \Sigma U$, since $U$ is orthogonal, it only changes
  the eigenspace of $\Sigma$, but does not change the eigenvalues. To see why
  this is the case, assume $\lambda$ is a eigenvalue of $\Sigma$ with $v$ being the corresponding
  eigenvector. Then we have $\Sigma v = \lambda v$.
  Then $U^{-1}v$ is a eigenvector for $U^\T \Sigma U$ with eigenvalue $\lambda$, since
  \begin{align*}
    U^\T \Sigma U U^{-1} v = U^\T \Sigma v = U^\T \lambda v = \lambda U^{-1} v.
  \end{align*}

  Going back to our problem, since $U^\T \Sigma U$ is positive definite,
  a sufficient condition for matrix $U^\T \Sigma U + U^\T \Sigma U \odot \left(K - \mathbf{1}\right)$ 
  to be PSD is that 
  \begin{align} \label{eq:aac}
    \lambda_{\min}\left(U^\T \Sigma U\right) = \lambda_{\min}\left(\Sigma\right) \geq \norm*{U^\T \Sigma U \odot \left(K - \mathbf{1}\right)}.
  \end{align}
  Let us look at $\norm*{U^\T \Sigma U \odot \left(K - \mathbf{1}\right)}$,
  we know that
\ifdefined\mpversion
  \begin{align*}
    \norm*{U^\T \Sigma U \odot \left(K - \mathbf{1}\right)}
    &\leq \sqrt{d} \norm*{U^\T \Sigma U} \sup_{i, j} \left| \left[K - \mathbf{1}\right]_{i, j} \right| \\
    &= \sqrt{d} \lambda_{\max}\left(\Sigma\right)
    \sup_{i, j} \left| \frac{\log\left((1-\eta \lambda_i)(1-\eta \lambda_j)\right)}{(1 - \eta \lambda_i)(1- \eta \lambda_j) - 1} - 1 \right|,
  \end{align*}
\else
  \begin{align*}
    \norm*{U^\T \Sigma U \odot \left(K - \mathbf{1}\right)}
    &\leq \sqrt{d} \norm*{U^\T \Sigma U} \sup_{i, j} \left| \left[K - \mathbf{1}\right]_{i, j} \right| \\
    &= \sqrt{d} \norm*{U^\T \Sigma U}
    \sup_{i, j} \left| \frac{\log\left((1-\eta \lambda_i)(1-\eta \lambda_j)\right)}{(1 - \eta \lambda_i)(1- \eta \lambda_j) - 1} - 1 \right| \\
    &= \sqrt{d} \lambda_{\max}\left(\Sigma\right)
    \sup_{i, j} \left| \frac{\log\left((1-\eta \lambda_i)(1-\eta \lambda_j)\right)}{(1 - \eta \lambda_i)(1- \eta \lambda_j) - 1} - 1 \right|,
  \end{align*}
\fi
  where we used \Cref{lemma:hadamard_norm} for the inequality.

  According to \Cref{lemma:ub_elm}, 
  \begin{align*}
    \sup_{i, j} \left| \frac{\log\left((1-\eta \lambda_i)(1-\eta \lambda_j)\right)}{(1 - \eta \lambda_i)(1- \eta \lambda_j) - 1} - 1 \right|
    \leq \sup_{i, j} 1 - \left(1 - \max\{|\eta \lambda_i|, |\eta \lambda_j|\}\right)^2
    \leq 1 - \left(1 - \eta \norm{\nabla^2 f(x)} \right)^2
  \end{align*}
  Plugging in back to \Cref{eq:aac}, a sufficient condition for $M$ to be PSD is
  \begin{align*}
    \lambda_{\min}\left(\Sigma\right) \geq \sqrt{d} \lambda_{\max}\left(\Sigma\right) \left(1 - \left(1 - \eta \norm{\nabla^2 f(x)} \right)^2\right).
  \end{align*}
  Rearranging the terms gives us the condition in the theorem.

\subsection{Proof of \Cref{prop:well_posed}}
  To show the result, it is sufficient to show that the eigenvalues of $D$ are
  lower bounded away from $0$. For the first condition, we have an explicit
  form of $D$ in \Cref{eq:analytic_D}. Note that when the abstract value of all 
  entries of $\eta \Lambda$ are smaller than 1, entries of $\log(1- \eta \Lambda)^2 / \left(\left(1-\eta \Lambda\right)^2 - 1\right)$
  are lower bounded away from $0$. Therefore, as long as eigenvalues of $\Sigma$
  are lower bounded, i.e., $\Sigma$ is positive definite, the eigenvalues of $D$
  are lower bounded from 0. For the second condition,
  Similar to the proof for \Cref{prop:SPF_exist}, we know that as long as the
  stepsize satisfies the condition, the eigenvalues of the diffusion coefficient
  are positive and lower bounded. 

  Now, we show that why lower boundedness of eigenvalues of $D$ is sufficient to prove
  the result. We need to show that the drift term $b$ and diffusion term $D$
  are Lipschitz.
  According to \Cref{lemma:bound_drift}, the drift term is Lipschitz,
  and according to \Cref{lemma:bound_diffu},
  we know that $DD^\T$ is Lipschitz. Therefore, $\partial [DD^\T]_{i,j}/ \partial x_k$
  is upper bounded.
  Then according to \citet[Equation~(7)]{lin2017improved}, we have
  \begin{align*}
    \vecto\left(\frac{\partial D_{i, j}}{\partial DD^\T}\right) = (D \otimes I + I \otimes D)^{-1} \vecto\left(\tilde{\mathbf{1}}_{i, j}\right),
  \end{align*}
  where $\vecto$ is the vectorization of matrices and $\otimes$ denotes the Kronecker product.
  $\tilde{\mathbf{1}}_{i,j}$ is a matrix whose $(i, j)$-th element is 1 and
  all other elements are 0. Next, we get
\ifdefined\mpversion
  \begin{align*}
    \frac{\partial D_{i, j}}{\partial x_k} 
    &= \sum_{p, q} \left[(D \otimes I + I \otimes D)^{-1} \vecto\left(\tilde{\mathbf{1}}_{i, j}\right)\right]_{qd+p} \frac{\partial [DD^\T]_{p, q}}{\partial x_k} \\
    &\leq \sum_{p, q} \norm{ (D \otimes I + I \otimes D)^{-1} \vecto\left(\tilde{\mathbf{1}}_{i, j}\right) } \abs{\frac{\partial [DD^\T]_{p, q}}{\partial x_k} } \\
    &\leq \sum_{p, q} \norm{ (D \otimes I + I \otimes D)^{-1} }   \abs{\frac{\partial [DD^\T]_{p, q}}{\partial x_k} }.
  \end{align*}
\else
  \begin{align*}
    \frac{\partial D_{i, j}}{\partial x_k} 
    &= \sum_{p, q} \frac{\partial D_{i, j}}{\partial [DD^\T]_{p, q}} \frac{\partial [DD^\T]_{p, q}}{\partial x_k} \\
    &= \sum_{p, q} \left[(D \otimes I + I \otimes D)^{-1} \vecto\left(\tilde{\mathbf{1}}_{i, j}\right)\right]_{qd+p} \frac{\partial [DD^\T]_{p, q}}{\partial x_k} \\
    &\leq \sum_{p, q} \norm{ (D \otimes I + I \otimes D)^{-1} \vecto\left(\tilde{\mathbf{1}}_{i, j}\right) } \abs{\frac{\partial [DD^\T]_{p, q}}{\partial x_k} } \\
    &\leq \sum_{p, q} \norm{ (D \otimes I + I \otimes D)^{-1} }   \abs{\frac{\partial [DD^\T]_{p, q}}{\partial x_k} }.
  \end{align*}
\fi
  It remains to show that the eigenvalues of $(D \otimes I + I \otimes D)^{-1}$ is bounded.
  According to \citet[Theorem~2.15]{steeb2011matrix}, the eigenvalues of
  $D \otimes I + I \otimes D$ are $\{a + b \mid a, b \in \{\lambda_i\}_{i=1}^d \}$,
  where $\lambda_i$ are the eigenvalues of $D$. Therefore, the eigenvalues
  of $D \otimes I + I \otimes D$ are lower bounded since we have the eigenvalues
  of $D$ being lower bounded by a positive constant. This implies that
  eigenvalues of $(D \otimes I + I \otimes D)^{-1}$ are upper bounded, which
  concludes the proof.

\subsection{Helper Lemmas}
\begin{lemma} \label{lemma:ub_elm}
  For any $x, y \in \mathbb{R}$ and $\frac{\sqrt{2}}{2} - 1 \leq x, y \leq 1 - \frac{\sqrt{2}}{2}$,
  it holds that
  \begin{align*}
    \left|\frac{\log((1-x)(1-y))}{(1-x)(1-y) - 1} - 1 \right| \leq 1 - \left(1 - \max\{|x|, |y|\}\right)^2.
  \end{align*}
\end{lemma}
\begin{proof}
  Let $z = (1-x)(1-y)$, then the LHS becomes $\left|\frac{\log z}{z-1} -1 \right|$.
  We first study the function
\ifdefined\mpversion 
$f(z) = \frac{\log z}{z - 1} - 1$,
  whose gradient is given by $f'(z) = \frac{\frac{z-1}{z}-\log z}{(z-1)^2} = \frac{1 - \frac{1}{z} - \log z}{(z-1)^2} \leq 0$,
\else
  \begin{align*}
    f(z) = \frac{\log z}{z - 1} - 1,
  \end{align*}
  whose gradient is given by
  \begin{align*}
    f'(z) = \frac{\frac{z-1}{z}-\log z}{(z-1)^2} = \frac{1 - \frac{1}{z} - \log z}{(z-1)^2} \leq 0,
  \end{align*}
\fi
  where the last inequality comes from the fact $\log z \geq 1 - \frac{1}{z}$ for $z > 0$.
  Also note that $f(1) = 0$, therefore, we know that
  \begin{align*}
    |f(z)| = \begin{cases} \frac{\log z}{z - 1} - 1, \quad z \leq 1 \\ 
        1 - \frac{\log z}{z - 1}, \quad z > 1.
    \end{cases}
  \end{align*}
  According to previous analysis, the maximum of $\left| \frac{\log z}{z - 1} - 1 \right|$
  would be at the minimum or maximum possible value of $z$.
  Denoting $m \coloneqq \max\{|x|, |y|\}$ with $0 \leq m \leq 1 - \frac{\sqrt{2}}{2}$,
  then we know that
  \begin{align*}
    \left|\frac{\log((1-x)(1-y))}{(1-x)(1-y) - 1} - 1 \right| 
    \leq \max\left\{ \frac{\log(1-m)^2}{(1-m)^2-1} - 1, 
    1-\frac{\log(1+m)^2}{(1+m)^2-1}\right\}
  \end{align*}
  Next, we will show that for $0 \leq m \leq 1$,
\ifdefined\mpversion 
$\frac{\log(1-m)^2}{(1-m)^2-1} - 1 \geq 1-\frac{\log(1+m)^2}{(1+m)^2-1}$.
\else
  \begin{align*}
    \frac{\log(1-m)^2}{(1-m)^2-1} - 1 \geq 1-\frac{\log(1+m)^2}{(1+m)^2-1}.
  \end{align*}
\fi
  We start with
  \begin{align*}
    &\fakeeq \frac{\log(1-m)^2}{(1-m)^2-1} - 1 - \left(1-\frac{\log(1+m)^2}{(1+m)^2-1}\right) \\
    &= \frac{2(m+2)\log(1-m) + 2(m-2)\log(1+m) - 2m(m+2)(m-2)}{m(m+2)(m-2)}.
  \end{align*}
  Since $m(m+2)(m-2)<0$, it suffice to show that the numerator is less or equal than 0.
  We let
\ifdefined\mpversion 
$g(m) = 2(m+2)\log(1-m) + 2(m-2)\log(1+m) - 2m(m+2)(m-2)$,
\else
  \begin{align*}
    g(m) = 2(m+2)\log(1-m) + 2(m-2)\log(1+m) - 2m(m+2)(m-2),
  \end{align*}
\fi
  whose gradient is
  \begin{align*}
    g'(m) &= \frac{m+2}{m-1} + \log(1-m) + \frac{m-2}{1+m} + \log(1+m) - 3m^2 + 4 \\
          &\leq \frac{m+2}{m-1} + \frac{m-2}{1+m} - 4m^2 + 4
          = \frac{3m^2(m^2-3)}{(1-m)(m+1)} \leq 0,
  \end{align*}
  where the inequality holds because $\log z \leq z - 1$ for $z > 0$.
  Therefore, for $0 \leq m \leq 1$, we have $g(m) \leq g(0) = 0$. So far, we
  have proved that
  \begin{align*}
    \left|\frac{\log((1-x)(1-y))}{(1-x)(1-y) - 1} - 1 \right| 
    \leq \frac{\log(1-m)^2}{(1-m)^2-1} - 1.
  \end{align*}
  Next, we will prove that an upper bound is as follows
\ifdefined\mpversion 
$\frac{\log(1-m)^2}{(1-m)^2-1} - 1 \leq 1 - (1-m)^2$,
\else
  \begin{align*}
    \frac{\log(1-m)^2}{(1-m)^2-1} - 1 \leq 1 - (1-m)^2,
  \end{align*}
\fi
  for $0 \leq m \leq 1 - \frac{\sqrt{2}}{2}$.
  With a change of variable for $z = (1-m)^2$, to prove the above inequality is
  the same as showing the following for $\frac{1}{2} \leq z \leq 1$,
  \begin{align*}
    \frac{\log z}{z-1} - 1 - (1 - z) = \frac{\log z}{z-1} + z - 2
    = \frac{\log z + z^2 - 3z + 2}{z - 1} \leq 0.
  \end{align*}
  Then it is sufficient to show
\ifdefined\mpversion 
$h(z) \coloneqq \log z + z^2 - 3z + 2 \geq 0$.
\else
  \begin{align*}
    h(z) \coloneqq \log z + z^2 - 3z + 2 \geq 0.
  \end{align*}
\fi
  We have
  \begin{align*}
    h'(z) = \frac{1}{z} + 2z -3 = \frac{(z-1)(2z-1)}{z} \leq 0.
  \end{align*}
  Therefore $h(z) \geq h(1) = 0$ for $\frac{1}{2} \leq z \leq 1$.
  The whole proof is finished.
\end{proof}

\begin{lemma} \label{lemma:nabla_2}
  Let $u: \mathbb{R}^{d} \to \mathbb{R}$ and $\{v_i(x)\}_{i=1}^{n}$ be vector fields, i.e., $v_i: \mathbb{R}^d \to \mathbb{R}^d$.
  The terms in the result of
  \begin{align*}
    \nabla^2 \left( v_1 \cdot \nabla \left( v_2 \cdot \nabla \left( v_3 \cdot \nabla 
    \left(\cdots v_n \cdot \nabla u \right) \right) \right) \right),
  \end{align*}
  that contain only $\nabla^2 u$ and $\nabla v_i$ for $i \in [1, n]$ are
  \begin{align*}
    \sum_{S \subseteq [1, n]} \left(\prod_{i \in \widetilde S} \nabla v_{i}\right)
    \nabla^2 u \left(\prod_{i \in \widehat{ [1, n] \setminus S}} \nabla v_{i}\right),
  \end{align*}
  where $\cdot$ is the inner product between vectors, $\widetilde S$ is the ascending ordered set containing all elements
  from $S$, $\widehat S$ is the corresponding descending ordered set,
  and $\setminus$ is the set difference operator.
\end{lemma}

\begin{proof}
  We will prove this by induction. First, let us consider the case
  of $n=1$, then clearly, we have
\ifdefined\mpversion 
  $\nabla^2 \left(v_1 \cdot \nabla u\right)
    = \nabla v_1 \nabla^2 u + \nabla^2 u \nabla v_1 + \cC$,
  \else
  \begin{align*}
    \nabla^2 \left(v_1 \cdot \nabla u\right)
    = \nabla v_1 \nabla^2 u + \nabla^2 u \nabla v_1 + \cC,
  \end{align*}
\fi
  where $\cC$ contains irrelevant terms, i.e., terms with higher order
  gradients of $v_i$ and $u$.
  Now assume the conclusion holds true for $n=1, \dots, m-1$, then
  when $n = m$, we have
  \begin{align*}
    &\fakeeq \nabla^2 \left( v_1 \cdot \nabla \left( v_2 \cdot \nabla \left( v_3 \cdot \nabla 
    \left(\cdots v_m \cdot \nabla u \right) \right) \right) \right) \\
    &= \nabla v_1 \nabla^2 \left( v_2 \cdot \nabla \left( v_3 \cdot \nabla 
    \left(\cdots v_m \cdot \nabla u \right) \right) \right)
    + \nabla^2 \left( v_2 \cdot \nabla \left( v_3 \cdot \nabla 
    \left(\cdots v_m \cdot \nabla u \right) \right) \right) \nabla v_1 + \cC\\
    &= \nabla v_1 
    \sum_{S \subseteq [2, m]} \left(\prod_{i \in \widetilde S} \nabla v_{i}\right)
    \nabla^2 u \left(\prod_{i \in \widehat{ [2, m] \setminus S}} \nabla v_{i}\right)
    + 
    \sum_{S \subseteq [2, m]} \left(\prod_{i \in \widetilde S} \nabla v_{i}\right)
    \nabla^2 u \left(\prod_{i \in \widehat{ [2, m] \setminus S}} \nabla v_{i}\right)
   \nabla v_1 + \cC \\
    &= \sum_{S \subseteq [1, m]} \left(\prod_{i \in \widetilde S} \nabla v_{i}\right)
    \nabla^2 u \left(\prod_{i \in \widehat{ [1, m] \setminus S}} \nabla v_{i}\right)
    + \cC,
  \end{align*}
  which concludes the proof.
\end{proof}

\begin{lemma} \label{lemma:hadamard_norm}
  We have for any $A, B \in \mathbb{R}^{d \times d}$,
  \begin{align*}
    \| A \odot B \| \leq \sqrt{d} \sup_{i, j} A_{i, j} \| B \|.
  \end{align*}
\end{lemma}
\begin{proof}
  Let $e_1, e_2, \dots, e_d$ be the standard basis of $\mathbb{R}^d$. Then
  we have for any $1 \leq k \leq d$,
  \begin{align*}
    \| (A \odot B) e_k \| \leq \sup_{i, j} A_{i, j} \| B e_k \| \leq \sup_{i, j} A_{i, j} \| B \|.
  \end{align*}
  Then for any vector $v \in \mathbb{R}^d$, which can be written as
  $v = v_1 e_1 + v_2 e_2 + \cdots$, we have
  \begin{align*}
    \| (A \odot B) v \| \leq \sum_{k=1}^d |v_k| \| (A \cdot B) e_k \|
    \leq \sup_{i, j} A_{i, j} \|B\| \sum_{k=1}^d |v_k| \leq \sup_{i, j} A_{i, j} \|B\| \sqrt{d} \|v\|,
  \end{align*}
  where the last inequality we used Cauchy-Schwartz inequality.
\end{proof}

\section{Approximation Error Analysis}

\subsection{Proofs for Section~\ref{sec:approx}}
\label{sec:apx_proofs_approx}

\begin{proof}[Proof for \Cref{thm:weak_approx}]
The proof follows the idea of Theorem~2.2 of \citet{feng2017semi}, however our
proof is more difficult since we need to deal with more complicated \oursde.
The assumption of $\| F(x; \xi)\|_{c^7} < \infty$ implies $\| f(x)\|_{c^7} < \infty$ and
$\| \Sigma(x) \|_{c^6} < \infty$.
According to \Cref{lemma:bound_drift}, we have boundedness of $\norm{b(x)}_{C^5}$,
where $b(x)$ is the drift term of \oursde.
According to \Cref{lemma:bound_drift}, we have boundedness of the diffusion
term of \oursde, i.e., $\frac{\norm{U(x)L(x)U(x)^\T}_{C^5}}{\eta}$ is upper bounded.
Let $u^n(x)$ defined the same as in \citet{feng2017semi}, i.e.,
$u^n(x_0) = \Ep{u(x_n)}$. According to
Theorem~2.1 of \citet{feng2017semi}, for small enough $\eta$, we have
boundedness of $\norm{u^n}_{C^6}$.

Starting both from $u^n(x)$, we first measure the error between the
discrete-time SGD and continuous-time iterates after one-step.
For the discrete-time SGD, by Taylor expansion, and boundedness of $\nabla f(x; \xi)$
and $\norm{u^n(x)}_{C^6}$, we have
\ifdefined\mpversion
\begin{align}
  \abs{u^{n+1}(x) - u^n(x) + \eta \inp{\nabla f(x)}{\nabla u^n(x)}
  - \frac{1}{2} \eta^2 \Ep{\nabla f(x; \xi) \nabla f(x; \xi)^\T}: \nabla^2 u^n(x)}
  \leq \cO\left(\eta^3\right). \label{eq:bbd}
\end{align}
\else
\begin{align}
  &\fakeeq \abs{u^{n+1}(x) - u^n(x) + \eta \inp{\nabla f(x)}{\nabla u^n(x)}
  - \frac{1}{2} \eta^2 \Ep{\nabla f(x; \xi) \nabla f(x; \xi)^\T}: \nabla^2 u^n(x)} \nonumber \\
  &= \abs{u^{n+1}(x) - u^n(x) + \eta \inp{\nabla f(x)}{\nabla u^n(x)}
  - \frac{1}{2} \eta^2 \left(\nabla f(x) \nabla f(x)^\T + \Sigma(x) \right): \nabla^2 u^n(x)} \nonumber \\
  &\leq \cO\left(\eta^3\right). \label{eq:bbd}
\end{align}
\fi

In continuous time, after time $\eta$, according to \Cref{lemma:continuous_expand}, we have
\begin{align} \label{eq:bba}
  \abs{e^{\eta \cL} - u^n(x) - \eta \cL u^n(x) - \frac{\eta^2}{2} \cL^2 u^n(x)} \leq 
  \cO(\eta^3).
\end{align}
Note that different from previous SDEs (\Cref{eq:SDE1,eq:SDE2}), now $\eta \cL u^n(x)$
and $\frac{\eta^2}{2} \cL^2 u^n(x)$ contains infinite many terms and we need to
consider their errors. We know that
\begin{align*}
  \eta \cL u^n(x) = \eta b(x) \cdot \nabla u^n(x)
  + \frac{\eta}{2} D(x)D(x)^\T : \nabla^2 u^n(x),
\end{align*}
where $b(x)$ is the drift term of \oursde{}.
By Taylor expansion of $\cL u^n$ w.r.t. $\eta$ and
\Cref{lemma:L_expand}, we have
\begin{align*}
  \bigg| 
  &\eta b(x) \cdot \nabla u^n(x) + \eta \nabla f(x) \cdot \nabla u^n (x)
  + \frac{\eta^2}{2} \nabla^2 f(x) \nabla f(x) \cdot \nabla u^n(x) \\
  &\fakeeq + \frac{\eta}{2} U(x)L(x)U(x)^\T : \nabla^2 u^n(x) - \frac{1}{2} \eta^2 \Sigma(x): \nabla^2 u^n(x)
\bigg|
  \leq \cO(\eta^3).
\end{align*}
Then we have
\begin{align} \label{eq:bbb}
  \abs{\eta \cL u^n(x) + \eta \nabla f(x) \cdot \nabla u^n (x)
  + \frac{\eta^2}{2} \nabla^2 f(x) \nabla f(x) \cdot \nabla u^n(x) - \frac{1}{2} \eta^2 \Sigma(x): \nabla^2 u^n(x)}
  \leq \cO\left(\eta^3\right).
\end{align}
Next, we look at the errors in $\frac{\eta^2}{2} \cL^2 u^n(x)$.
By Taylor expansion of $\cL^2 u^n$ w.r.t. $\eta$ and
\Cref{lemma:L_expand}, we have
\begin{align} \label{eq:bbc}
  \abs{\frac{\eta^2}{2} \cL^2 u^n(x) - \frac{\eta^2}{2}\nabla^2 f(x) \nabla f(x) \cdot \nabla u^n(x)
  - \frac{\eta^2}{2} \nabla f(x)\nabla f(x)^\T : \nabla^2 u^n(x)} \leq \cO\left(\eta^3\right).
\end{align}
Combining \Cref{eq:bba,eq:bbb,eq:bbc}, we get
\begin{align*}
  \abs{e^{\eta \cL} u^n(x) - u^n(x) + \eta \nabla f(x) \cdot \nabla u^n (x)
  - \frac{\eta^2}{2} \Sigma(x): \nabla^2 u^n(x)
  - \frac{\eta^2}{2} \nabla f(x)\nabla f(x)^\T : \nabla^2 u^n(x)} \leq \cO\left(\eta^3\right).
\end{align*}
Combining the above equation and \Cref{eq:bbd}, we have
\ifdefined\mpversion
$\abs{e^{\eta \cL} u^n(x) - u^{n+1}(x)} \leq \cO\left(\eta^3\right)$.
Denote $E^n = \norm{u^n(x) - u(x, n\eta)}_{L^{\infty}}$.
\else
\begin{align*}
  \abs{e^{\eta \cL} u^n(x) - u^{n+1}(x)} \leq \cO\left(\eta^3\right).
\end{align*}
Denote
\begin{align*}
  E^n = \norm{u^n(x) - u(x, n\eta)}_{L^{\infty}}.
\end{align*}
\fi
We obtain
\begin{align*}
  E^{n+1} 
  &= \norm{u^{n+1}(x) - u(x, (n+1)\eta) + e^{\eta \cL} u^n(x) - e^{\eta \cL} u^n(x)}_{L^{\infty}} \\
  &\leq \norm{e^{\eta \cL} \left(u^n(x) - u(x, n \eta)\right)}_{L^{\infty}}
  + \norm{u^{n+1}(x) - e^{\eta \cL} u^n(x)}_{L^{\infty}} \\
  &= \norm{e^{\eta \cL} \left(u^n(x) - u(x, n \eta)\right)}_{L^{\infty}}
  + \cO\left(\eta^3\right) \\
  &\leq \norm{u^n(x) - u(x, n \eta)}_{L^{\infty}}
  + \cO\left(\eta^3\right)
  = E^n + \cO\left(\eta^3\right),
\end{align*}
where the last inequality comes from the $L^{\infty}$ contraction of $e^{t\cL}$
(see Lemma~2.1 of \citet{feng2017semi}). Then we know that 
\ifdefined\mpversion 
$E^n \leq n\cO\left(\eta^3\right) = \frac{T}{\eta} \cO\left(\eta^3\right) \leq \cO\left(\eta^2\right)$.
\else
\begin{align*}
E^n \leq n\cO\left(\eta^3\right) = \frac{T}{\eta} \cO\left(\eta^3\right) \leq \cO\left(\eta^2\right).
\end{align*}
\fi
  
\end{proof}

\begin{proof}[Proof of \Cref{lemma:regul_drift_diff}]
  The lemma follows by \Cref{lemma:bound_drift,lemma:bound_diffu}
\end{proof}

\begin{proof}[Proof for \Cref{thm:weak_approx_lambda}]
The assumption of $\| F(x; \xi)\|_{c^8} < \infty$ implies $\| f(x)\|_{c^8} < \infty$.
According to \Cref{lemma:bound_drift}, we have boundedness of $\norm{b(x)}_{C^6}$,
where $b(x)$ is the drift term of \oursde.
According to \Cref{lemma:bound_drift}, we have boundedness of the diffusion
term of \oursde, i.e., $\frac{\norm{U(x)L(x)U(x)^\T}_{C^6}}{\eta}$ is upper bounded.
Denote $M_p^n \coloneqq \sum_{1 \leq \abs{J} \leq p} \abs{D^J u^n(x)}$.

Similar to the proof for \Cref{thm:weak_approx}, we first have
\begin{align*}
  \norm{u^{n+1} - u^n + \eta \nabla f \cdot \nabla u
  - \frac{1}{2} \eta^2 \Sigma: \nabla^2 u^n}_{L^{\infty}} \leq \cO\left(\eta^3 s^3 M^n_3 \right).
\end{align*}
First, we expand $e^{\eta \cL} u^n$ using Taylor's expansion and \Cref{lemma:continuous_expand},
\begin{align*}
  \norm{e^{\eta \cL} u^n - u^n - \eta \cL u^n - \frac{1}{2}\eta^2 \cL^2 u^n - \frac{1}{3!}\eta^3 \cL^3 u^n}_{L^{\infty}}
  \leq \cO\left(\eta^4 \lambda^3 s^4 M_8^n \right).
\end{align*}
Then using \Cref{lemma:L_expand}, we find the expansion of $\cL u^n$, $\cL^2 u^n$
and $\cL^3 u^n$, respectively.
\begin{align}
  \cL u^n 
  &= -\nabla f \cdot \nabla u^n + \eta\left(-\frac{1}{2}\nabla^2 f \nabla f \cdot \nabla u^n + \frac{1}{2} \Sigma : \nabla^2 u^n\right) \nonumber \\
  &\fakeeq + \frac{\eta^2}{2}\left(-\frac{2}{3}\left(\nabla^2 f\right)^2 \nabla f \cdot \nabla u^n  + \frac{1}{2}\left(\Sigma \nabla^2 f
  + \nabla^2 f \Sigma\right) : \nabla^2 u^n \right) +
  \cO\left(\eta^3 \lambda^3 s M_2^n \right). \label{eq:lambda_dep_1}
\end{align}
For $\cL^2 u^n$, we have
\begin{align*}
  \cL^2 u^n
  &= \nabla f(x)^\T \nabla^2 f \nabla u^n + \nabla f^\T \nabla^2 u^n \nabla f \\
  &\fakeeq + \eta \left( \nabla f^\T \left(\nabla^2 f\right)^2 \nabla u^n + \nabla f^\T \nabla^2 f \nabla^2 u^n \nabla f
    - \frac{1}{2} \Sigma: \left( \nabla^2 u^n \nabla^2 f + \nabla^2 f \nabla^2 u^n \right) + \cO(M_3^n)\right) \\
  &\fakeeq + \cO\left(\eta^2 \lambda^3 s^2 M_4^n \right).
\end{align*}
For $\cL^3 u^n$, we have
\begin{align*}
  \cL^3 u^n
  &= - \nabla f^\T \left(\nabla^2 f \right)^2 \nabla u^n - 3 \nabla f^\T \nabla^2 f \left(\nabla u^n\right)^2 
  + \cO\left(M_3^n\right) + \cO\left(\eta \lambda^3 s^3 M_6^n \right).
\end{align*}

Summarizing the above, we obtain
\begin{align*}
  \norm{e^{\eta \cL} u^n - u^{n+1}}_{L^{\infty}} \leq \cO\left(\eta^3 s^3 M^n_3 + \eta^4 \lambda^3 s^4 M_8^n \right).
\end{align*}
Following the last few steps in the proof of \Cref{thm:weak_approx}, we obtain
\begin{align*}
  E^n \leq \sum_{k=0}^{n-1} \cO\left(\eta^3 s^3 M^k_3 + \eta^4 \lambda^3 s^4 M_8^k \right)
\end{align*}
finish the proof.

\end{proof}

\begin{proof}[Proof for \Cref{thm:weak_approx_lambda_other_SDEs}]
  The proof follows similarly to the proof of \Cref{thm:weak_approx_lambda}.
  For \sme-2, in the expansion of $\cL u^n$, i.e., \Cref{eq:lambda_dep_1},
  the term
  \begin{align*}
    \frac{\eta^2}{2}\left(-\frac{2}{3}\left(\nabla^2 f\right)^2 \nabla f \cdot \nabla u^n  + \frac{1}{2}\left(\Sigma \nabla^2 f
  + \nabla^2 f \Sigma\right) : \nabla^2 u^n \right)
  \end{align*}
  is missing. Therefore, after cancellation of terms, there will be error terms of
  $\lambda^2 s \norm{\nabla u^n}$ and $\lambda \norm{\nabla^2 u^n}$.
  For \pfs{}, in \Cref{eq:lambda_dep_1}, the term
  \begin{align*}
    \frac{\eta^2}{2}\left(\frac{1}{2}\left(\Sigma \nabla^2 f
  + \nabla^2 f \Sigma\right) : \nabla^2 u^n \right)
  \end{align*}
  is missing. Therefore, it results in additional error of $\lambda \norm{\nabla^2 u^n}$.
\end{proof}

\begin{proof}[Proof of \Cref{lemma:strong_sc}]
  We will show the results for both the convex and strongly-convex 
  (corresponding to our \Cref{remark:sc}) settings.
  According to \citet{feng2017semi}, for any $n \geq 0$, $u^{n+1}(x)$ can be written as
  \begin{align*}
    u^{n+1} (x) = \Ep{u^n\left(x - \eta \nabla F(x; \xi)\right)}.
  \end{align*}
  Let us check the first-order derivatives of $u^{n+1}$. Denoting $y =
  x-\eta \nabla F(x; \xi)$, we obtain
  \begin{align*}
    \nabla u^{n+1} (x) = \Ep{\left(I - \eta \nabla^2 F(x; \xi) \right)\nabla u^{n} (y) }.
  \end{align*}
  \begin{enumerate}
    \item If $f(\cdot)$ is strongly-convex, we obtain
  \begin{align*}
    \sup_{x} \norm{\nabla u^{n+1} (x)} \leq (1-\eta \mu) \sup_x \norm{\nabla u^{n} (x)}.
  \end{align*}
   With small $\eta$, we have $\eta \sup_{x, \xi} \norm{\nabla^2 F(x; \xi)} < 1$, and the recursion form is
  a contraction:
  \begin{align*}
    \sup_x \norm{\nabla u^{n} (x)} 
    &\leq (1-\mu \eta)^{n} \sup_x \norm{\nabla u (x)}
    \leq e^{-\mu \eta n} \sup_x \norm{\nabla u (x)}
    \leq \cO\left(e^{-\mu \eta n} \norm{u}_{C^1}\right).
  \end{align*}
\item If $f(\cdot)$ is convex, we have
  \begin{align*}
    \sup_{x} \norm{\nabla u^{n+1} (x)} \leq \sup_x \norm{\nabla u^{n} (x)} \leq \cO\left(\norm{u}_{C^1}\right).
  \end{align*}
  \end{enumerate}
  The notation $\cO\left(\cdot\right)$ does not depend on $n$, $\eta$ or upper bounds for derivatives of $F$ in $\cO(\cdot)$ with orders higher than the second order.

  Next, we consider the second-order gradients:
  \begin{align*}
    \nabla^2 u^{n+1} (x) = \Ep{\left(I - \eta \nabla^2 F(x; \xi) \right)\nabla^2 u^{n} (y) \left(I - \eta \nabla^2 F(x; \xi) \right)
    - \eta \nabla u^n(y)^\T \nabla^3 f(x) }.
  \end{align*}
  Then we also obtain a recursion for the 2-norm of its vector form:
\ifdefined\mpversion
  \begin{align}
    &\fakeeq \underbrace{\sup_x \norm{\vectt{\nabla^2 u^{n+1} (x)}}_2}_{a_{n+1}} \\
    &=\sup_x \norm{\nabla^2 u^{n+1} (x)}_F \\
    &\leq \sup_x \norm{\left(I - \eta \nabla^2 F(x; \xi) \right)\nabla^2 u^{n} (y) \left(I - \eta \nabla^2 F(x; \xi) \right)}_F + \eta \sup_x \norm{\nabla u^n(y)^\T \nabla^3 f(x)}_F \nonumber \\
    &= \sup_x \norm{\left(I - \eta \nabla^2 F(x; \xi) \right)^{\otimes 2} \vectt{\nabla^2 u^{n} (y)} }_2 + \eta \sup_x \norm{\nabla u^n(y)^\T \nabla^3 f(x)}_F \nonumber \\
    &\leq \sup_x \norm{\left(I - \eta \nabla^2 F(x; \xi) \right)^{\otimes 2}}_2 \norm{\vectt{\nabla^2 u^{n} (y)} }_2 + \eta \sup_x \norm{\nabla u^n(y)^\T \nabla^3 f(x)}_F \nonumber \\
    &\leq \sup_x \norm{I - \eta \nabla^2 F(x; \xi)}_2^{2} \norm{ \vectt{\nabla^2 u^{n} (y)} }_2 + \eta \sup_x \norm{\nabla u^n(y)^\T \nabla^3 f(x)}_F \nonumber,
  \end{align}
\else
  \begin{align}
    &\fakeeq \underbrace{\sup_x \norm{\vectt{\nabla^2 u^{n+1} (x)}}_2}_{a_{n+1}} \\
    &=\sup_x \norm{\nabla^2 u^{n+1} (x)}_F \\
    &\leq \sup_x \norm{\left(I - \eta \nabla^2 F(x; \xi) \right)\nabla^2 u^{n} (y) \left(I - \eta \nabla^2 F(x; \xi) \right)}_F + \eta \sup_x \norm{\nabla u^n(y)^\T \nabla^3 f(x)}_F \nonumber \\
    &= \sup_x \norm{\vectt{\left(I - \eta \nabla^2 F(x; \xi) \right)\nabla^2 u^{n} (y) \left(I - \eta \nabla^2 F(x; \xi) \right)}}_2 + \eta \sup_x \norm{\nabla u^n(y)^\T \nabla^3 f(x)}_F \nonumber \\
    &= \sup_x \norm{\left(I - \eta \nabla^2 F(x; \xi) \right)^{\otimes 2} \vectt{\nabla^2 u^{n} (y)} }_2 + \eta \sup_x \norm{\nabla u^n(y)^\T \nabla^3 f(x)}_F \nonumber \\
    &\leq \sup_x \norm{\left(I - \eta \nabla^2 F(x; \xi) \right)^{\otimes 2}}_2 \norm{\vectt{\nabla^2 u^{n} (y)} }_2 + \eta \sup_x \norm{\nabla u^n(y)^\T \nabla^3 f(x)}_F \nonumber \\
    &\leq \sup_x \norm{I - \eta \nabla^2 F(x; \xi)}_2^{2} \norm{ \vectt{\nabla^2 u^{n} (y)} }_2 + \eta \sup_x \norm{\nabla u^n(y)^\T \nabla^3 f(x)}_F \nonumber,
  \end{align}
\fi
  where $\otimes$ is the Kronecker product, and we used that $\norm{A\otimes B}_2 \leq \norm{A}_2 \norm{B}_2$
  for matrices $A$ and $B$.
  \begin{enumerate}
    \item If $f(\cdot)$ is strongly-convex, the recursion becomes
      \begin{align}
    \underbrace{\sup_x \norm{\vectt{\nabla^2 u^{n+1} (x)}}_2}_{a_{n+1}}
    \leq \underbrace{(1-\eta \mu)^2}_{c} \underbrace{\sup_x \norm{\vectt{\nabla^2 u^{n} (x)}}_2}_{a_n} + \underbrace{\eta \sup_x \norm{\nabla u^n(y)^\T \nabla^3 f(x)}_F}_{b_n}. \label{eq:recur}
      \end{align}
  According to our
  assumption and previous results, $b_n \leq \cO(\eta e^{-\mu \eta n} \norm{u}_{C^1})$.
  Also we choose small $\eta$ such that $1-\eta \mu > 0$.
  It holds for the recursion that
  \begin{align*}
    a_{n+1} \leq c^{n+1} a_0 + \sum_{s=0}^{n} c^{n-s} b_s.
  \end{align*}
  In our case, it becomes
\ifdefined\mpversion
  \begin{align}
    \sup_x \norm{\vectt{\nabla^2 u^{n} (x)}}_2 
  &\leq (1-\eta \mu)^{2n} \sup_x \norm{\vectt{\nabla^2 u (x)}}_2 + \sum_{s=0}^n (1-\eta \mu)^{2(n-s)} \cO(\eta e^{-\mu \eta s} \norm{u}_{C^1}) \nonumber \\
  &\leq e^{-\mu \eta n} \sup_x \norm{\vectt{\nabla^2 u (x)}}_2 + \sum_{s=0}^n e^{\mu \eta s} \cO(\eta e^{-2 \mu \eta n} \norm{u}_{C^1}) \nonumber \\
  &\leq e^{-\mu \eta n} \sup_x \norm{\vectt{\nabla^2 u (x)}}_2 + \frac{1}{\eta \mu}\cO(\eta e^{- \mu \eta n} \norm{u}_{C^1}) \nonumber \\
    &\leq \cO(e^{-\mu \eta n} \norm{u}_{C^2}), \label{eq:recur_2}
  \end{align}
\else
  \begin{align}
    \sup_x \norm{\vectt{\nabla^2 u^{n} (x)}}_2 
  &\leq (1-\eta \mu)^{2n} \sup_x \norm{\vectt{\nabla^2 u (x)}}_2 + \sum_{s=0}^n (1-\eta \mu)^{2(n-s)} \cO(\eta e^{-\mu \eta s} \norm{u}_{C^1}) \nonumber \\
  &\leq (1-\eta \mu)^{n} \sup_x \norm{\vectt{\nabla^2 u (x)}}_2 + \sum_{s=0}^n e^{-\mu \eta 2(n-s)} \cO(\eta e^{-\mu \eta s} \norm{u}_{C^1}) \nonumber \\
  &\leq e^{-\mu \eta n} \sup_x \norm{\vectt{\nabla^2 u (x)}}_2 + \sum_{s=0}^n e^{\mu \eta s} \cO(\eta e^{-2 \mu \eta n} \norm{u}_{C^1}) \nonumber \\
  &= e^{-\mu \eta n} \sup_x \norm{\vectt{\nabla^2 u (x)}}_2 + \frac{e^{\eta \mu(n+1)}-1}{e^{\eta \mu} - 1}\cO(\eta e^{-2 \mu \eta n} \norm{u}_{C^1}) \nonumber \\
  &\leq e^{-\mu \eta n} \sup_x \norm{\vectt{\nabla^2 u (x)}}_2 + \frac{e^{\eta \mu(n+1)}}{e^{\eta \mu} - 1}\cO(\eta e^{-2 \mu \eta n} \norm{u}_{C^1}) \nonumber \\
  &\leq e^{-\mu \eta n} \sup_x \norm{\vectt{\nabla^2 u (x)}}_2 + \frac{e^{\eta \mu n}}{e^{\eta \mu} - 1}\cO(\eta e^{-2 \mu \eta n} \norm{u}_{C^1}) \nonumber \\
  &= e^{-\mu \eta n} \sup_x \norm{\vectt{\nabla^2 u (x)}}_2 + \frac{1}{e^{\eta \mu} - 1}\cO(\eta e^{- \mu \eta n} \norm{u}_{C^1}) \nonumber \\
  &\leq e^{-\mu \eta n} \sup_x \norm{\vectt{\nabla^2 u (x)}}_2 + \frac{1}{\eta \mu}\cO(\eta e^{- \mu \eta n} \norm{u}_{C^1}) \nonumber \\
  &= e^{-\mu \eta n} \sup_x \norm{\vectt{\nabla^2 u (x)}}_2 + \cO(e^{- \mu \eta n} \norm{u}_{C^1}) \nonumber \\
    &\leq \cO(e^{-\mu \eta n} \norm{u}_{C^2}), \label{eq:recur_2}
  \end{align}
\fi
  where we used that $e^x \geq x + 1$.
\item If $f(\cdot)$ is convex, we have
      \begin{align*}
    \sup_x \norm{\vectt{\nabla^2 u^{n+1} (x)}}_2
    \leq \sup_x \norm{\vectt{\nabla^2 u^{n} (x)}}_2 + \cO(\eta).
      \end{align*}
      Therefore, we obtain
      \begin{align*} \sup_x \norm{\vectt{\nabla^2 u^{n} (x)}}_2 \leq \cO\left(T \norm{u}_{C^2}\right).
      \end{align*}
  \end{enumerate}

  Next, we proceed with induction.
  \begin{enumerate}
    \item If $f(\cdot)$ is strongly-convex,
  assume it holds that for any $0 < s \leq p$,
  we have $\sup_x \norm{\vectt{\nabla^{s} u^{n}(x)}}_F \leq \cO(e^{-\mu \eta n} \norm{u}_{C^s})$.
  Then, we study the $p+1$-th order gradient of $u^{n+1}$:
  \begin{align*}
    \norm{\vectt{\nabla^{p+1} u^{n+1}(x)}}_2 
    &= \norm{\left(I - \eta \nabla^2 F(x; \xi)\right)^{\otimes p+1} \vectt{\nabla^{p+1} u^n(y)}}_2 + \eta \cO(e^{-\mu \eta n } \norm{u}_{C^p}).
  \end{align*}
  The first term is the result by applying the $\nabla$ operator $p+1$ times on $u$.
  For other terms, at least one $\nabla$ is applied on $\left(I - \eta \nabla^2 F(x; \xi)\right)$,
  which results in a $\eta$ and gradients of $F$ higher than the second-order
  (which is hidden in $\cO(\cdot)$).
  Also, the gradients of $u$ are at most the $p$-th order, which by induction are
  already upper bounded. Therefore, similarly to \Cref{eq:recur}, we obtain
  \begin{align*}
    &\fakeeq \sup_x \norm{\vectt{\nabla^{p+1} u^{n+1}(x)}}_2 \\
    &\leq \sup_x \norm{\left(I - \eta \nabla^2 F(x; \xi)\right)^{\otimes p+1} \vectt{\nabla^{p+1} u^n(y)}}_2 + \eta \cO(e^{-\mu \eta n } \norm{u}_{c^p}) \\
    &\leq \sup_x \norm{\left(I - \eta \nabla^2 F(x; \xi)\right)}_2^{p+1} \sup_x \norm{\vectt{\nabla^{p+1} u^n(x)}}_2 + \eta \cO(e^{-\mu \eta n } \norm{u}_{c^p}) \\
    &\leq (1-\eta \mu)^{p+1} \sup_x \norm{\vectt{\nabla^{p+1} u^n(x)}}_2 + \eta \cO(e^{-\mu \eta n } \norm{u}_{c^p}),
  \end{align*}
  which can be then recursively bounded similarly to \Cref{eq:recur_2}, i.e.,
  \begin{align*}
    &\fakeeq \sup_x \norm{\vectt{\nabla^{p+1} u^{n}(x)}}_2 \\
    &\leq e^{-\mu \eta n} \sup_x \norm{\vectt{\nabla^{p+1} u(x)}}_2 + \cO(e^{-\mu \eta n}  \norm{u}_{c^p}) \\
    &\leq \cO(e^{-\mu \eta n} \norm{u}_{c^{p+1}}).
  \end{align*}
  Now we know that 
  \begin{align*}
    \sum_{1 \leq \abs{J} \leq p} \abs{D^J u^n(x)} \leq \cO\left(e^{-\mu \eta n} \norm{u}_{C^p}\right).
  \end{align*}
  It follows
  \begin{align*}
    &\fakeeq \eta \sum_{k=0}^{\floor{T/\eta}-1} \sum_{1 \leq \abs{J} \leq p} \abs{D^J u^k(x)}
    \leq \eta \sum_{k=0}^{\floor{T/\eta}-1} \cO\left(e^{-\mu \eta k} \norm{u}_{C^p}\right) \\
    &\leq \eta \frac{1-e^{-\eta \mu \floor{T/\eta} }}{1-e^{-\mu \eta}}\cO\left(\norm{u}_{C^p}\right)
    \leq \frac{\eta}{1-e^{-\mu \eta}}\cO\left(\norm{u}_{C^p}\right)
    \leq \cO\left(\norm{u}_{C^p}\right).
  \end{align*}
\item If $f(\cdot)$ is convex, we assume for any $0 < s \leq p$,
$\sup_x \norm{\vectt{\nabla^{s} u^{n}(x)}}_F \leq \cO(T^{s-1} \norm{u}_{C^s})$.
Similarly to the strongly-convex case, we obtain
  \begin{align*}
    \norm{\vectt{\nabla^{p+1} u^{n+1}(x)}}_2 
    &= \norm{\left(I - \eta \nabla^2 F(x; \xi)\right)^{\otimes p+1} \vectt{\nabla^{p+1} u^n(y)}}_2 + \eta \cO(T^{p-1} \norm{u}_{C^p}),
  \end{align*}
  which implies 
  \begin{align*}
    \sup_x \norm{\vectt{\nabla^{p+1} u^{n}(x)}}_2 
    \leq \cO(T^{p} \norm{u}_{C^{p+1}}).
  \end{align*}
  Further it holds that
  \begin{align*}
    \eta \sum_{k=0}^{\floor{T/\eta}-1} \sum_{1 \leq \abs{J} \leq p} \abs{D^J u^k(x)}
    \leq \cO\left(T^{p+1} \norm{u}_{C^{p}}\right).
  \end{align*}
  \end{enumerate}
\end{proof}

\begin{proof}[Proof for \Cref{thm:weak_approx_sc}]
  The proof is simply to combine the results from \Cref{thm:weak_approx_lambda}
  and \Cref{lemma:strong_sc}.
\end{proof}

\subsection{Helper Lemmas}

\begin{lemma} \label{lemma:bound_drift}
    For $n \geq 0$, assume $f \in C^{n+2}_b(\mathbb{R}^d)$ and denote $\lambda \coloneqq \sup_{x \in \mathbb{R}^d}\norm{\nabla^2 f(x)}$
    and $s \coloneqq \sup_{x \in \mathbb{R}^d}\norm{\nabla f(x)}$.
    There exist constants $\eta_0$ and $C > 0$, both independent of $\lambda$ and $s$, such that for any $\eta$ satisfying $\eta < \eta_0$ and $\eta \lambda < C$,
  the drift 
  term of \oursde, defined in in \Cref{eq:HASME-drifiting-term},
  satisfies
  \begin{equation*}
    \max_{0 \leq i \leq d} \abs{[b(x)]_i} < \cO\left(s\right) \text{ and } \max_{0 \leq i \leq d} \norm{[b(x)]_i}_{C^n} < \cO\left(s + \lambda\right),
  \end{equation*}
  where
  $[b(x)]_i$ is the $i$-th entry of $b(x)$, and $\cO(\cdot)$ hides the dependence
  on the upper bounds for the derivatives of $f$ higher than the second-order derivatives.
\end{lemma}
\begin{proof}
    Recall the definition of the drifting term of \oursde\ in \Cref{eq:HASME-drifiting-term} 
    \begin{align*}
        b(x) = U(x) \frac{\log\left(I - \eta \Lambda(x)\right)}{\eta \Lambda(x)} U(x)^\T \nabla f(x).
    \end{align*}
  We first note that $b(x)$ can be written as 
  \begin{align*}
    b(x) = - \sum_{p=0}^{\infty} \frac{1}{p + 1} \eta^p \left(\nabla^2 f(x)\right)^p \nabla f(x).
  \end{align*}
  The $(i, j)$-th element of $A^p$ for matrix $A \in \mathbb{R}^{d \times d}$
  can be represented as
  \begin{align*}
    \left[A^p\right]_{i, j} = \sum_{s_1=1}^{d} [A]_{i, s_1} \sum_{s_2=1}^{d} [A]_{s_1, s_2}
    \sum_{s_3=1}^{d} [A]_{s_2, s_3} \cdots \sum_{s_{p-1}} [A]_{s_{p-2}, s_{p-1}} [A]_{s_{p-1}, j}.
  \end{align*}
  Therefore, for any $i$, the $i$-th element of $b(x)$ can be written as
\ifdefined\mpversion
  \begin{align*}
    [b(x)]_i 
    &= - \sum_{p=0}^{\infty} \frac{1}{p + 1} \eta^p \underbrace{ \sum_{s_1=1}^{d} \partial_i \partial_{s_1} f(x)
    \sum_{s_2=1}^{d} \partial_{s_1} \partial_{s_2} f(x) \cdots 
  \sum_{s_p=1}^{d} \partial_{s_{p-1}}\partial_{s_p} f(x) \partial_{s_p} f(x)}_{\mlabelterm{term:a}}.
  \end{align*}
\else
  \begin{align*}
    [b(x)]_i 
    &= - \sum_{p=0}^{\infty} \frac{1}{p + 1} \eta^p \left[\left(\nabla^2 f(x)\right)^p \nabla f(x)\right]_i \\
    &= - \sum_{p=0}^{\infty} \frac{1}{p + 1} \eta^p \sum_{s_1=1}^{d} \left[\nabla^2 f(x)\right]_{i, s_1}
      \sum_{s_2=1}^{d} \left[\nabla^2 f(x)\right]_{s_1, s_2} \cdots 
      \sum_{s_p=1}^{d} \left[\nabla^2 f(x)\right]_{s_{p-1}, s_p} \left[\nabla f(x)\right]_{s_p}  \\
    &= - \sum_{p=0}^{\infty} \frac{1}{p + 1} \eta^p \underbrace{ \sum_{s_1=1}^{d} \partial_i \partial_{s_1} f(x)
    \sum_{s_2=1}^{d} \partial_{s_1} \partial_{s_2} f(x) \cdots 
  \sum_{s_p=1}^{d} \partial_{s_{p-1}}\partial_{s_p} f(x) \partial_{s_p} f(x)}_{\mlabelterm{term:a}}.
  \end{align*}
\fi
  Notice that \Refterm{term:a} can be expanded as a summation of $d^p$ terms,
  each of which contains $p+1$ factors. Note that the first-order and second-order partial derivatives
  are upper bounded by $s$ and $\lambda$ respectively, i.e.,
  $|D^\alpha f| \leq s$ for $|a|=1$ and $|D^\alpha f| \leq \lambda$ for $|a|=2$,
  and we have
  \begin{align*}
    \abs{[b(x)]_i}
    \leq \sum_{p=0}^{\infty} \frac{1}{p+1} \eta^p d^p \lambda^p s
    = s \sum_{p=0}^{\infty} \frac{1}{p+1} \left(\eta d \lambda\right)^p.
  \end{align*}
  We get a power series in the above equation, and according to Cauchy–Hadamard 
  theorem~\citep{cauchy1821analyse}, the convergence radius for 
  $g(x) = \sum_{p=0}^{\infty} \frac{1}{p+1} x^p$ is
  \begin{align*}
    \frac{1}{\limsup_{p \to \infty} \abs{\frac{1}{p+1}}^{\frac{1}{p}}} = \frac{1}{1} = 1.
  \end{align*}
  Then we know that $\abs{[b(x)]_i}$ is upper bounded by $\cO(s)$ as long as 
  \begin{align*}
    \eta d \lambda < 1 \quad \Longrightarrow \quad \eta \lambda < \frac{1}{d}.
  \end{align*}

  Next, we consider the first-order derivative of $b(x)$.
  We denote the upper bound for higher order partial derivatives as $B$, i.e.,
  $\max_{2 < |a| \leq n+2} |D^\alpha f| \leq B$.
  Recall that in $b(x)$, \Refterm{term:a} contains $d^p$ terms and each term
  has $p$ factors of second-order derivatives of $f(x)$ and one factor of the
  first-order derivative of $f(x)$.
  If we take gradient of \Refterm{term:a} w.r.t. $x$, each term, according to
  the product rule of gradients and by taking derivatives w.r.t. each factor,
  will result in: (1) $p$ terms, each consisting of one factors of $\partial^3 f$,
  $p-1$ factor of $\partial^2 f$, and one factor of $\partial f$; (2) 1 term
  consisting of $p+1$ factors of $\partial^2 f$.
  Therefore, we can bound
  \begin{align*}
    \abs{\partial_j [b(x)]_i}
    &\leq \sum_{p=0}^{\infty} \frac{1}{p+1} \eta^p d^p \left(p B s \lambda^{p-1} + \lambda^{p+1} \right) \\
    &= s\sum_{p=0}^{\infty} \frac{p}{p+1} d^p \eta B (\eta \lambda)^{p-1} 
    + \lambda \sum_{p=0}^{\infty} \frac{1}{p+1} (d \eta \lambda)^{p}  \\
  \end{align*}
  For the second summation, the convergence radius of the power series is the
  same as the upper bound for $[b(x)]_i$, and it is bounded
  by $\cO(\lambda)$ for sufficiently small $\eta$ and $\eta \lambda$.
  For the first summation, comparing to the upper bound of $[b(x)]_i$, the
  coefficients of the power series are multiplied by $p$. However, this will
  not change the convergence radius. It is known that for sequences $a_p$ and
  $b_p$, if $\limsup_{p \to \infty} a_p = A$ and $\lim_{p \to \infty} b_p = B$,
  then $\limsup_{p \to \infty} \left(a_p b_p\right) = AB$. Note that
 \begin{align*}
   \lim_{p \to \infty} \abs{p}^{\frac{1}{p}} = 1.
 \end{align*}
 Therefore, by applying Cauchy–Hadamard theorem, multiplying the coefficients
 with $p$ would not change the radius of convergence. Therefore,
 if we have $\eta \lambda$ and $\eta B$ smaller than the
 convergence radius of $\sum_p \frac{1}{1+p} x^p$, that is 1, the first summation
 in the upper bound of $|\partial_j[b(x)]_i|$ can be upper bounded by some constants.
 As a result, we can conclude that
 \begin{align*}
   \abs{\partial_j [b(x)]_i} \leq \cO(s + \lambda).
 \end{align*}

 Following the same idea, we identify that after applying twice partial derivatives
 on $b(x)$, i.e., $\partial_j \partial_k [b(x)]_i$, each term in \Refterm{term:a}
 becomes: 
\begin{align*}
 p\partial^4 f \left(\partial^2 f\right)^{p-1} \partial f + 
p(p-1) \left(\partial^3 f\right)^2 \left(\partial^2 f\right)^{p-2} \partial f +
(2p+1) \partial^3 f \left(\partial^2 f\right)^p
\end{align*}
Thus, we have
\begin{align*}
  |\partial_j \partial_k [b(x)]_i|
    &\leq \sum_{p=0}^{\infty} \frac{1}{p+1} \eta^p d^p \left(p(p-1) B^2 s \lambda^{p-2}
    + p B s \lambda^{p-1} + (2p+1) B \lambda^{p} \right) \\
    &= s \sum_{p=0}^{\infty} \frac{p(p-1)}{p+1} d^p (B \eta)^2 (\lambda \eta)^{p-2}
    + s \sum_{p=0}^{\infty} \frac{p}{p+1} d^p B \eta (\lambda \eta)^{p-1} +
    B \sum_{p=0}^{\infty} \frac{2p+1}{p+1} d^p (\lambda \eta)^{p}.
\end{align*}
By the same reason as what we did for the $\partial_j [b(x)]_i$, the three
summations converge, and are upper bounded by some constants not relying on $\lambda$.

For derivatives higher than the second ones, we can follow the same logic to
show that as long as $\eta B$ and $\eta \lambda$ is small enough, the
derivatives are upper bounded and such bound does not rely on $\lambda$.
Specifically, for the $k$-th order derivative of $b(x)$ with $k \geq 2$,
we will have terms of the following form derived from each term in
\Refterm{term:a}:
\begin{align*}
  \cO\left(p^k\right) \left(\prod_{j=1}^{m} \partial^{q_j} f\right) \left(\partial^2 f\right)^s \left(\partial f\right)^r,
\end{align*}
with $m+s+r = p+1$ and $q_j \geq 3$, $s \leq p$ and $r \leq 1$. As we only 
consider derivatives of $b(x)$ up to some finite order. We have finite such terms
and they can be bounded using the idea presented previously.

\end{proof}

\begin{lemma} \label{lemma:bound_diffu}
  For $n \geq 0$, assume $f \in C^{n+2}_b\left(\mathbb{R}^d\right)$ and $[\Sigma]_{i, j} \in C^{n}_b\left(\mathbb{R}^d\right)$ for all $1 \leq i, j \leq d$ 
  where $\left[A\right]_{i, j}$ is the $(i, j)$-th entry of matrix $A$, then there exists $\eta_0, C > 0$,
  such that for any $\eta < \eta_0$ and $\eta \lambda < C$, where 
  $\lambda \coloneqq \sup_{x \in \mathbb{R}^d}\norm{\nabla^2 f(x)}$ and $C, \eta_0$
  does not depend on $\lambda$, we have that the diffusion term $D(x)$ of 
  stochastic principal flow satisfy 
  \[
    \max_{0 \leq i, j \leq d}\frac{\norm{\left[D(x)D(x)^\T\right]_{i, j}}_{C^n}}{\eta} < \cO\left(1\right),
  \]
where $\cO(\cdot)$ hides the dependence
  on the upper bounds for the derivatives of $f$ higher than the second-order derivatives
  and $\max_{0 \leq i, j \leq d} \norm{[\Sigma]_{i,j}}_{C^n}$.
\end{lemma}
\begin{proof}
  The proof idea is similar to the proof of \Cref{lemma:bound_drift}, however,
  in this case, we do not have an explicit form for the coefficients of the power series.
  The diffusion term can be represented as
  \begin{align*}
    D(x)D(x)^\T = \sum_{p=1}^{\infty} \eta^p \sum_{k = 0}^{p-1} a_{k, p-1-k} \cdot \left(\nabla^2 f(x)\right)^{k} \Sigma(x)
  \left(\nabla^2 f(x)\right)^{p-1-k},
  \end{align*}
  where $a_{k, p-1-k}$ are absolute constants such that by Taylor expansion at $(0, 0)$,
  \begin{align*}
    g(x, y) \coloneqq \frac{\log(1-x)(1-y)}{x y - (x + y)} = \sum^{+\infty}_{s, m \geq 0} a_{s, m} x^s y^m .
  \end{align*}
  For the power series above, we know that its convergence radius is greater or
  equal than $1$ for both $x$ and $y$, meaning that $x < 1$ and $y < 1$ is a
  sufficient condition for it to converge. Also, according to the Cauchy–Hadamard
  theorem for multiple variables~\citep{shabat1992introduction}, the signs of the coefficients do affect
  the convergence radius, therefore the convergence radius of
  \begin{align*}
    \widetilde{g}(x, y) \coloneqq \sum^{+\infty}_{s, m \geq 0} \abs{a_{s, m}} x^s y^m
  \end{align*}
  is also greater or equal than $1$.
  Now we look back into $D(x)D(x)^\T$. We note that the $(i, j)$-th element of
  $A^pBA^k$ for $A, B \in \mathbb{R}^{d \times d}$ can be written as
  \begin{align*}
    \left[A^pBA^k\right]_{i, j}
    &= \sum_{s_1=1}^d \left[A\right]_{i, s_1}
    \sum_{s_2=1}^d \left[A\right]_{s_1, s_2}  \cdots
    \sum_{s_p=1}^d \left[A\right]_{s_{p-1}, s_p}
    \sum_{m=1}^d \left[S\right]_{s_p, m}
    \sum_{q_1=1}^d \left[A\right]_{m, q_1}
    \sum_{q_2=1}^d \left[A\right]_{q_1, q_2} \cdots \\
    &\fakeeq \sum_{q_{k-1}=1}^d \left[A\right]_{q_{k-2}, q_{k-1}} \left[A\right]_{q_{k-1}, j}.
  \end{align*}
  Therefore, we can derive that
\ifdefined\mpversion
  \begin{align*}
    &\fakeeq \left[D(x)D(x)^\T\right]_{i,j} \\
    &= \sum_{p=1}^{\infty} \eta^p \sum_{k = 0}^{p-1} a_{k, p-1-k} 
    \sum_{s_1=1}^d \partial_{i} \partial_{s_1} f(x)
    \sum_{s_2=1}^d \partial_{s_1} \partial_{s_2} f(x)  \cdots
    \sum_{s_k=1}^d \partial_{s_{k-1}} \partial_{s_k} f(x)
    \sum_{m=1}^d \left[\Sigma(x)\right]_{s_k, m} \\
    &\fakeeq \sum_{q_1=1}^d \partial_{m} \partial_{q_1} f(x)
    \sum_{q_2=1}^d \partial_{q_1} \partial_{q_2} f(x) \cdots
    \sum_{q_{p-k-2}=1}^d \partial_{q_{p-k-3}} \partial_{q_{p-k-2}} f(x)
    \partial_{q_{p-k-2}} \partial_{j} f(x).
  \end{align*}
\else
  \begin{align*}
    &\fakeeq \left[D(x)D(x)^\T\right]_{i,j} \\
    &= \sum_{p=1}^{\infty} \eta^p \sum_{k = 0}^{p-1} a_{k, p-1-k} 
    \sum_{s_1=1}^d \left[\nabla^2 f(x)\right]_{i, s_1}
    \sum_{s_2=1}^d \left[\nabla^2 f(x)\right]_{s_1, s_2}  \cdots
    \sum_{s_k=1}^d \left[\nabla^2 f(x)\right]_{s_{k-1}, s_k}
    \sum_{m=1}^d \left[\Sigma(x)\right]_{s_k, m} \\
    &\fakeeq \sum_{q_1=1}^d \left[\nabla^2 f(x)\right]_{m, q_1}
    \sum_{q_2=1}^d \left[\nabla^2 f(x)\right]_{q_1, q_2} \cdots
    \sum_{q_{p-k-2}=1}^d \left[\nabla^2 f(x)\right]_{q_{p-k-3}, q_{p-k-2}} \left[\nabla^2 f(x)\right]_{q_{p-k-2}, j} \\
    &= \sum_{p=1}^{\infty} \eta^p \sum_{k = 0}^{p-1} a_{k, p-1-k} 
    \sum_{s_1=1}^d \partial_{i} \partial_{s_1} f(x)
    \sum_{s_2=1}^d \partial_{s_1} \partial_{s_2} f(x)  \cdots
    \sum_{s_k=1}^d \partial_{s_{k-1}} \partial_{s_k} f(x)
    \sum_{m=1}^d \left[\Sigma(x)\right]_{s_k, m} \\
    &\fakeeq \sum_{q_1=1}^d \partial_{m} \partial_{q_1} f(x)
    \sum_{q_2=1}^d \partial_{q_1} \partial_{q_2} f(x) \cdots
    \sum_{q_{p-k-2}=1}^d \partial_{q_{p-k-3}} \partial_{q_{p-k-2}} f(x)
    \partial_{q_{p-k-2}} \partial_{j} f(x).
  \end{align*}
\fi
  Let us look at the term associated with $a_{k, p-1-k}$, i.e., everything after
  $a_{k, p-1-k}$ in the above equation. It contains $p-1$ summations, thus
  resulting in $d^{p-1}$ terms. Each term has $p$ factors consisting of one element
  of $\Sigma(x)$ and others being second-order derivatives of $f(x)$.
  Denoting $S \coloneqq \max_{i, j} \norm{[\Sigma(x)]_{i, j}}_{C^6}$, we can derive
\ifdefined\mpversion
  \begin{align*}
    \abs{\left[D(x)D(x)^\T\right]_{i,j}}
    &\leq \sum_{p=1}^{\infty} \eta^p \sum_{k = 0}^{p-1} \abs{a_{k, p-1-k}}
    d^{p-1} S \lambda^{p-1}
    = \eta S \sum^{\infty}_{s, m \geq 0} \abs{a_{s, m}}
    \left(\eta d \lambda \right)^{s} \left(\eta d \lambda \right)^{m}.
  \end{align*}
\else
  \begin{align*}
    \abs{\left[D(x)D(x)^\T\right]_{i,j}}
    &\leq \sum_{p=1}^{\infty} \eta^p \sum_{k = 0}^{p-1} \abs{a_{k, p-1-k}}
    d^{p-1} S \lambda^{p-1} \\
    &= \eta S \sum_{p=1}^{\infty} \sum_{k = 0}^{p-1} \abs{a_{k, p-1-k}}
    \left(\eta d \lambda \right)^{p-1} \\
    &= \eta S \sum_{p=1}^{\infty} \sum_{k = 0}^{p-1} \abs{a_{k, p-1-k}}
    \left(\eta d \lambda \right)^{k} \left(\eta d \lambda \right)^{p-1-k} \\
    &= \eta S \sum^{\infty}_{s, m \geq 0} \abs{a_{s, m}}
    \left(\eta d \lambda \right)^{s} \left(\eta d \lambda \right)^{m}.
  \end{align*}
\fi
  According to our previous reasoning, the power series is convergent if
  \begin{align*}
    \eta d \lambda < 1 \quad \Longrightarrow \quad \eta \lambda  < \frac{1}{d},
  \end{align*}
  and we have
  \begin{align*}
    \max_{i, j}\frac{\norm{\left[D(x)D(x)^\T\right]_{i, j}}_{C^0}}{\eta} < \infty.
  \end{align*}

  Next, we consider the first-order derivative of $D(x)D(x)^\T$. Similar to the
  reasoning in the proof of \Cref{lemma:bound_drift}, after taking derivatives
  there are $p d^{p-1}$ terms associated with coefficients $a_{k, p-1-k}$,
  and each term has $p$ factors consisting of one factor being derivatives of $\Sigma(x)$ 
  or $\Sigma(x)$ and others being up to third-order derivatives of $f(x)$.
  Specifically, we have the following terms associated with $a_{k, p-1-k}$:
  \begin{align*}
    (p-1) \partial^3 f \left(\partial^2 f\right)^{p-2} \Sigma + \left(\partial^2 f\right)^{p-1} \partial \Sigma.
  \end{align*}
  Let $B$ be the constant such that $\max_{2 < |a| \leq n+2 \text{and} |a| \neq 2} |D^\alpha f| \leq B$,
  and let $m \coloneqq \max\left\{B \eta, \lambda \eta\right\}$.
  We have
  \begin{align*}
    \abs{\partial_k \left[D(x)D(x)^\T\right]_{i,j}} 
    &\leq \eta S \sum_{p=1}^{\infty} \sum_{k = 0}^{p-1} \abs{a_{k, p-1-k}} (p-1) d^{p-1} B\eta \left(\lambda \eta\right)^{p-2}
     + \eta S \sum_{p=1}^{\infty} \sum_{k = 0}^{p-1} \abs{a_{k, p-1-k}} \left(d \lambda \eta\right)^{p-1}.
  \end{align*}
  The second summation can be bounded the same way as before. For the first 
  summation, it is upper bounded by
  \begin{align*}
    &\eta S \sum_{p=1}^{\infty} \sum_{k = 0}^{p-1} \abs{a_{k, p-1-k}} (p-1)
    \left(d m\right)^{k} \left(dm \right)^{p-1-k}
    = \eta S \sum^{\infty}_{s, m \geq 0} (s+m) \abs{a_{s, m}}
    \left(dm\right)^{s} \left(dm\right)^{m}.
  \end{align*}
  Now we have a power series with coefficients multiplied by $(s+m)$ compared to the previous one.
  However this would not change the convergence radius. Denote the convergence radius
  of $\widetilde{g}(x, y)$ as $r_1$ and $r_2$. According to Cauchy–Hadamard theorem
  for multiple variables,
  \begin{align*}
    \limsup_{s+m \to \infty} \abs{a_{s, m} r_1^s r_2^m}^{\frac{1}{s+m}} = 1.
  \end{align*}
  We note that $r_1$ and $r_2$ are also convergence radius for our new power
  series, since
  \begin{align*}
    \limsup_{s+m \to \infty} \abs{(s+m)a_{s, m}r_1^s r_2^m}^{\frac{1}{s+m}}
    &= \limsup_{s+m \to \infty} \abs{a_{s, m}r_1^s r_2^m}^{\frac{1}{s+m}} 
    \underbrace{\lim_{s+m \to \infty} \abs{s+m}^{\frac{1}{s+m}}}_{=1} \\
    &= \limsup_{s+m \to \infty} \abs{a_{s, m}r_1^s r_2^m}^{\frac{1}{s+m}} = 1.
  \end{align*}
  Therefore a sufficient condition for the convergence of the power series is
  \begin{align*}
    \max\left\{\eta \lambda, \eta B\right\}  < \frac{1}{d}.
  \end{align*}
  
  The proof is similar for higher-order derivatives. 
  Every time we take derivative, the coefficients would be multiplied by a
  factor of order $s+m$, however, that would not change the radius of
  convergence. Specifically, for the $k$-th derivative of $D(x)D(x)^\T$ with
  $k \geq 2$, terms associated with $a_{k, p-1-k}$ have the following form:
  \begin{align*}
   \cO\left(p^k\right) \left(\prod_{j=1}^{m} \partial^{q_j} f\right) \left(\partial^2 f\right)^s \partial^r \Sigma,
  \end{align*}
  with $m+s = p-1$, $q_j \geq 3$, $s \leq p-1$ and $r \leq k$. These terms can be bounded
  using similar idea of the previous proof.

\end{proof}

\begin{lemma} \label{lemma:continuous_expand}
  Let $\cL$ be the infinitesimal generator for an SDE defined as
  \begin{align*}
    d X_t = b(X_t) dt + D(X_t) dW_t,
  \end{align*}
  where $b: \bR^d \to \bR^d$ and $D: \bR^d \to \bR^{d \times d}$.
  We have for any $n \geq 1$ and $t > 0$, it holds
  \begin{align*}
    \norm{e^{t\cL}u - \sum_{p=0}^{n-1} \frac{t^p}{p!} \cL^p u}_{L^{\infty}} \leq \cO\left(t^n \max_{i} \norm{b}_{C^0}\norm{b}_{C^{2(n-1)}}^{n-1} \norm{DD^\T}_{C^{2(n-1)}}^{n} \norm{u}_{C^{2n}}\right),
  \end{align*}
  where $\norm{b}_{C^m} \coloneqq \max_{0 \leq i \leq d} \norm{[b]_i}_{C^m}$ and
  $\norm{DD^\T}_{C^m} \coloneqq \max_{0 \leq i \leq d} \norm{\left[DD^\T\right]_{i, j}}_{C^m}$.
\end{lemma}

\begin{proof}
  According to the Taylor expansion of $e^{t\cL}$ on $t$, we obtain for any $x$
  \begin{align*}
    \abs{e^{t\cL} u(x) - \sum_{p=0}^{n-1} \frac{t^p}{p!} \cL^p u(x)} = \abs{\frac{t^n}{n!}\evalat{\frac{\partial^n e^{t\cL} u(x)}{\left(\partial t\right)^n}}{t=s}},
  \end{align*}
  for some $0 \leq s \leq t$. According to \citet[Lemma~2.1]{feng2017semi}, $e^{t\cL}$
  is a contraction, therefore, we have
  \begin{align*}
    \norm{e^{t\cL} u(x) - \sum_{p=0}^{n-1} \frac{t^p}{p!} \cL^p u(x)}_{L^{\infty}} 
    &= \norm{\frac{t^n}{n!}\evalat{\frac{\partial^n e^{t\cL} u(x)}{\left(\partial t\right)^n}}{t=s}}_{L^{\infty}} \\
    &\leq \frac{t^n}{n!} \norm{\cL^n u}_{L^{\infty}}.
  \end{align*}
  Note that $\cL$ is linear operator and it will apply $\nabla^2$ operator
  Therefore, we obtain gradients of $b$ and $DD^\T$ up to the $2(n-1)$-th order and
  gradients of $u$ up to the $2n$-th order. Note that we bound the first $b$
  using $\norm{b}_{C^0}$ and the following $n-1$ occurrences of $b$ using $\norm{b}_{C^{2(n-1)}}$,
  since they have different dependence on $\lambda$ for \oursde{} according to
  \Cref{lemma:bound_drift}. It would be more convenient for fine-grained analysis later. 
\end{proof}

\begin{lemma} \label{lemma:L_expand}
  For $n \geq 0$, assume $f \in C^{2n}_b(\mathbb{R}^d)$ and $[\Sigma(x)]_{i, j} \in C^{2n-2}_b\left(\mathbb{R}^d\right)$ for any $0 \leq i, j \leq d$.
  Let $\cL$ be the
  infinitesimal generator for \oursde{}, which depends on $\eta$. We have for
  any function $u: \bR^d \to \bR$, and integer $n, p > 0$,
  \begin{align*}
    \norm{\frac{\partial^p \cL^n u}{\left(\partial \eta\right)^p}}_{L^{\infty}} \leq \cO\left(\lambda^{n+p-1} s^n \norm{u}_{C^{2n}}\right),
  \end{align*}
where $\lambda \coloneqq \sup_{x \in \mathbb{R}^d}\norm{\nabla^2 f(x)}$, $s \coloneqq \sup_{x \in \mathbb{R}^d}\norm{\nabla f(x)}$ and
$\cO(\cdot)$ hides the dependence on the upper bounds for the derivatives of $f$
other than the second-order derivatives and $\max_{0 \leq i, j \leq d}
\norm{[\Sigma]_{i,j}}_{C^n}$.
\end{lemma}
\begin{proof}
  We note that when taking partial
  gradient of $b$ w.r.t. $\eta$, the dependence on $\lambda$ on the upper bounds for the derivatives, i.e.,
  $\norm{\left[\frac{\partial b}{\partial \eta}\right]_i}_{C^m}$, would increase
  its order by 1. That is to say, if 
  $\norm{\left[\frac{\partial^k b}{\left(\partial \eta\right)^k}\right]_i}_{C^m} \leq \cO(\lambda^q)$,
  then $\norm{\left[\frac{\partial^{k+1} b}{\left(\partial \eta\right)^{k+1}}\right]_i}_{C^m} \leq \cO(\lambda^{q+1})$.
  The reason is that gradients w.r.t. $x$ of $b$ can be written as power series of $\eta$ and 
  $\eta \lambda$ (see the proof of \Cref{lemma:bound_drift}). Whenever we taking
  partial derivative w.r.t. $\eta$, the order of $\eta$ compared to $\lambda$ in 
  each term is decreased by 1. Then we can take a $\lambda$ factor out and the
  remaining series still converges (the same reasoning as in the proof of \Cref{lemma:bound_drift}).
  The same logic applies to $DD^\T$. In $\cL^n u$, we will have at most $n-1$ occurrences
  of gradients of $b$. Therefore, before taking derivatives w.r.t. $\eta$, the 
  upper bounds would be at most $\lambda^{n-1}$. After taking the derivative
  $p$ times, the order of $\lambda$ is at most $n+p-1$. Similar logic applies to
$DD^\T$ and corresponding \Cref{lemma:bound_diffu}.
\end{proof}

\section{Exact Match of SGD on Quadratics}
Before diving into the proof, we will introduce some basics about
complex normal distribution and complex OU process.

\subsection{Complex Normal Distribution}
\label{sec:pre_complex}
Since our proposed SDE may operate in complex space, similar to PF, we provide
here some basics about the complex distribution, specifically the complex normal distribution.
A complex random vector is defined by two real random variables:
\begin{align*}
  z = x + i y,
\end{align*}
where $x$ and $y$ are two real random vectors, which may be correlated.
The random vector $z$ is called a complex normal vector if 
$\begin{bmatrix} x \\ y \end{bmatrix}$ is a normal vector.
In contrast to real-valued random vector, there are three parameters that
define a complex normal variable $\mathcal{CN}(\mu, \Gamma, C)$,
\begin{align*}
  &\mu \coloneqq \Ep{z} \tag{Expectation} \\
  &\Gamma \coloneqq \Ep{\left(z - \Ep{z}\right)\left(z - \Ep{z}\right)^H} \tag{Covariance} \\
  &C \coloneqq \Ep{\left(z - \Ep{z}\right)\left(z - \Ep{z}\right)^\T} \tag{Pseudo-Covariance}
\end{align*}
Similar to the real case, for $z \sim \mathcal{CN}\left(\mu, \Gamma, C \right)$,
the following property holds
\begin{align*}
  A z + b \sim \mathcal{CN}\left(A \mu + b, A \Gamma A^H, A C A^\T \right).
\end{align*}
If we look at the real and imaginary part separately, we have
\begin{align} \label{eq:complex_normal_cov}
  \begin{split}
  \Gamma &= \Cov{x, x} + \Cov{y, y} + i\left(\Cov{y, x} - \Cov{x, y}\right), \\
  C &= \Cov{x, x} - \Cov{y, y} + i\left(\Cov{y, x} + \Cov{x, y}\right).
\end{split}
\end{align}
For the covariance of $x$ and $y$, we have
\begin{align} \label{eq:complex_normal_cov_2}
  \begin{split}
    \Cov{x, x} = \frac{1}{2} \Re\left(\Gamma + C\right), \quad \Cov{y, y} = \frac{1}{2} \Re\left(\Gamma - C\right).
\end{split}
\end{align}

\subsection{Proofs for Section~\ref{sec:quad}}

\begin{proof}[Proof for \Cref{prop:other_SDE_quad}]

\textbf{Solution for SGD}
  Let us consider more general covariance matrix $\Sigma$ instead of $\sigma^2 I$.
 The iterates of SGD can be written as 
\ifdefined\mpversion
  \begin{align*}
    x_{k} = \left(I - \eta A\right)^{k} x_0 - \eta \sum_{m=0}^{k-1} \left(I - \eta A\right)^m \xi_{k-1-m}.
  \end{align*}
\else
  \begin{align*}
    x_{k} 
    &= x_{k-1} - \eta \left(Ax_{k-1} + \xi_{k-1} \right) \\
    &= \left(I - \eta A\right) x_{k-1} - \eta \xi_{k-1} \\
    &= \left(I - \eta A\right)^2 x_{k-2} - \eta \left(I - \eta A\right) \xi_{k-2} - \eta \xi_{k-1} \\
    &\dots \\
    &= \left(I - \eta A\right)^{k} x_0 - \eta \sum_{m=0}^{k-1} \left(I - \eta A\right)^m \xi_{k-1-m}.
  \end{align*}
\fi
  Clearly, it is a linear combination of independent Gaussian variables, therefore
  $x_{k+1}$ is also has a Gaussian distribution. Moreover,
  \begin{gather*}
    \Ep{x_k} = \left(I - \eta A\right)^{k} x_0 = U \left(I - \eta \Lambda \right)^{k} U^\T x_0 \\
    \Cov{x_k, x_k} = \eta^2 \sum_{m=0}^{k-1} \left(I - \eta A\right)^m \Sigma \left(I - \eta A\right)^m
    = \eta^2 \sum_{m=0}^{k-1} U \left(I - \eta \Lambda \right)^m U^\T \Sigma U \left(I - \eta \Lambda \right)^m U^\T.
  \end{gather*}
  Plugging in $\Sigma = \sigma^2 I$ gives us the desired result.

  The three SDEs considered applied on quadratics with \Cref{assume:add_indep_noise}
  are OU processes. According to \citet[Section~6.2]{sarkka2019applied}, for a linear time-invariant SDE
  \begin{align*}
    d X_t = F X_t dt + L d W_t,
  \end{align*}
  the solution would be a Gaussian variable and its mean and covariance satisfy
  \begin{gather*}
    \Ep{X(x_0, t)} = \exp\left(Ft\right) x_0 \\
    \Cov{X(x_0, t), X(x_0, t)} = \int_0^t \exp\left(F(t-\tau)\right) L L^\T \exp\left(F(t-\tau)\right)^\T d\tau.
  \end{gather*}
  In our case, since we assume isotropic noise, i.e., the noise covariance is $\sigma^2 I$,
  the diffusion coefficients of the three SDEs we consider is also isotropic, i.e., $\sqrt{\eta} \sigma I$.
  Under this condition, $F$ and $L$ commute, and the covariance has closed form
  solution.

  \textbf{Solution for \sme-1} The SDE \sme-1 applied on the problem we consider is 
  \begin{align*}
    d X_t = - A X_t dt + \sqrt{\eta} \sigma I dW_t.
  \end{align*}
  Then the mean is
\ifdefined\mpversion 
$\Ep{X(x_0, t)} = \exp \left(-At\right) x_0$,
\else
  \begin{align*}
    \Ep{X(x_0, t)} = \exp \left(-At\right) x_0,
  \end{align*}
\fi
  and the covariance is 
\ifdefined\mpversion
  \begin{align*}
    \Cov{X(x_0, t), X(x_0, t)} 
    &= \int_0^t \exp\left(A(\tau - t)\right) \eta \sigma^2 \exp\left(A(\tau - t)\right) d\tau \\
    &= \eta \sigma^2 U \frac{I - \exp(-2\Lambda t)}{2 \Lambda} U^\T.
  \end{align*}
\else
  \begin{align*}
    \Cov{X(x_0, t), X(x_0, t)} 
    &= \int_0^t \exp\left(A(\tau - t)\right) \eta \sigma^2 \exp\left(A(\tau - t)\right) d\tau \\
    &= \eta \sigma^2 \int_0^t U \exp\left(\Lambda(\tau - t)\right) U^\T U \exp\left(\Lambda(\tau - t)\right) U^\T d\tau \\
    &= \eta \sigma^2 U \int_0^t \exp\left(\Lambda(\tau - t)\right) \exp\left(\Lambda(\tau - t)\right) d\tau U^\T \\
    &= \eta \sigma^2 U \int_0^t \exp\left(2 \Lambda(\tau - t)\right) d\tau U^\T \\
    &= \eta \sigma^2 U \frac{I - \exp(-2\Lambda t)}{2 \Lambda} U^\T.
  \end{align*}
\fi

  \textbf{Solution for \sme-2} The SDE \sme-2 applied on the problem we consider is 
  \begin{align*}
    d X_t = - (A + \frac{\eta}{2} A^2) X_t dt + \sqrt{\eta} \sigma I dW_t.
  \end{align*}
  Then the mean is
\ifdefined\mpversion 
$\Ep{X(x_0, t)} = \exp \left(-\left(A + \frac{\eta}{2}A^2\right)t\right) x_0$,
\else
  \begin{align*}
    \Ep{X(x_0, t)} = \exp \left(-\left(A + \frac{\eta}{2}A^2\right)t\right) x_0,
  \end{align*}
\fi
  and the covariance is 
\ifdefined\mpversion
  \begin{align*}
    \Cov{X(x_0, t), X(x_0, t)} 
    &= \int_0^t \exp\left(\left(A + \frac{\eta}{2}A^2\right)(\tau - t)\right) \eta \sigma^2 \exp\left(\left(A + \frac{\eta}{2}A^2\right)(\tau - t)\right) d\tau \\
    &= \eta \sigma^2 U \frac{I - \exp(-2\left(\Lambda + \frac{\eta}{2}\Lambda^2\right) t)}{2 \left(\Lambda + \frac{\eta}{2}\Lambda^2\right)} U^\T.
  \end{align*}
\else
  \begin{align*}
    \Cov{X(x_0, t), X(x_0, t)} 
    &= \int_0^t \exp\left(\left(A + \frac{\eta}{2}A^2\right)(\tau - t)\right) \eta \sigma^2 \exp\left(\left(A + \frac{\eta}{2}A^2\right)(\tau - t)\right) d\tau \\
    &= \eta \sigma^2 \int_0^t U \exp\left(\left(\Lambda + \frac{\eta}{2}\Lambda^2\right)(\tau - t)\right) U^\T U \exp\left(\left(\Lambda + \frac{\eta}{2}\Lambda^2\right)(\tau - t)\right) U^\T d\tau \\
    &= \eta \sigma^2 U \int_0^t \exp\left(\left(\Lambda + \frac{\eta}{2}\Lambda^2\right)(\tau - t)\right) \exp\left(\left(\Lambda + \frac{\eta}{2}\Lambda^2\right)(\tau - t)\right) d\tau U^\T \\
    &= \eta \sigma^2 U \int_0^t \exp\left(2 \left(\Lambda + \frac{\eta}{2}\Lambda^2\right)(\tau - t)\right) d\tau U^\T \\
    &= \eta \sigma^2 U \frac{I - \exp(-2\left(\Lambda + \frac{\eta}{2}\Lambda^2\right) t)}{2 \left(\Lambda + \frac{\eta}{2}\Lambda^2\right)} U^\T.
  \end{align*}
\fi

  \textbf{Solution for \pfs} The SDE \pfs applied on the problem we consider is 
  \begin{align*}
    d X_t = U \frac{\log\left(1-\eta \Lambda\right)}{\eta} U^\T X_t dt + \sqrt{\eta} \sigma I dW_t.
  \end{align*}
  Then the mean is
\ifdefined\mpversion
  \begin{align*}
    \Ep{X(x_0, t)} 
    = \exp \left(U \frac{\log\left(1-\eta \Lambda\right)}{\eta} U^\T t\right) x_0
    = U \left(1-\eta \Lambda\right)^{t/\eta} U^\T x_0,
  \end{align*}
\else
  \begin{align*}
    \Ep{X(x_0, t)} 
    &= \exp \left(U \frac{\log\left(1-\eta \Lambda\right)}{\eta} U^\T t\right) x_0 \\
    &= U \exp \left( \log\left(1-\eta \Lambda\right)^{t/\eta} \right) U^\T x_0 \\
    &= U \left(1-\eta \Lambda\right)^{t/\eta} U^\T x_0,
  \end{align*}
\fi
  and the covariance is 
\ifdefined\mpversion
  \begin{align*}
    \Cov{X(x_0, t), X(x_0, t)} 
    &= \int_0^t \exp\left(U \frac{\log\left(1-\eta \Lambda\right)}{\eta} U^\T (\tau - t)\right) \eta \sigma^2 \exp\left(U \frac{\log\left(1-\eta \Lambda\right)}{\eta} U^\T(\tau - t)\right) d\tau \\
    &= \eta^2 \sigma^2 U \left(\frac{1}{-2\log \left(1-\eta \Lambda\right)} \left(1 - \left(1-\eta \Lambda\right)^{2t/\eta} \right)\right) U^\T.
  \end{align*}
\else
  \begin{align*}
    \Cov{X(x_0, t), X(x_0, t)} 
    &= \int_0^t \exp\left(U \frac{\log\left(1-\eta \Lambda\right)}{\eta} U^\T (\tau - t)\right) \eta \sigma^2 \exp\left(U \frac{\log\left(1-\eta \Lambda\right)}{\eta} U^\T(\tau - t)\right) d\tau \\
    &= \eta \sigma^2 U \int_0^t \exp\left(2 \frac{\log\left(1-\eta \Lambda\right)}{\eta} (\tau - t)\right) d\tau U^\T \\
    &= \eta \sigma^2 U \left(\frac{\eta}{-2\log \left(1-\eta \Lambda\right)} \left(1 - \exp\left(\frac{2\log\left(1-\eta \Lambda\right)}{\eta} t\right) \right)\right) U^\T \\
    &= \eta^2 \sigma^2 U \left(\frac{1}{-2\log \left(1-\eta \Lambda\right)} \left(1 - \left(1-\eta \Lambda\right)^{2t/\eta} \right)\right) U^\T.
  \end{align*}
\fi
\end{proof}

\begin{proof}[Proof for \Cref{lemma:dist_match}]
  If the distribution of complex variable matches the real variable, then
  by the definition of matching, we have
  \begin{gather*}
    \Ep{\Re(z)} = \Ep{\widetilde{z}} \quad \Ep{\Im(z)} = 0 \\
    \Cov{\Re(z), \Re(z)} = \Cov{\widetilde{z}, \widetilde{z}} \quad
    \Cov{\Im(z), \Im(z)} = 0.
  \end{gather*}
  Then according to \Cref{eq:complex_normal_cov}, we get
  \begin{align*}
    \Gamma(z) &= \Cov{\Re(z), \Re(z)} + \Cov{\Im(z), \Im(z)} + i\left(\Cov{\Im(z), \Re(z)} - \Cov{\Re(z), \Im(z)}\right)  \\
              &= \Cov{\Re(z), \Re(z)} = \Cov{\widetilde{z}, \widetilde{z}} \\
    \Gamma(z) &= \Cov{\Re(z), \Re(z)} - \Cov{\Im(z), \Im(z)} + i\left(\Cov{\Im(z), \Re(z)} + \Cov{\Re(z), \Im(z)}\right) \\
              &= \Cov{\Re(z), \Re(z)} = \Cov{\widetilde{z}, \widetilde{z}},
  \end{align*}
  as $\Im(z) = 0$ deterministically.
  Next, we consider the opposite direction. The means match trivially. 
  According to \Cref{eq:complex_normal_cov_2}, we have
  \begin{align*}
    \Cov{\Re(z), \Re(z)} &= \frac{1}{2} \left(\Gamma(z) + C(z)\right) = \Cov{\widetilde{z}, \widetilde{z}} \\
    \Cov{\Im(z), \Im(z)} &= \frac{1}{2} \left(\Gamma(z) - C(z)\right) = 0,
  \end{align*}
  which completes the proof.
\end{proof}

\begin{proof}[Proof for \Cref{thm:quad_general}] \label{proof_hard_instance}
Consider function
\begin{align*}
  f(x) = \frac{1}{2} x^\T 
  \begin{bmatrix}
    1 & 0 \\ 0 & -1
  \end{bmatrix} x,
\end{align*}
and
\begin{align*}
  \Sigma = \begin{bmatrix}
    \Sigma_{11} & \Sigma_{12} \\ \Sigma_{12} & \Sigma_{22}
  \end{bmatrix}
  \quad \text{with} \quad \Sigma_{11}, \Sigma_{12}, \Sigma_{22} > 0
  \quad \text{and} \quad \Sigma_{11} \Sigma_{22} - \Sigma_{12}^2 = 0.
\end{align*}
Note that the last condition implies that $\Sigma$ is positive semi-definite
and not positive definite. It is easy to come up with such a $\Sigma$, e.g.,
$\begin{bmatrix} 1 & 2 \\ 2 & 4 \end{bmatrix}$.

According to \Cref{thm:approx_quad}, a necessary condition for matching
the discrete-time iterates is that the RHS of \Cref{eq:match_cov} is 
positive semi-definite, otherwise it cannot be decomposed to 
$\left(U^\T D\right) \left(U^\T D\right)^H$. Therefore, as long as we can show
that the determinant of $\Sigma$ is negative, i.e., it contains a negative
eigenvalue, then we can claim that there is no such a decomposition, i.e., 
no linear SDE can match the discrete-time iterates.

The RHS matrix of \Cref{eq:match_cov} in our case is
\begin{align*}
  \begin{bmatrix}
    \frac{2 \Sigma_{11} \Re\left(\log(1-\eta)\right)}{\eta - 2}
    & \frac{\Sigma_{12} \left(\log(1-\eta) + \log(1+\eta)\right)}{-\eta} \\
  \frac{\Sigma_{12} \left(\log(1+\eta) + \overline{\log(1-\eta)}\right)}{-\eta}
    & \frac{2 \Sigma_{22} \log(1+\eta)}{\eta + 2}
  \end{bmatrix},
\end{align*}
where we used $x + \overline{x} = 2\Re(x)$ to simplify the result.
Then, our goal is to show
\begin{align*}
  \frac{4 \Sigma_{11} \Sigma_{22} \Re\left(\log(1-\eta)\right) \log(1+\eta)}{(\eta - 2) (\eta + 2)}
  - \frac{\Sigma_{12}^2 \left(\log(1-\eta) + \log(1+\eta)\right) \left(\overline{\log(1-\eta)} + \log(1+\eta)\right)}{\eta^2} < 0.
\end{align*}
It is then sufficient to show the following:
\begin{align} \label{eq:ineq_a}
  \frac{\left(\log(1-\eta) + \log(1+\eta)\right) \left(\overline{\log(1-\eta)} + \log(1+\eta)\right) (\eta - 2) (\eta + 2)}{4 \eta^2 \Re\left(\log(1-\eta)\right) \log(1+\eta)}
  > \frac{\Sigma_{11}\Sigma_{22}}{\Sigma_{12}^2} = 1.
\end{align}
To see why this is the case, notice that for all for all $\eta > 0$, it holds that
\begin{align*}
  \frac{\Re(\log(1-\eta))}{\eta - 2} > 0.
\end{align*}
Next, we will discuss in three different cases, i.e., $0 < \eta < 1$, $1 < \eta < 2$ and $\eta > 2$,
to prove the above inequality. For the boundary case, one can check that the
inequality holds when $\eta \to 2$.

\paragraph{When $0 < \eta < 1$} In this case, we have $\Re(\log(1-\eta)) < 0$,
therefore, \Cref{eq:ineq_a} is equivalent to
\begin{align} 
  4 \eta^2 \Re\left(\log(1-\eta)\right) \log(1+\eta) >
  \left(\log(1-\eta) + \log(1+\eta)\right) \left(\overline{\log(1-\eta)} + \log(1+\eta)\right) (\eta - 2) (\eta + 2)  \nonumber \\
  \Longleftrightarrow 4 (\log(1-\eta) + \log(1+\eta)) (\overline{\log(1-\eta)} + \log(1+\eta))
  > \eta^2 (\log(1-\eta) - \log(1+\eta)) (\overline{\log(1-\eta)} - \log(1+\eta)). \label{eq:ineq_b}
\end{align}
When $0<\eta<1$, we have $\overline{\log(1-\eta)} = \log(1-\eta)$,
which implies the following equivalence of \Cref{eq:ineq_b}:
\ifdefined\mpversion
\begin{align*}
  &4 \left(\log(1-\eta) + \log(1+\eta)\right)^2  > \eta^2 \left(\log(1-\eta) - \log(1+\eta)\right)^2 \\
  \Longleftrightarrow \quad&
  \underbrace{ \eta \log(1+\eta) - \eta \log(1-\eta) + 2 \log\left(1-\eta^2\right)}_{g(\eta)} < 0.
\end{align*}
\else
\begin{align*}
  &4 \left(\log(1-\eta) + \log(1+\eta)\right)^2  > \eta^2 \left(\log(1-\eta) - \log(1+\eta)\right)^2 \\
  \Longleftrightarrow \quad&
  4 \left(\log\left(1-\eta^2\right)\right)^2  > \eta^2 \left(\log(1-\eta) - \log(1+\eta)\right)^2 \\
  \Longleftrightarrow \quad&
  - 2 \log\left(1-\eta^2\right) > \eta \left(\log(1+\eta) - \log(1-\eta) \right) \\
  \Longleftrightarrow \quad&
  \underbrace{ \eta \log(1+\eta) - \eta \log(1-\eta) + 2 \log\left(1-\eta^2\right)}_{g(\eta)} < 0.
\end{align*}
\fi
Now we analyze the function $g(\eta)$.
\ifdefined\mpversion
\begin{align*}
  g'(\eta) 
  &= \log(1+\eta) + \frac{\eta}{1+\eta} - \log(1-\eta) + \frac{\eta}{1-\eta}
   - \frac{4 \eta}{1-\eta^2} \\
  &= \log\left(1 + \frac{2\eta}{1-\eta}\right) - \frac{2\eta}{1-\eta^2}
  \leq \frac{\frac{2\eta}{1-\eta}\left(6+\frac{2\eta}{1-\eta}\right)}{6 + \frac{8\eta}{1-\eta}} - \frac{2\eta}{1-\eta^2}
  = -\frac{8 \eta^3}{(1-\eta)(1+\eta)(6+2\eta)} < 0,
\end{align*}
\else
\begin{align*}
  g'(\eta) 
  &= \log(1+\eta) + \frac{\eta}{1+\eta} - \log(1-\eta) + \frac{\eta}{1-\eta}
   - \frac{4 \eta}{1-\eta^2} \\
  &= \log\left(\frac{1+\eta}{1-\eta}\right) - \frac{2\eta}{1-\eta^2} \\
  &= \log\left(1 + \frac{2\eta}{1-\eta}\right) - \frac{2\eta}{1-\eta^2} \\
  &\leq \frac{\frac{2\eta}{1-\eta}\left(6+\frac{2\eta}{1-\eta}\right)}{6 + \frac{8\eta}{1-\eta}} - \frac{2\eta}{1-\eta^2} \\
  &= \frac{12\eta-8\eta^2}{(1-\eta)(6+2\eta)} - \frac{2\eta}{1-\eta^2} \\
  &= -\frac{8 \eta^3}{(1-\eta)(1+\eta)(6+2\eta)} < 0,
\end{align*}
\fi
\ifdefined\mpversion
where the first inequality holds since $\log(1+x) \leq \frac{x(6+x)}{6+4x}$ for $x \geq 0$~\citep[Equation~(22)]{topsoe12007some}.
\else
where the first inequality is according to \Cref{lemma:log_ineq}.
\fi
Now we know that $g(\eta)$ is strictly decreasing over $0 < \eta < 1$,
therefore $\sup_{0<\eta<1} g(\eta) < g(0) = 0$.

\paragraph{When $1 < \eta < 2$} In this case, we also obtain \Cref{eq:ineq_b}.
In addition, since $\log(1-\eta) = \pi i + \log(\eta - 1)$, we have
\begin{align*}
  \log(1-\eta) \overline{\log(1-\eta)} = \log^2(1-\eta) + \pi^2.
\end{align*}
Therefore, \Cref{eq:ineq_b} is equivalent to
\ifdefined\mpversion
\begin{align}
  &4\left(\log^2(\eta-1) + \pi^2 + 2 \log(\eta-1) \log(1+\eta) + \log^2(1+\eta)\right) \nonumber \\
  &\qquad > \eta^2\left(\log^2(\eta-1) + \pi^2 - 2 \log(\eta-1) \log(1+\eta) + \log^2(1+\eta)\right) \nonumber \\
  \Longleftrightarrow \quad&
  (4-\eta^2) \left(\log^2(\eta-1) + \log^2(\eta+1) + \pi^2 \right)
  + (8+2\eta^2) \log(\eta-1) \log(1+\eta) > 0. \label{eq:ineq_c}
\end{align}
\else
\begin{align}
  &4\left(\log^2(\eta-1) + \pi^2 + 2 \log(\eta-1) \log(1+\eta) + \log^2(1+\eta)\right) \nonumber \\
  &\qquad > \eta^2\left(\log^2(\eta-1) + \pi^2 - 2 \log(\eta-1) \log(1+\eta) + \log^2(1+\eta)\right) \nonumber \\
  \Longleftrightarrow \quad&
  4 \left(\log(\eta-1) + \log(1+\eta)\right)^2 - \eta^2 \left(\log(\eta-1) - \log(1+\eta)\right)^2
  + \pi^2 (4 - \eta^2) > 0 \label{eq:ineq_c} \\
  \Longleftrightarrow \quad&
  (4-\eta^2) \left(\log^2(\eta-1) + \log^2(\eta+1) + \pi^2 \right)
  + (8+2\eta^2) \log(\eta-1) \log(1+\eta) > 0. \nonumber
\end{align}
\fi
To prove the above inequality, it is sufficient to prove a lower bound of LHS
is greater than 0.
\ifdefined\mpversion
\begin{align*}
  &\fakeeq \left(4-\eta^2\right) \left(\log^2(\eta-1) + \log^2(\eta+1) + \pi^2 \right)
  + (8+2\eta^2) \log(\eta-1) \log(1+\eta) \\
  &\geq \underbrace{(6-3\eta) \log^2(\eta-1) + \left(4-\eta^2\right) \left( (\eta\log(3/2)+\log(4/3))^2 + \pi^2 \right)
  + (22\eta-12) \log(\eta-1)}_{h(\eta)},
\end{align*}
\else
\begin{align*}
  &\fakeeq \left(4-\eta^2\right) \left(\log^2(\eta-1) + \log^2(\eta+1) + \pi^2 \right)
  + (8+2\eta^2) \log(\eta-1) \log(1+\eta) \\
  &> \left(4-\eta^2\right) \left(\log^2(\eta-1) + \log^2(\eta+1) + \pi^2 \right)
  + (8+2\eta^2) \log(\eta-1) \eta \\
  &\geq \left(4-\eta^2\right) \left(\log^2(\eta-1) + \log^2(\eta+1) + \pi^2 \right)
  + (22\eta-12) \log(\eta-1) \\
  &\geq (6-3\eta) \log^2(\eta-1) + \left(4-\eta^2\right) \left( \log^2(1+\eta) + \pi^2 \right)
  + (22\eta-12) \log(\eta-1) \\
  &\geq \underbrace{(6-3\eta) \log^2(\eta-1) + \left(4-\eta^2\right) \left( (\eta\log(3/2)+\log(4/3))^2 + \pi^2 \right)
  + (22\eta-12) \log(\eta-1)}_{h(\eta)},
\end{align*}
\fi
where the first inequality is because $\log(1+\eta) < \eta$ for $1<\eta<2$,
and the second inequality holds as $2\eta^3 + 2\eta \leq 22\eta - 12$. This
can be seen by noticing that $2\eta^3 + 2\eta$ is convex on this interval and
$22\eta-12$ is the linear interpolation of the endpoints. Similarly, $4-\eta^2$
is concave, and lower bounded by linear interpolation $6-3\eta$, which gives us
the third inequality. The final inequality comes from the linear interpolation
of endpoints of concave function $\log(1+\eta)$.

Next, we study the lower bound of $h(\eta)$. By taking its derivative, we obtain
\ifdefined\mpversion
\begin{align*}
  h'(\eta)
  &= \frac{1}{\eta - 1} ( 3 (1-\eta) \log^2(\eta-1) + (16\eta - 10) \log(\eta-1)
    -4 \eta^4 \log^2 (3/2) \\
  &\fakeeq+ \eta^3 \left(4 \log^2(3/2)-6\log(3/2)\log(4/3)\right)
    +\eta^2\left(6\log(3/2)\log(4/3) - 2\log^2(4/3) - 2 \pi^2 + 8 \log^2(3/2)\right) \\
  &\fakeeq+ \eta (2\log^2(4/3)+2\pi^2-8\log^2(3/2)+8\log(4/3)\log(3/2)+22)
  -8\log(4/3)\log(3/2) - 12) \\
  &\leq  \frac{1}{\eta - 1} ( 16 \eta - 32 
    + 2.5-3\eta + 0.09 - 0.1\eta \\
  &\fakeeq+ \eta^2\left(6\log(3/2)\log(4/3) - 2\log^2(4/3) - 2 \pi^2 + 8 \log^2(3/2)\right) \\
  &\fakeeq+ \eta (2\log^2(4/3)+2\pi^2-8\log^2(3/2)+8\log(4/3)\log(3/2)+22)
  -8\log(4/3)\log(3/2) - 12).
\end{align*}
\else
\begin{align*}
  h'(\eta)
  &= \frac{1}{\eta - 1} ( 3 (1-\eta) \log^2(\eta-1) + (16\eta - 10) \log(\eta-1)
    -4 \eta^4 \log^2 (3/2) \\
  &\fakeeq+ \eta^3 \left(4 \log^2(3/2)-6\log(3/2)\log(4/3)\right)
    +\eta^2\left(6\log(3/2)\log(4/3) - 2\log^2(4/3) - 2 \pi^2 + 8 \log^2(3/2)\right) \\
  &\fakeeq+ \eta (2\log^2(4/3)+2\pi^2-8\log^2(3/2)+8\log(4/3)\log(3/2)+22)
  -8\log(4/3)\log(3/2) - 12) \\
  &\leq  \frac{1}{\eta - 1} (  (16\eta - 10) \log(\eta-1)
    -4 \eta^4 \log^2 (3/2) + \eta^3 \left(4 \log^2(3/2)-6\log(3/2)\log(4/3)\right) \\
  &\fakeeq+ \eta^2\left(6\log(3/2)\log(4/3) - 2\log^2(4/3) - 2 \pi^2 + 8 \log^2(3/2)\right) \\
  &\fakeeq+ \eta (2\log^2(4/3)+2\pi^2-8\log^2(3/2)+8\log(4/3)\log(3/2)+22)
  -8\log(4/3)\log(3/2) - 12) \\
  &\leq  \frac{1}{\eta - 1} ( 16 \eta - 32 
    -4 \eta^4 \log^2 (3/2) + \eta^3 \left(4 \log^2(3/2)-6\log(3/2)\log(4/3)\right) \\
  &\fakeeq+ \eta^2\left(6\log(3/2)\log(4/3) - 2\log^2(4/3) - 2 \pi^2 + 8 \log^2(3/2)\right) \\
  &\fakeeq+ \eta (2\log^2(4/3)+2\pi^2-8\log^2(3/2)+8\log(4/3)\log(3/2)+22)
  -8\log(4/3)\log(3/2) - 12) \\
  &\leq  \frac{1}{\eta - 1} ( 16 \eta - 32 
    + 2.5-3\eta + \eta^3 \left(4 \log^2(3/2)-6\log(3/2)\log(4/3)\right) \\
  &\fakeeq+ \eta^2\left(6\log(3/2)\log(4/3) - 2\log^2(4/3) - 2 \pi^2 + 8 \log^2(3/2)\right) \\
  &\fakeeq+ \eta (2\log^2(4/3)+2\pi^2-8\log^2(3/2)+8\log(4/3)\log(3/2)+22)
  -8\log(4/3)\log(3/2) - 12) \\
  &\leq  \frac{1}{\eta - 1} ( 16 \eta - 32 
    + 2.5-3\eta + 0.09 - 0.1\eta \\
  &\fakeeq+ \eta^2\left(6\log(3/2)\log(4/3) - 2\log^2(4/3) - 2 \pi^2 + 8 \log^2(3/2)\right) \\
  &\fakeeq+ \eta (2\log^2(4/3)+2\pi^2-8\log^2(3/2)+8\log(4/3)\log(3/2)+22)
  -8\log(4/3)\log(3/2) - 12).
\end{align*}
\fi
\ifdefined\mpversion
where the second inequality holds because $(5-8x)\log(x-1) \geq 16-8x$ for 
$1\leq x \leq 2$, the third inequality holds since $-4 x^4 \log^2(3/2) \leq 5/2 - 3x$
for $1 \leq x \leq 2$, and the forth inequality is because $x^3 \left(4\log^2(3/2) - 6\log(3/2)\log(4/3)\right) \leq 0.09 - 0.1x$
for $1 \leq x \leq 2$.
\else
where the second, third and forth inequalities are according to 
\Cref{lemma:log_ineq2,poly_ineq1,poly_ineq2} respectively.
\fi
The remaining numerator is a quadratic function of $\eta$, and one
can easily verify that it is less than 0. Therefore, we conclude that 
$h'(\eta) < 0$. Hence $h(x) \geq h(2) = 0$.

\paragraph{When $\eta > 2$} In this case, our target becomes to prove
\Cref{eq:ineq_c}, however with the opposite sign of the inequality.
\ifdefined\mpversion
\begin{align*}
  &\fakeeq 4\left(\log(\eta-1)+\log(1+\eta)\right)^2 - \eta^2\left(\log(\eta-1) - \log(1+\eta)\right)^2
  +\pi^2\left(4-\eta^2\right) \\
  &\leq \frac{-(x-2)\left(\left(2\pi^2-8\right)x^2+\left(7\pi^2-48\right)x + 6\pi^2\right)}{3+2\eta}.
\end{align*}
\else
\begin{align*}
  &\fakeeq 4\left(\log(\eta-1)+\log(1+\eta)\right)^2 - \eta^2\left(\log(\eta-1) - \log(1+\eta)\right)^2
  +\pi^2\left(4-\eta^2\right) \\
  &\leq 4\left(\log(\eta-1)+\log(1+\eta)\right)^2 - 4\left(\log(\eta-1) - \log(1+\eta)\right)^2
  +\pi^2\left(4-\eta^2\right) \\
  &= 16 \log(\eta-1) \log(1+\eta) +\pi^2\left(4-\eta^2\right) \\
  &< 16 (\eta-2) \log(1+\eta) +\pi^2\left(4-\eta^2\right) \\
  &\leq 16 (\eta-2) \frac{\eta(6+\eta)}{6+4\eta} +\pi^2\left(4-\eta^2\right) \\
  &= \frac{(8-2\pi^2)\eta^3 + (32-3\pi^2)\eta^2 +(8\pi^2-96)\eta + 12 \pi^2}{3+2\eta} \\
  &= \frac{-(x-2)\left(\left(2\pi^2-8\right)x^2+\left(7\pi^2-48\right)x + 6\pi^2\right)}{3+2\eta},
\end{align*}
where the second inequality holds as $\log(x) < x-1$ for $x>1$ and the third inequality
is according to \Cref{lemma:log_ineq}.
\fi
In addition, $\left(2\pi^2-8\right)x^2+\left(7\pi^2-48\right)x + 6\pi^2 \geq 0$,
therefore, we have
\begin{align*}
  4\left(\log(\eta-1)+\log(1+\eta)\right)^2 - \eta^2\left(\log(\eta-1) - \log(1+\eta)\right)^2
  +\pi^2\left(4-\eta^2\right) < 0,
\end{align*}
which concludes the proof.

\end{proof}

\begin{proof}[Proof for \Cref{thm:approx_quad}]
  To match the expectation, the argument is similar to the proof in 
  \citet[Section~A.6]{rosca2022continuous}. Here, we proceed to match the 
  covariance and pseudo-covariance.
  Let us first look at the continuous-time process.
  We denote the covairance of $X$ as $\Gamma(X)$.
  According to \Cref{thm:cov_complex}, we have 
\ifdefined\mpversion
  \begin{align*}
    \left[U^\T \Gamma(t) U\right]_{i, j}
    = \frac{\eta^2 \left[U^\T \Sigma U\right]_{i, j}}{\left(1-\eta \lambda_i\right)\left(1-\eta \lambda_j\right) - 1}
    \left(\left((1 - \eta \lambda_i)(1 - \eta \lambda_j)\right)^{\frac{t}{\eta}} - 1\right),
  \end{align*}
\else
  \begin{align*}
    &\fakeeq \left[U^\T \Gamma(t) U\right]_{i, j} \\
    &= \frac{\left[U^\T \Sigma U\right]_{i, j} \left(\log\left(1-\eta \lambda_i\right) + \overline{\log\left(1-\eta \lambda_i\right)} \right)}{\eta \lambda_i \lambda_j - (\lambda_i + \lambda_j)}
    \int_0^t \exp\left(\tau \frac{\log\left(1-\eta \lambda_i\right)}{\eta}\right)
    \exp\left(\tau \frac{\overline{\log\left(1-\eta \lambda_j\right)}}{\eta}\right) d\tau \\
    &= \frac{\eta \left[U^\T \Sigma U\right]_{i, j} \left(\log\left(1-\eta \lambda_i\right) + \overline{\log\left(1-\eta \lambda_i\right)} \right)}{\left(1-\eta \lambda_i\right)\left(1-\eta \lambda_j\right) - 1}
    \int_0^t \exp\left(\frac{\tau}{\eta} \left(\log\left(1-\eta \lambda_i\right) + \overline{\log\left(1-\eta \lambda_j\right)}\right) \right) d\tau \\
    &= \frac{\eta \left[U^\T \Sigma U\right]_{i, j} \left(\log\left(1-\eta \lambda_i\right) + \overline{\log\left(1-\eta \lambda_i\right)} \right)}{\left(1-\eta \lambda_i\right)\left(1-\eta \lambda_j\right) - 1}
    \frac{\eta}{\log\left(1-\eta \lambda_i\right) + \overline{\log\left(1-\eta \lambda_i\right)}}
      \left(\exp\left(\frac{t}{\eta} \left(\log\left(1-\eta \lambda_i\right) + \overline{\log\left(1-\eta \lambda_j\right)}\right) \right) - 1\right) \\
    &= \frac{\eta^2 \left[U^\T \Sigma U\right]_{i, j}}{\left(1-\eta \lambda_i\right)\left(1-\eta \lambda_j\right) - 1}
      \left(\exp\left(\frac{t}{\eta} \left(\log\left(1-\eta \lambda_i\right) + \overline{\log\left(1-\eta \lambda_j\right)}\right) \right) - 1\right) \\
    &= \frac{\eta^2 \left[U^\T \Sigma U\right]_{i, j}}{\left(1-\eta \lambda_i\right)\left(1-\eta \lambda_j\right) - 1}
    \left(\left((1 - \eta \lambda_i)(1 - \eta \lambda_j)\right)^{\frac{t}{\eta}} - 1\right),
  \end{align*}
\fi
  where we used the fact that for complex-valued variable $z$, it holds 
  that $\exp\left(\overline{z}\right) = \overline{\exp(z)}$.
  Now we consider time stamp $t=k \eta$ for $k = 1, 2, \dots$, where we have
\ifdefined\mpversion
  \begin{align*}
    \left[U^\T \Gamma(k \eta) U\right]_{i, j}
    = \eta^2 \left[U^\T \Sigma U\right]_{i, j} \sum_{m=0}^{k-1} \left((1 - \eta \lambda_i)(1 - \eta \lambda_j)\right)^m.
  \end{align*}
\else
  \begin{align*}
    \left[U^\T \Gamma(k \eta) U\right]_{i, j}
    &= \frac{\eta^2 \left[U^\T \Sigma U\right]_{i, j}}{\left(1-\eta \lambda_i\right)\left(1-\eta \lambda_j\right) - 1}
    \left(\left((1 - \eta \lambda_i)(1 - \eta \lambda_j)\right)^{k} - 1\right) \\
    &= \eta^2 \left[U^\T \Sigma U\right]_{i, j} \sum_{m=0}^{k-1} \left((1 - \eta \lambda_i)(1 - \eta \lambda_j)\right)^m.
  \end{align*}
\fi
  Then we have
  \begin{align*}
    U^\T \Gamma(k \eta) U = \eta^2 \sum_{m=0}^{k-1} \left(I - \eta \Lambda\right)^m U^\T \Sigma U \left(I - \eta \Lambda\right)^m,
  \end{align*}
  which implies
\ifdefined\mpversion
  \begin{align*}
    \Gamma(k \eta) 
    &= \eta^2 \sum_{m=0}^{k-1} \left(I - \eta A\right)^m \Sigma \left(I - \eta A\right)^m,
  \end{align*}
\else
  \begin{align*}
    \Gamma(k \eta) 
    &= \eta^2 \sum_{m=0}^{k-1} U \left(I - \eta \Lambda\right)^m U^\T \Sigma U \left(I - \eta \Lambda\right)^m U^\T \\
    &= \eta^2 \sum_{m=0}^{k-1} \left(I - \eta A\right)^m \Sigma \left(I - \eta A\right)^m,
  \end{align*}
\fi
  which matches the covariance of discrete-time variable.
  The proof for matching pseudo-covariance is almost the same.
\end{proof}

\begin{proof}[Proof of \Cref{cor:spf_match_quad}]
  This is a direct result from \Cref{thm:approx_quad}.
\end{proof}

\subsection{Helper Lemmas}
\ifdefined\mpversion
\else
\begin{lemma}[{\citet[Equation~(22)]{topsoe12007some}}] \label{lemma:log_ineq}
  For $x \in \mathbb{R}$ and $x \geq 0$, it holds that
  \begin{align*}
    \log(1+x) \leq \frac{x(6+x)}{6+4x}.
  \end{align*}
\end{lemma}

\begin{lemma} \label{lemma:log_ineq2}
  For $x \in \mathbb{R}$ and $1\leq x \leq 2$, it holds that
  \begin{align*}
    (5-8x)\log(x-1) \geq 16-8x.
  \end{align*}
\end{lemma}
\begin{proof}
  Denote $f(x) = (5-8x)\log(x-1) - 16+8x$, and we have
  \begin{align*}
    f'(x) = \frac{-8(x-1)\log(x-1) -3}{x-1}.
  \end{align*}
  Let $g(x) = (x-1)\log(x-1)$, then we get
  \begin{align*}
    g'(x) = 1 + \log(x-1).
  \end{align*}
  We know that the minimum of $g(x)$ is obtained at $x=1/e + 1$, i.e.,
  $g(x) \geq -1/e$. Therefore,
  \begin{align*}
    f'(x) \leq \frac{8/e - 3}{x-1} \leq 0.
  \end{align*}
  Hence $f(x)$ is decreasing, and $f(x) \geq f(2) \geq 0$, which concludes
  the proof.
\end{proof}

\begin{lemma} \label{poly_ineq1}
  For $x \in \mathbb{R}$ and $1\leq x \leq 2$, it holds that
  \begin{align*}
    -4 x^4 \log^2(3/2) \leq 5/2 - 3x.
  \end{align*}
\end{lemma}
\begin{proof}
  Denote $f(x)=-4 x^4 \log^2(3/2) - 2.5 + 3x$. We have
  \begin{align*}
    f'(x) = -16 x^3 \log^2(3/2) + 3, \quad \text{and} \quad
    f''(x) = -48 x^2 \log^2(3/2) < 0.
  \end{align*}
  Therefore, $f'(x)$ is decreasing and is 0 when $x=\sqrt[3]{\frac{3}{16\log^2(3/2)}}$.
  Hence, $f(x)$ obtains its maximum at this point and $f\left(\sqrt[3]{\frac{3}{16\log^2(3/2)}}\right) < 0$,
  which completes the proof.
\end{proof}

\begin{lemma} \label{poly_ineq2}
  For $x \in \mathbb{R}$ and $1\leq x \leq 2$, it holds that
  \begin{align*}
    x^3 \left(4\log^2(3/2) - 6\log(3/2)\log(4/3)\right) \leq 0.09 - 0.1x.
  \end{align*}
\end{lemma}
\begin{proof}
  Denote $f(x) = x^3 \left(4\log^2(3/2) - 6\log(3/2)\log(4/3)\right) - 0.09 + 0.1x$.
  We have
  \begin{align*}
    f'(x) = 3 x^2 \left(4\log^2(3/2) - 6\log(3/2)\log(4/3)\right) + 0.1 \leq f'(1) \leq  0,
  \end{align*}
  since $4\log^2(3/2) - 6\log(3/2)\log(4/3) < 0$. Therefore, $f(x)$ is non-increasing
  and $f(x) \leq f(1) \leq 0$.
\end{proof}
\fi

\begin{lemma} \label{thm:cov_complex}
 Consider complex-valued OU process
  \begin{align*}
    d z = A z dt + D dW_t,
  \end{align*}
  where $z \in \mathbb{C}^d$, $A \in \mathbb{C}^{d \times d}$ is diagonalizable
  with orthogonal basis, i.e., $A = U S U^\T$ with diagonal matrix $S$ and orthogonal matrix $U$,
  $D \in \mathbb{C}^{d \times m}$, and $W_t \in \mathbb{C}^{m}$ is a standard 
  real Brownian motion. The covariance and pseudo-covariance at time $t$ are
\begin{align*}
  \Gamma(t) = U \int_0^t \exp\left(\tau S\right) U^\T D D^H U \exp\left(\tau \overline S\right) d \tau U^\T \\
  C(t) = U \int_0^t \exp\left(\tau S\right) U^\T D D^\T U \exp\left(\tau S\right) d \tau U^\T.
\end{align*}
\end{lemma}
\begin{proof}
Recall that the covariance and pseudo-covariance are closely related to 4 matrices,
$\Cov{\Re(z), \Re(z)}$, $\Cov{\Im(z), \Im(z)}$, $\Cov{\Re(z), \Im(z)}$
$\Cov{\Im(z), \Re(z)}$. Therefore, we study the dynamics of these matrices.

We first notice that complex-valued OU process can be re-written
as a real-valued SDE.
\begin{align*}
  d \begin{bmatrix} \Re(z) \\ \Im(z) \end{bmatrix} = 
  \begin{bmatrix} \Re(A)  & -\Im(A) \\ \Im(A) & \Re(A)  \end{bmatrix}
  \begin{bmatrix} \Re(z) \\ \Im(z) \end{bmatrix} dt
  + \begin{bmatrix} \Re(D) \\ \Im(D) \end{bmatrix} dW_t.
\end{align*}
We assume $z(0)$ is a deterministic variable, i.e., its covariance is $0$,
then according to Equation~(6.20) of \citet{sarkka2019applied},
\begin{align}
  \begin{split} \label{eq:aaa}
  &\fakeeq \Ep{\begin{bmatrix} \Re(z(t)) - \Ep{\Re(z(t))}\\ \Im(z(t)) - \Ep{\Im(z(t))} \end{bmatrix} \begin{bmatrix} \Re(z(t))  - \Ep{\Re(z(t))} \\ \Im(z(t)) - \Ep{\Im(z(t))} \end{bmatrix}^\T}
    = \begin{bmatrix} \Cov{\Re(z), \Re(z)} & \Cov{\Re(z), \Im(z)} \\ \Cov{\Im(z), \Re(z)} & \Cov{\Im(z), \Im(z)} \end{bmatrix} \\
  &= \int_0^t \exp\left(\tau \begin{bmatrix} \Re(A)  & -\Im(A) \\ \Im(A) & \Re(A)  \end{bmatrix}\right)
    \begin{bmatrix} \Re(D) \\ \Im(D) \end{bmatrix}
  \left(\exp\left(\tau \begin{bmatrix} \Re(A)  & -\Im(A) \\ \Im(A) & \Re(A)  \end{bmatrix}\right)
    \begin{bmatrix} \Re(D) \\ \Im(D) \end{bmatrix}\right)^\T d\tau,
  \end{split}
\end{align}
where the exponential operator is matrix exponential.

Next, let us look at the matrix $\exp\left(\tau \begin{bmatrix} \Re(A)  & -\Im(A) \\ \Im(A) & \Re(A)  \end{bmatrix}\right)$,
which has some special properties. Note that we have
\begin{align*}
  \Re(A) = U \Re(S) U^\T, \quad \text{and} \quad \Im(A) = U \Im(S) U^\T.
\end{align*}
Therefore, we have
\begin{align*}
  \begin{bmatrix} \Re(A)  & -\Im(A) \\ \Im(A) & \Re(A)  \end{bmatrix}
  = \begin{bmatrix} U\Re(S)U^\T & -U\Im(S)U^\T \\ U\Im(S)U^\T & U\Re(S)U^\T  \end{bmatrix}
  = \begin{bmatrix} U &  \\  & U \end{bmatrix} \begin{bmatrix} \Re(S)  & -\Im(S) \\ \Im(S) & \Re(S)  \end{bmatrix} \begin{bmatrix} U^\T &  \\  & U^\T \end{bmatrix}.
\end{align*}
Note that the matrix $\begin{bmatrix} U &  \\  & U \end{bmatrix}$ is orthogonal,
so by the definition of matrix exponential,
\begin{align*}
  \exp\left(\tau \begin{bmatrix} \Re(A)  & -\Im(A) \\ \Im(A) & \Re(A)  \end{bmatrix} \right)
  &= \sum_{p=0}^{\infty} \frac{\tau^p}{p!} \begin{bmatrix} \Re(A)  & -\Im(A) \\ \Im(A) & \Re(A)  \end{bmatrix}^p \\
  &= \begin{bmatrix} U &  \\  & U \end{bmatrix} \sum_{p=0}^{\infty} \frac{\tau^p}{p!} \begin{bmatrix} \Re(S)  & -\Im(S) \\ \Im(S) & \Re(S)  \end{bmatrix}^p \begin{bmatrix} U^\T &  \\  & U^\T \end{bmatrix}\\
  &= \begin{bmatrix} U &  \\  & U \end{bmatrix} \exp\left(\tau \begin{bmatrix} \Re(S)  & -\Im(S) \\ \Im(S) & \Re(S)  \end{bmatrix} \right) \begin{bmatrix} U^\T &  \\  & U^\T \end{bmatrix}.
\end{align*}
By the property of matrix exponential, if for matrices $X$ and $Y$, we have $X Y = Y X$,
then $\exp(X+Y) = \exp(X)\exp(Y)$. We verify that this property holds for our case:
\begin{align*}
  \begin{bmatrix} \Re(S)  & -\Im(S) \\ \Im(S) & \Re(S)  \end{bmatrix} = 
\begin{bmatrix} \Re(S)  & \\ & \Re(S)  \end{bmatrix} + \begin{bmatrix}  & -\Im(S) \\ \Im(S) & \end{bmatrix},
\end{align*}
and 
\begin{align*}
  \begin{bmatrix} \Re(S)  & \\ & \Re(S)  \end{bmatrix} \begin{bmatrix}  & -\Im(S) \\ \Im(S) & \end{bmatrix}
&= \begin{bmatrix}  & -\Re(S) \Im(S) \\ \Re(S)\Im(S) & \end{bmatrix}
= \begin{bmatrix}  & -\Im(S) \\ \Im(S) & \end{bmatrix} \begin{bmatrix} \Re(S)  & \\ & \Re(S)  \end{bmatrix}.
\end{align*}
Therefore, we have
\begin{align*}
  \exp\left( \tau \begin{bmatrix} \Re(S)  & -\Im(S) \\ \Im(S) & \Re(S)  \end{bmatrix} \right)
  &= \exp\left(\tau \begin{bmatrix} \Re(S)  & \\ & \Re(S)  \end{bmatrix}\right)
  \exp\left(\tau \begin{bmatrix}  & -\Im(S) \\ \Im(S) & \end{bmatrix}\right) \\
  &= \begin{bmatrix} \exp(\tau \Re(S))  & \\ & \exp(\tau \Re(S))  \end{bmatrix}
  \exp\left(\tau \begin{bmatrix}  & -\Im(S) \\ \Im(S) & \end{bmatrix}\right),
\end{align*}
Where the last equality holds because $\Re(S)$ is diagonal. Next, we will show that 
\begin{align*}
  \exp\left(\tau \begin{bmatrix}  & -\Im(S) \\ \Im(S) & \end{bmatrix}\right)
  = \begin{bmatrix} \cos(\tau \Im(S)) & -\sin(\tau \Im(S)) \\ \sin(\tau \Im(S)) & \cos(\tau \Im(S)) \end{bmatrix},
\end{align*}
where the $\sin$ and $\cos$ operators are applied element-wise.
To see this, let us write down the expansion of the matrix exponential.
\begin{align*}
 \exp\left(\tau \begin{bmatrix}  & -\Im(S) \\ \Im(S) & \end{bmatrix}\right)
 &= I + \begin{bmatrix}  & -\tau \Im(S) \\ \tau \Im(S) & \end{bmatrix}
  + \frac{1}{2} \begin{bmatrix}  -\tau^2 \Im(S)^2 & \\ & -\tau^2 \Im(S)^2 & \end{bmatrix} \\
 &\fakeeq + \frac{1}{3!}  \begin{bmatrix}  & \tau^3 \Im(S)^3 \\ -\tau^3 \Im(S)^3 & \end{bmatrix} 
 + \frac{1}{4!} \begin{bmatrix}  \tau^4 \Im(S)^4 & \\ & \tau^4 \Im(S)^4 & \end{bmatrix} + \cdots \\
&= \begin{bmatrix} \cos(\tau \Im(S)) & -\sin(\tau \Im(S)) \\ \sin(\tau \Im(S)) & \cos(\tau \Im(S)) \end{bmatrix}.
\end{align*}
Now going back to \Cref{eq:aaa}, we have
\ifdefined\mpversion
\begin{align*}
  &\fakeeq \begin{bmatrix} \Cov{\Re(z), \Re(z)} & \Cov{\Re(z), \Im(z)} \\ \Cov{\Im(z), \Re(z)} & \Cov{\Im(z), \Im(z)} \end{bmatrix} \\
  &= \begin{bmatrix} U &  \\  & U \end{bmatrix} \int_0^t \begin{bmatrix} \exp(\tau \Re(S)) \cos(\tau \Im(S)) & -\exp(\tau \Re(S)) \sin(\tau \Im(S)) 
  \\ \exp(\tau \Re(S)) \sin(\tau \Im(S)) & \exp(\tau \Re(S)) \cos(\tau \Im(S)) \end{bmatrix} 
  \begin{bmatrix} U^\T \Re(D) \\ U^\T \Im(D) \end{bmatrix}  \\
  &\fakeeq \left(\begin{bmatrix} \exp(\tau \Re(S)) \cos(\tau \Im(S)) & -\exp(\tau \Re(S)) \sin(\tau \Im(S)) 
  \\ \exp(\tau \Re(S)) \sin(\tau \Im(S)) & \exp(\tau \Re(S)) \cos(\tau \Im(S)) \end{bmatrix} 
  \begin{bmatrix} U^\T \Re(D) \\ U^\T \Im(D) \end{bmatrix}\right)^\T d\tau \begin{bmatrix} U^\T &  \\  & U^\T \end{bmatrix} \\
  &= \begin{bmatrix} U &  \\  & U \end{bmatrix} \int_0^t \begin{bmatrix} P(\tau) & Q(\tau) \\ M(\tau) & N(\tau) \end{bmatrix} d \tau \begin{bmatrix} U^\T &  \\  & U^\T \end{bmatrix},
\end{align*}
where $P, Q, M, N$ are defined according to the last equality.
Denote the covariance of $z(t)$ as $\Gamma(t)$. 
According to the definition of covariance for complex random variables, we have
\begin{align*}
  U^\T \Gamma(t) U = \int_0^t P(\tau) d \tau + \int_0^t N(\tau) d \tau + i\left(\int_0^t M(\tau) d \tau - \int_0^t Q(\tau) d \tau\right).
\end{align*}
Let us denote the $i$-th eigenvalue of $A$ as $a_i$, i.e., the $i$-th
diagonal element of $S$. We can derive that
\begin{align*}
  \left[U^\T \Re(D)  \Re(D)^\T U\right]_{i, j} = 
  \int_0^t \exp\left(\tau (a_i + \overline a_j)\right) d \tau,
\end{align*}
where $\overline x$ is the complex conjugate of $x$.
\else
\begin{align*}
  &\fakeeq \begin{bmatrix} \Cov{\Re(z), \Re(z)} & \Cov{\Re(z), \Im(z)} \\ \Cov{\Im(z), \Re(z)} & \Cov{\Im(z), \Im(z)} \end{bmatrix} \\
  &= \begin{bmatrix} U &  \\  & U \end{bmatrix} \int_0^t \begin{bmatrix} \exp(\tau \Re(S)) \cos(\tau \Im(S)) & -\exp(\tau \Re(S)) \sin(\tau \Im(S)) 
  \\ \exp(\tau \Re(S)) \sin(\tau \Im(S)) & \exp(\tau \Re(S)) \cos(\tau \Im(S)) \end{bmatrix} 
  \begin{bmatrix} U^\T \Re(D) \\ U^\T \Im(D) \end{bmatrix}  \\
  &\fakeeq \left(\begin{bmatrix} \exp(\tau \Re(S)) \cos(\tau \Im(S)) & -\exp(\tau \Re(S)) \sin(\tau \Im(S)) 
  \\ \exp(\tau \Re(S)) \sin(\tau \Im(S)) & \exp(\tau \Re(S)) \cos(\tau \Im(S)) \end{bmatrix} 
  \begin{bmatrix} U^\T \Re(D) \\ U^\T \Im(D) \end{bmatrix}\right)^\T d\tau \begin{bmatrix} U^\T &  \\  & U^\T \end{bmatrix} \\
  &= \begin{bmatrix} U &  \\  & U \end{bmatrix} \int_0^t \begin{bmatrix} P(\tau) & Q(\tau) \\ M(\tau) & N(\tau) \end{bmatrix} d \tau \begin{bmatrix} U^\T &  \\  & U^\T \end{bmatrix},
\end{align*}
where
\begin{align*}
  P(\tau) &= \exp(\tau \Re(S)) \cos(\tau \Im(S)) U^\T \Re(D)  \Re(D)^\T U \exp(\tau \Re(S)) \cos(\tau \Im(S)) \\
          &\fakeeq - \exp(\tau \Re(S)) \cos(\tau \Im(S)) U^\T \Re(D) \Im(D)^\T U \exp(\tau \Re(S)) \sin(\tau \Im(S)) \\
          &\fakeeq - \exp(\tau \Re(S)) \sin(\tau \Im(S)) U^\T \Im(D) \Re(D)^\T U \exp(\tau \Re(S)) \cos(\tau \Im(S)) \\
          &\fakeeq + \exp(\tau \Re(S)) \sin(\tau \Im(S)) U^\T \Im(D) \Im(D)^\T U \exp(\tau \Re(S)) \sin(\tau \Im(S)), \\
  Q(\tau) &= \exp(\tau \Re(S)) \cos(\tau \Im(S)) U^\T \Re(D)  \Re(D)^\T U \exp(\tau \Re(S)) \sin(\tau \Im(S)) \\
          &\fakeeq + \exp(\tau \Re(S)) \cos(\tau \Im(S)) U^\T \Re(D) \Im(D)^\T U \exp(\tau \Re(S)) \cos(\tau \Im(S)) \\
          &\fakeeq - \exp(\tau \Re(S)) \sin(\tau \Im(S)) U^\T \Im(D) \Re(D)^\T U \exp(\tau \Re(S)) \sin(\tau \Im(S)) \\
          &\fakeeq - \exp(\tau \Re(S)) \sin(\tau \Im(S)) U^\T \Im(D) \Im(D)^\T U \exp(\tau \Re(S)) \cos(\tau \Im(S)), \\
  M(\tau) &= \exp(\tau \Re(S)) \sin(\tau \Im(S)) U^\T \Re(D)  \Re(D)^\T U \exp(\tau \Re(S)) \cos(\tau \Im(S)) \\
          &\fakeeq - \exp(\tau \Re(S)) \sin(\tau \Im(S)) U^\T \Re(D) \Im(D)^\T U \exp(\tau \Re(S)) \sin(\tau \Im(S)) \\
          &\fakeeq + \exp(\tau \Re(S)) \cos(\tau \Im(S)) U^\T \Im(D) \Re(D)^\T U \exp(\tau \Re(S)) \cos(\tau \Im(S)) \\
          &\fakeeq - \exp(\tau \Re(S)) \cos(\tau \Im(S)) U^\T \Im(D) \Im(D)^\T U \exp(\tau \Re(S)) \sin(\tau \Im(S)), \\
  N(\tau) &= \exp(\tau \Re(S)) \sin(\tau \Im(S)) U^\T \Re(D)  \Re(D)^\T U \exp(\tau \Re(S)) \sin(\tau \Im(S)) \\
          &\fakeeq + \exp(\tau \Re(S)) \sin(\tau \Im(S)) U^\T \Re(D) \Im(D)^\T U \exp(\tau \Re(S)) \cos(\tau \Im(S)) \\
          &\fakeeq + \exp(\tau \Re(S)) \cos(\tau \Im(S)) U^\T \Im(D) \Re(D)^\T U \exp(\tau \Re(S)) \sin(\tau \Im(S)) \\
          &\fakeeq + \exp(\tau \Re(S)) \cos(\tau \Im(S)) U^\T \Im(D) \Im(D)^\T U \exp(\tau \Re(S)) \cos(\tau \Im(S)).
\end{align*}
Now we have
\begin{align*}
  \begin{bmatrix} U^\T \Cov{\Re(z), \Re(z)} U & U^\T \Cov{\Re(z), \Im(z)} U \\
  U^\T \Cov{\Im(z), \Re(z)} U & U^\T \Cov{\Im(z), \Im(z)} U \end{bmatrix} 
    = \begin{bmatrix} \int_0^t P(\tau) d \tau & \int_0^t Q(\tau) d \tau \\ \int_0^t M(\tau) d \tau & \int_0^t N(\tau) d \tau \end{bmatrix}.
\end{align*}
Denote the covariance of $z(t)$ as $\Gamma(t)$. 
According to the definition of covariance for complex random variables, we have
\begin{align*}
  U^\T \Gamma(t) U = \int_0^t P(\tau) d \tau + \int_0^t N(\tau) d \tau + i\left(\int_0^t M(\tau) d \tau - \int_0^t Q(\tau) d \tau\right).
\end{align*}
Let us denote the $i$-th eigenvalue of $A$ as $a_i$, i.e., the $i$-th
diagonal element of $S$. Since $\exp(\tau \Re(S)) \cos(\tau \Im(S))$
and $\exp(\tau \Re(S)) \sin(\tau \Im(S))$ are both diagonal, we have
the $(i, j)$-th element of $U^\T \Gamma(t) U$ satisfies
\begin{align*}
  &\fakeeq \left[U^\T \Gamma(t) U\right]_{i, j}  \\
  &= \int_0^t \exp\left(\tau \Re(a_i) + \Re(a_j) \right) 
  \cos\left(\tau \Im(a_i)\right) \cos\left(\tau \Im(a_j)\right) 
  \left[U^\T \Re(D)  \Re(D)^\T U\right]_{i, j} d \tau \\
  &\fakeeq + \int_0^t \exp\left(\tau \Re(a_i) + \Re(a_j) \right) 
  \sin\left(\tau \Im(a_i)\right) \sin\left(\tau \Im(a_j)\right) 
  \left[U^\T \Re(D)  \Re(D)^\T U\right]_{i, j} d \tau \\
  &\fakeeq + i \int_0^t \exp\left(\tau \Re(a_i) + \Re(a_j) \right) 
  \sin\left(\tau \Im(a_i)\right) \cos\left(\tau \Im(a_j)\right) 
  \left[U^\T \Re(D)  \Re(D)^\T U\right]_{i, j} d \tau \\
  &\fakeeq - i \int_0^t \exp\left(\tau \Re(a_i) + \Re(a_j) \right) 
  \cos\left(\tau \Im(a_i)\right) \sin\left(\tau \Im(a_j)\right) 
  \left[U^\T \Re(D)  \Re(D)^\T U\right]_{i, j} d \tau \\
  &\fakeeq - \int_0^t \exp\left(\tau \Re(a_i) + \Re(a_j) \right) 
  \cos\left(\tau \Im(a_i)\right) \sin\left(\tau \Im(a_j)\right) 
  \left[U^\T \Re(D)  \Im(D)^\T U\right]_{i, j} d \tau \\
  &\fakeeq + \int_0^t \exp\left(\tau \Re(a_i) + \Re(a_j) \right) 
  \sin\left(\tau \Im(a_i)\right) \cos\left(\tau \Im(a_j)\right) 
  \left[U^\T \Re(D)  \Im(D)^\T U\right]_{i, j} d \tau \\
  &\fakeeq - i \int_0^t \exp\left(\tau \Re(a_i) + \Re(a_j) \right) 
  \sin\left(\tau \Im(a_i)\right) \sin\left(\tau \Im(a_j)\right) 
  \left[U^\T \Re(D)  \Im(D)^\T U\right]_{i, j} d \tau \\
  &\fakeeq - i \int_0^t \exp\left(\tau \Re(a_i) + \Re(a_j) \right) 
  \cos\left(\tau \Im(a_i)\right) \cos\left(\tau \Im(a_j)\right) 
  \left[U^\T \Re(D)  \Im(D)^\T U\right]_{i, j} d \tau \\
  &\fakeeq - \int_0^t \exp\left(\tau \Re(a_i) + \Re(a_j) \right) 
  \sin\left(\tau \Im(a_i)\right) \cos\left(\tau \Im(a_j)\right) 
  \left[U^\T \Im(D)  \Re(D)^\T U\right]_{i, j} d \tau \\
  &\fakeeq + \int_0^t \exp\left(\tau \Re(a_i) + \Re(a_j) \right) 
  \cos\left(\tau \Im(a_i)\right) \sin\left(\tau \Im(a_j)\right) 
  \left[U^\T \Im(D)  \Re(D)^\T U\right]_{i, j} d \tau \\
  &\fakeeq + i \int_0^t \exp\left(\tau \Re(a_i) + \Re(a_j) \right) 
  \cos\left(\tau \Im(a_i)\right) \cos\left(\tau \Im(a_j)\right) 
  \left[U^\T \Im(D)  \Re(D)^\T U\right]_{i, j} d \tau \\
  &\fakeeq + i \int_0^t \exp\left(\tau \Re(a_i) + \Re(a_j) \right) 
  \sin\left(\tau \Im(a_i)\right) \sin\left(\tau \Im(a_j)\right) 
  \left[U^\T \Im(D)  \Re(D)^\T U\right]_{i, j} d \tau \\
  &\fakeeq + \int_0^t \exp\left(\tau \Re(a_i) + \Re(a_j) \right) 
  \sin\left(\tau \Im(a_i)\right) \sin\left(\tau \Im(a_j)\right) 
  \left[U^\T \Im(D)  \Im(D)^\T U\right]_{i, j} d \tau \\
  &\fakeeq + \int_0^t \exp\left(\tau \Re(a_i) + \Re(a_j) \right) 
  \cos\left(\tau \Im(a_i)\right) \cos\left(\tau \Im(a_j)\right) 
  \left[U^\T \Im(D)  \Im(D)^\T U\right]_{i, j} d \tau \\
  &\fakeeq - i \int_0^t \exp\left(\tau \Re(a_i) + \Re(a_j) \right) 
  \cos\left(\tau \Im(a_i)\right) \sin\left(\tau \Im(a_j)\right) 
  \left[U^\T \Im(D)  \Im(D)^\T U\right]_{i, j} d \tau \\
  &\fakeeq + i \int_0^t \exp\left(\tau \Re(a_i) + \Re(a_j) \right) 
  \sin\left(\tau \Im(a_i)\right) \cos\left(\tau \Im(a_j)\right) 
  \left[U^\T \Im(D)  \Im(D)^\T U\right]_{i, j} d \tau.
\end{align*}
Let us look at the terms associated with $\left[U^\T \Re(D)  \Re(D)^\T U\right]_{i, j}$.
According to properties of trigonometric functions, we have
\begin{align*}
  \cos(x)\cos(y) + \sin(x)\sin(y) = \cos(x - y) \\
  \sin(x)\cos(y) - \cos(x)\sin(y) = \sin(x - y),
\end{align*}
which implies the factor associated with $\left[U^\T \Re(D)  \Re(D)^\T U\right]_{i, j}$
is 
\begin{align}
  \begin{split} \label{eq:aab}
  \int_0^t \exp\left(\tau \Re(a_i) + \Re(a_j) \right)
  \cos\left(\tau \Im(a_i) - \tau \Im(a_j) \right) d \tau \\
 + \int_0^t i \exp\left(\tau \Re(a_i) + \Re(a_j) \right)
  \sin\left(\tau \Im(a_i) - \tau \Im(a_j) \right) d \tau.
  \end{split}
\end{align}
Note that for complex variable $z$, $\Re\left(\exp(z)\right) = \exp(\Re(z)) \cos(\Im(z))$
and $\Im\left(\exp(z)\right) = \exp(\Re(z)) \sin(\Im(z))$. To see why this is the
case, by Euler's formula,
\begin{align*}
  \exp(z) = \exp\left(\Re(z) + i \Im(z)\right) = \exp\left(\Re(z)\right) \exp\left(i \Im(z) \right)
   = \exp\left(\Re(z)\right) \left(\cos\left(\Im(z)\right) + i \sin \left(\Im(z)\right)\right).
\end{align*}
Therefore, going back to \Cref{eq:aab}, this factor is equal to
\begin{align*}
  & \int_0^t \Re\left(\exp\left(\tau (\Re(a_i) + \Re(a_j) + i (\Im(a_i) - \Im(a_j)))\right)\right) d \tau \\
  &\fakeeq + \int_0^t i \Im\left(\exp\left(\tau (\Re(a_i) + \Re(a_j) + i (\Im(a_i) - \Im(a_j)))\right)\right) d \tau \\
  &= \int_0^t \exp\left(\tau (\Re(a_i) + \Re(a_j) + i (\Im(a_i) - \Im(a_j)))\right) d \tau \\
  &= \int_0^t \exp\left(\tau (a_i + \overline a_j)\right) d \tau,
\end{align*}
where $\overline x$ is the complex conjugate of $x$.
\fi

Similarly, we can calculate the factors associated with 
$\left[U^\T \Im(D)  \Re(D)^\T U\right]_{i, j}$, $\left[U^\T \Re(D)  \Im(D)^\T U\right]_{i, j}$,
and $\left[U^\T \Im(D)  \Im(D)^\T U\right]_{i, j}$.
\begin{align*}
  &\fakeeq \left[U^\T \Gamma(t) U\right]_{i, j} \\
  &= \left[U^\T \Re(D)  \Re(D)^\T U\right]_{i, j} \int_0^t \exp\left(\tau (a_i + \overline a_j)\right) d \tau
    -i \left[U^\T \Re(D)  \Im(D)^\T U\right]_{i, j} \int_0^t \exp\left(\tau (a_i + \overline a_j)\right) d \tau \\
  &\fakeeq + i \left[U^\T \Im(D)  \Re(D)^\T U\right]_{i, j} \int_0^t \exp\left(\tau (a_i + \overline a_j)\right) d \tau
    + \left[U^\T \Im(D)  \Im(D)^\T U\right]_{i, j} \int_0^t \exp\left(\tau (a_i + \overline a_j)\right) d \tau \\
  &= \left[U^\T \Re(D)  \Re(D)^\T U + U^\T \Im(D)  \Im(D)^\T U + i U^\T \Im(D)  \Re(D)^\T U - i U^\T \Re(D)  \Im(D)^\T U \right]_{i, j} \int_0^t \exp\left(\tau (a_i + \overline a_j)\right) d \tau \\
  &= \left[U^\T D D^H U\right]_{i, j} \int_0^t \exp\left(\tau (a_i + \overline a_j)\right) d \tau.
\end{align*}
Therefore, we can obtain that
\ifdefined\mpversion 
  $\Gamma(t) = U \int_0^t \exp\left(\tau S\right) U^\T D D^H U \exp\left(\tau \overline S\right) d \tau U^\T$.
Similarly, for pseudo-covariance $C$, we have 
$C(t) = U \int_0^t \exp\left(\tau S\right) U^\T D D^\T U \exp\left(\tau S\right) d \tau U^\T$.
\else
\begin{align*}
  \Gamma(t) = U \int_0^t \exp\left(\tau S\right) U^\T D D^H U \exp\left(\tau \overline S\right) d \tau U^\T.
\end{align*}

Similarly, for pseudo-covariance $C$, we have
\begin{align*}
  C(t) = U \int_0^t \exp\left(\tau S\right) U^\T D D^\T U \exp\left(\tau S\right) d \tau U^\T.
\end{align*}
\fi

\end{proof}

{

\section{Escape Analysis}

\begin{proof}[Proof of \Cref{thm:escape_analysis}]
  Let us w.l.o.g. assume that $x_0 = 0$ and $f(0) = 0$.
  Let $S_1 = [-b, b]$ for $b>0$.
  We can find a small enough $b$ independent of $\eta$, such that $f''(x) \leq
  -\hat{c}_1$ for some $\hat c_1 > 0$. Therefore, $f''(x)$ is $-\Theta(1)$,
  and $f'(x) < 0$ in $(0, b]$.
  It is known that the hitting time $\Ep{\tau}$ for the one-dimensional SDE
  \begin{align*}
    dX_t = b(X_t) dt + \sqrt{\cD(X_t)} dW_t
  \end{align*}
  can be analytically formulated as~\citep{pavliotis2014stochastic}:
  \begin{align} \label{eq:hitting_time}
    \Ep{\tau} = 2 \int_0^b \frac{\exp\left(-\psi(z)\right)}{\cD(z)} dz
    \int_z^b \exp\left(\psi(y)\right) dy,
  \end{align}
  where $\psi'(x) = -2b(x) / \cD(x)$.
  Since we assume the function to be symmetric, we can focus on the right half
  of $S_1$.

In this problem \oursde{} is
\begin{align*}
dx = \frac{\log(1-\eta f''(x))}{\eta f''(x)} f'(x) dt + \sqrt{\sigma(x) \frac{2 \log(1-\eta f''(x))}{\eta f''(x)^2 - 2 f''(x)}} dW_t.
\end{align*}
Then $\psi'(x)$ in \Cref{eq:hitting_time} becomes such that:
\begin{align*}
\psi'(x) = \frac{1}{\sigma(x)}\left(\frac{2}{\eta} - f''(x)\right) f'(x).
\end{align*}
Since we assume that $-f''(x)$ and $\sigma(x)$ is $\Theta(1)$
as the domain is compact, we have
\begin{align*}
  f'(x) = -\int_0^x f''(y) dy = -\Theta(x), \quad \text{and} \quad
  \psi'(x) = -\Theta\left(\frac{x}{\eta}\right).
\end{align*}
For $\cD$, if we do a Taylor expansion around $\eta$, we will obtain
\begin{align*}
  \cD = \sigma(x) \left(\eta + O(\eta^2)\right) = \Theta(\eta).
\end{align*}
We first start with the upper bound.
For small enough $\eta$, we can split the integral into two parts, let $m_1 = \sqrt{\eta}$, we have
\begin{align}
  \Ep{\tau} = \int_0^{m_1} \frac{\exp(-\psi(x))}{\cD(x)} dx \int_x^b \exp(\psi(y)) dy 
  + \int_{m_1}^b \frac{\exp(-\psi(x))}{\cD(x)} dx \int_x^b \exp(\psi(y)) dy. \label{eq:escape_HA}
\end{align}
For the first part, we have
\begin{align*}
\int_0^{m_1} \frac{\exp(-\psi(x))}{\cD(x)} dx \int_x^b \exp(\psi(y)) dy.
\end{align*}
For the inner integral, we let $m_2 = \hat{c}_2 \sqrt{\eta \log \frac{1}{\eta}}$, where $\hat{c}_2 > 0$ does not depend on $\eta$ and we will determine later. Then, we have
\begin{align*}
\int_x^b \exp(\psi(y)) dy &= \int_x^{x+m_2} \exp(\psi(y)) dy + \int_{x+m_2}^b \exp(\psi(y)) dy \\
&\le \int_x^{x+m_2} \exp(\psi(0)) dy + \int_{x+m_2}^b \exp(\psi(x+m_2)) dy.
\end{align*}
Note that
\begin{align*}
\psi(x+m_2) - \psi(x) &= \int_x^{x+m_2} \psi'(y) dy \\
&= -\Theta\left(\frac{1}{\eta}\right) \int_x^{x+m_2}y dy \\
&= -\Theta\left(\frac{1}{\eta}\right) \left[(x+m_2)^2 - x^2\right] \\
&\le -\Theta\left(\frac{m_2^2}{\eta}\right).
\end{align*}
Next, we have
\begin{align*}
&\fakeeq \int_x^{x+m_2} \exp(\psi(0)) dy + \int_{x+m_2}^b \exp(\psi(x+m_2)) dy \\
&\le m_2 \exp(\psi(0)) + b \exp(\psi(x+m_2)) \\
&= m_2 \exp(\psi(0)) + b \exp(\psi(x) + \psi(x+m_2) - \psi(x)) \\
&\le m_2 \exp(\psi(0)) + b \exp(\psi(x)) \exp\left(-\Theta\left(\frac{m_2^2}{\eta}\right)\right) \\
&\le \exp(\psi(0)) \left(m_2 + b \exp\left(-\Omega\left(\frac{m_2^2}{\eta}\right)\right)\right) \\
&\le \exp(\psi(0)) (m_2 + b \eta) \\
&\le O\left(\sqrt{\eta \log \frac{1}{\eta}}\right) \exp(\psi(0)),
\end{align*}
where we choose $\hat{c}_2$ such that $-\Omega\left(\frac{m_2^2}{\eta}\right)= -\Omega(\hat{c}_2^2 \log \frac{1}{\eta}) $
in the third to last line less or equal to $\log \eta$.

Next, we estimate the whole first part of \eqref{eq:escape_HA}.
We have that for $|x-0| \le m_1$, there exists $\epsilon > 0$, such that $\psi(x) \ge \psi(0) + (\psi''(0) - \epsilon) x^2$. The integral can be upper bounded by
\begin{align*}
&\fakeeq O\left( \sqrt{\eta\log \frac{1}{\eta}}\right) \exp(\psi(0)) \int_0^{m_1} \frac{\exp\left(-(\psi(0) + \frac{1}{2}(\psi''(0) - \epsilon) x^2)\right)}{\cD(x)} dx \\
&= O\left(\sqrt{\eta \log \frac{1}{\eta}}\right) \int_0^{m_1} \frac{\exp\left(-\frac{1}{2}(\psi''(0) - \epsilon) x^2\right)}{\cD(x)} dx \\
&\le O\left(\sqrt{\eta \log \frac{1}{\eta}}\right) m_1 \exp\left(\frac{1}{2} |\psi''(0) - \epsilon| m_1^2\right) \\
&\le O\left(\sqrt{\eta^{-1} \log \frac{1}{\eta}}\right) m_1 \exp\left(O\left(\frac{m_1^2}{\eta}\right)\right),
\end{align*}
since $\psi''(0) = \frac{1}{\sigma(0)}(\frac{2}{\eta} - f''(0)) f''(0) = -\Theta(\frac{1}{\eta})$, $\cD(x) = \Theta(\eta)$.
In addition, we note that as $\eta$ gets smaller, $m_1$ becomes smaller,
and thus $\epsilon$ gets smaller, while $\abs{\psi''(0)}$ gets larger, therefore,
$\epsilon$ is negligible compared to $\psi''(0)$.
Then we have the upper bound being $O\left(\sqrt{\log \frac{1}{\eta}}\right)$.

Before we move to the second part of \eqref{eq:escape_HA},
we first show that we can choose $b$ independent of $\eta$,
such that $\psi''(x) \le 0$ in $[0, b]$.
We have
\begin{align*}
  \psi''(x) = \frac{1}{\sigma(x)} \left(\underbrace{-f'''(x) f'(x)}_{\leq O(x)} + 
  \underbrace{\left(\frac{2}{\eta} - f''(x)\right) f''(x)}_{\leq - \frac{1}{\eta}}\right)
  \underbrace{- \frac{1}{\sigma(x)^2} \left(\frac{2}{\eta} - f''(x)\right) f'(x) \sigma'(x)}_{\leq O(x/\eta)},
\end{align*}
and there exists constants $k_1, k_2 > 0$ such that
\begin{align*}
\psi''(x) \le \frac{k_1}{\eta} x - \frac{k_2}{\eta},
\end{align*}
since according to our assumption $f'''(x) \le O(1)$, $f''(x) = -\Theta(1)$, $-f'(x) = -\int_0^x f''(y) dy = \Theta(x)$,
$\sigma(x) = \Theta(1)$ and $\abs{\sigma'(x)} \leq O(1)$.
Therefore we can choose $b$ between $[0, \frac{k_2}{k_1}]$ independent of $\eta$.

For the second part of \eqref{eq:escape_HA}, we have that the maximum of inner integrand is at $x$.
Further, we have for $y > x$ and $y \leq b$,
\begin{align*}
\psi(y) \le \psi(x) + \psi'(x) (y-x),
\end{align*}
since $\psi''(z) \le 0$ for $z \in [0, b]$. Therefore, for the inner integral,
\begin{align*}
\int_x^b \exp(\psi(y)) dy &\le \int_x^b \exp(\psi(x) + \psi'(x) (y-x)) dy \\
&= \exp(\psi(x)) \frac{1}{-\psi'(x)} \left(1 - \exp(\psi'(x) (b-x))\right) \\
&\le \exp(\psi(x)) \frac{1}{-\psi'(x)},
\end{align*}
where $\psi'(x) < 0$.

Incorporating with the outer integral, we have
\begin{align*}
\int_{m_1}^b \frac{\exp(-\psi(x))}{\cD(x)} \exp(\psi(x)) \frac{1}{-\psi'(x)} dx &= \int_{m_1}^b \frac{1}{\cD(x)} \frac{1}{-\psi'(x)} dx \\
&= \Theta\left(\frac{1}{\eta}\right) \int_{m_1}^b \frac{1}{-\psi'(x)} dx.
\end{align*}
Note that we have $\psi'(x) =  -\Theta\left(\frac{1}{\eta}\right) x$.
Therefore the above can be upper bounded by
\begin{align*}
\Theta\left(\frac{1}{\eta}\right) \int_{m_1}^b \frac{1}{\frac{1}{\eta} x} dx &= \Theta(1) \int_{m_1}^b \frac{1}{x} dx \\
&= \Theta(1) (\log b - \log m_1) \\
&= \Theta(1) \log \frac{1}{m_1} = O\left(\log \frac{1}{\eta}\right),
\end{align*}
which completes the proof for the upper bound.

For the lower bound, we can directly start with the second part
of \eqref{eq:escape_HA}, which is a lower bound for the whole integral.
This time we choose $m_1 = \sqrt{\eta} \log 1/\eta$.
For any $\delta > 0$, the inner integral can be lower bounded by
\begin{align*}
  \int_x^b \exp(\psi(y)) dy \geq \int_x^{x+\delta} \exp(\psi(y)) dy.
\end{align*}
By Taylor's theorem we have for $y \in [x, x+\delta]$,
\begin{align*}
  \psi(y) \geq \psi(x) + \left(\psi'(x) - \epsilon\right) (y-x),
\end{align*}
where $\epsilon = \max_{z \in [x, x+\delta]} \abs{\psi''(z)} (y-x) / 2 $.
Then we obtain
\begin{align*}
  &\fakeeq \int_x^b \exp(\psi(y)) dy \\
  &\geq \int_x^{x+\delta} \exp(\psi(x) + \left(\psi'(x) - \epsilon\right) (y-x)) dy \\
  &= \exp(\psi(x)) \frac{1}{-\left(\psi'(x) - \epsilon\right)} \left(1 - \exp(\left(\psi'(x) - \epsilon\right) \delta)\right) \\
  &\geq \Omega\left(\frac{\exp(\psi(x))}{-\psi'(x)}\right),
\end{align*}
where we choose $\delta$ to be $\sqrt{\eta}$. The last inequality holds because
\begin{align*}
\psi'(x) = - \Theta\left(\frac{x}{\eta}\right) \leq - \Theta\left(\frac{m_1}{\eta}\right)
= O\left(\frac{\log 1/\eta}{\sqrt{\eta}}\right),
\end{align*}
therefore, $\psi'(x) \delta \to - \infty$ when $\eta \to 0$.
Note that $\epsilon$ is $O\left(1/\eta \cdot \delta\right) = O\left(1/\sqrt{\eta}\right)$,
which is negligible compared to $\psi'(x)$.

The outer integral follows the same as the upper bound, where in the proof
we only have equality. Therefore, we obtain that the lower bound is
\begin{align*}
  \Theta\left(\log \frac{1}{m_1}\right) = \Theta\left(\log \frac{1}{\eta}\right) - \Theta\left(\log \log \frac{1}{\eta}\right) = \Omega\left(\log \frac{1}{\eta}\right).
\end{align*}

Next, we consider the case when $x_0$ is a local maximum.
For this case, it has been already studied in \citet{aleksian2025freidlin},
where the drift and diffusion terms can depend on $\eta$.
We have that for small enough $\eta$, the drift and diffusion terms of \oursde{}
converge to these of \sme-1.
Therefore, the result can be directly obtained by \citet[Corollary~1.9]{aleksian2025freidlin}.

\end{proof}

\begin{proof}[Proof of \Cref{thm:escape_analysis_multi_dim}]
\sme-2 follows the following dynamics:
\begin{align} \label{eq:sme2_large_eta}
  dX_t &= -\frac{1}{\eta} \nabla  \underbrace{\left(\hat{f}(X_t) + \frac{1}{4} \|\nabla \hat{f}(X_t)\|^2\right)}_{\tilde f(X_t)}  dt + \sqrt{\eta \Sigma(X_t)} dW_t,
\end{align}
where $\nabla^2 \hat f(x) = g(x)$.
Let $Y_t = X_{\eta t}$, then we obtain for $Y_t$:
\begin{align*}
dY_t = -\nabla \tilde{f}(Y_t) dt + \sqrt{\eta^2 \Sigma(Y_t)} dW_t.
\end{align*}
Note that according to our assumption, $x_0$ is a local minimum for $\tilde f$ as
\begin{align*}
\nabla^2 \tilde{f}(x_0) = \left(I + \frac{1}{2} g(x_0)\right) g(x_0),
\end{align*}
since $\nabla \tilde{f}(x_0) = 0$.
For the negative eigenvalues,
we have $\eta \lambda_i$, which equals the corresponding eigenvalues of $g$, smaller than $-2$.
Therefore, the corresponding eigenvalues of $\nabla^2 \tilde{f}(x_0)$ are greater than $0$.
Then we could apply Freidlin and Wentzel's~\citep{vent1973some,FreidlinWentzell2012}
estimate for the expected hitting time, where the Jacobian of the drift term has
only negative eigenvalues.
Let $\tau_{S_1}^Y$
to be the first time that $Y_t$ hits the boundary of $S_1$, and we know that
$\eta \tau_{S_1}^Y = \tau_{S_1}^X$.
According to \citet[Lemma~1]{vent1973some}, we have
\begin{align*}
\lim_{\eta \to 0} \eta^2 \log \mathbb{E}\left[\tau_{S_1}^Y\right] = \inf_{x \in \partial S_1} V(x),
\end{align*}
where $V$ is the quasi-potential and is independent of $\eta$. Since the boundary of $S_1$ is also independent of $\eta$, $\inf_{x \in \partial S_1} V(x)$ is a constant when considering the asymptotics of $\eta$.
Then the above implies
\begin{align*}
\lim_{\eta \to 0} \eta^2 \log \mathbb{E}\left[\tau_{S_1}^X\right] = \lim_{\eta \to 0} \eta^2 \log \eta \mathbb{E}\left[\tau_{S_1}^Y\right] = \lim_{\eta \to 0} \eta^2 \log \mathbb{E}\left[\tau_{S_1}^Y\right] = \inf_{x \in \partial S_1} V(x).
\end{align*}
Therefore, there exists $c_1, \eta_0 > 0$ such that for all $\eta < \eta_0$,
\begin{align*}
\mathbb{E}\left[\tau_{S_1}^X\right] \ge \exp\left(\frac{c_1}{\eta^2}\right).
\end{align*}

For \oursde{}, the SDE can also be written as similar to \Cref{eq:sme2_large_eta},
where the drift has factor $1/\eta$ and diffusion has factor $\sqrt{\eta}$.
The difference is that now the Jacobian of the drift term contains also positive
eigenvalues. To see why this is the case, we compute the Jacobian at $x_0$
\begin{align*}
  \nabla b(x_0) = \frac{1}{\eta} U(x_0) \log \left(I - \hat \Lambda(x_0)\right) U(x_0),
\end{align*}
where $\hat \Lambda(x_0)$ is the diagonal matrix of the eigenvalues of $g(x_0)$.
Note that we only need to differentiate w.r.t. the $\nabla f(x)$ factor as
$\nabla f(x_0) = 0$.
According to our assumption, for negative entries of $\hat \Lambda(x_0)$, the
corresponding eigenvalue of $\nabla b(x_0)$ is positive.

We do the same trick as for the proof of \sme-2 to change the time scale and
follow \citet[Theorem~2.2]{kifer1981exit}. The result follows.

\end{proof}

}

\end{document}